\documentclass[twoside,11pt]{article}

\usepackage{jmlr2e}
 \usepackage{lastpage}
\usepackage{smile_jmlr}
 \usepackage{mathtools}
\usepackage{glatent}
\usepackage{cme-math}
 \usepackage{rotating}

 \newtheorem{assumption}[theorem]{Assumption}
            
\jmlrheading{21}{2020}{1-\pageref{LastPage}}{6/18; Revised 12/19}{3/20}{18-357}{Carson Eisenach, Florentina Bunea, Yang Ning and Claudiu Dinicu}

\ShortHeadings{High-Dimensional Inference for Cluster-Based Graphical Models}{Eisenach, Bunea, Ning and Dinicu}
\firstpageno{1}

\begin{document}

\title{High-Dimensional Inference for Cluster-Based Graphical Models}

\author{\name Carson Eisenach \email eisenach@princeton.edu \\
       \addr Department of Operations Research and Financial Engineering\\
       Princeton University\\
       Princeton, NJ 08544, USA
       \AND
       \name Florentina Bunea \email fb238@cornell.edu \\
       \name Yang Ning \email yn265@cornell.edu\\
       \name Claudiu Dinicu \email cd535@cornell.edu\\
       \addr Department of Statistics and Data Science\\
       Cornell University\\
       Ithaca, NY 14850, USA       }

\editor{Nicolas Vayatis}

\maketitle

\begin{abstract}
Motivated by modern applications in which one constructs graphical models based on a very large number of features, this paper introduces a new class of cluster-based graphical models, in which   variable clustering is applied as an initial step for reducing the dimension of the feature space. We employ model assisted clustering, in which the clusters contain features that are similar to the same unobserved latent variable. Two different cluster-based Gaussian graphical models are considered: the latent variable graph, corresponding to the graphical model associated with the unobserved latent variables, and the cluster-average graph, corresponding to the vector of features averaged over clusters. Our study reveals that likelihood based inference for the latent graph, not analyzed previously,  is analytically intractable. Our main contribution is the development and analysis of  alternative estimation and inference strategies, for  the precision matrix of an unobservable latent vector $Z$. We replace the likelihood of the data by an appropriate class of empirical risk functions, that can be specialized  to the latent graphical model and to the simpler, but under-analyzed, cluster-average graphical model. The estimators thus derived can be used  for inference on the graph structure, for instance  on edge strength or pattern recovery. Inference is based on the asymptotic limits of the entry-wise estimates of the precision matrices associated with the conditional independence  graphs under consideration. While taking the uncertainty induced by the clustering step into account, we establish  Berry-Esseen central limit theorems for the proposed estimators. It is noteworthy that, although the clusters are estimated adaptively from the data, the central limit theorems regarding the entries of the estimated graphs  are proved under the same conditions one would use if the clusters were known in advance.  As an illustration  of the usage of  these newly developed inferential tools, we show that they can be reliably used for recovery of the sparsity pattern of the graphs we study, under FDR control, which is verified via simulation studies and an fMRI data analysis. These experimental results confirm the theoretically established difference between the two graph structures. Furthermore, the data analysis suggests that the  latent variable graph, corresponding to  the unobserved cluster centers, can help provide more insight into the understanding of the brain connectivity networks relative to the simpler, average-based, graph. \end{abstract}


 \section{Introduction}
\label{sec:intro}

Over the last several decades, graphical models have become an increasingly popular method for understanding independence and conditional independence relationships between components of random vectors. More recently, the challenges posed by the estimation and statistical analysis of graphical models with many more nodes than the number of observations has led to renewed interest in these models, such as  \cite{Meinshausen06,Yuan07, Friedman08,verzelen2008gaussian,Lam09,rothman2008sparse, peng2009partial, ravikumar2011high,Yuan10, Cai11a, Sun12b, liu2012high,xuezou2012,ning2013high,cai2016estimating,tan2016replicates,fan2017high,yang2018semiparametric,feng2019high}, to give only an incomplete list.

Nonetheless, when the dimension (number of nodes) grows very large and the sample size is small, the dependency among the components of a random vector may become weak, if it exists at all, and difficult to detect without additional information. If the dimension of the random vector is in the thousands, even if the dependency structure can be detected by an estimated graphical model, it can be difficult to interpret the results and extract meaningful scientific insights.

One solution to both of the aforementioned issues is to employ an initial dimension reduction procedure on the high dimensional vector. For example, in neuroscience applications, a typical functional magnetic resonance image (fMRI) consists of blood-oxygen-level-dependent (BOLD) activities measured at 200,000+ voxels of the brain, over a period of time. Instead of analyzing voxel-level data directly, scientists routinely  cluster  voxels into several regions of interest (ROI) with homogeneous functions using domain knowledge, and then carry out the analysis at the ROI-level. In this example, using the language of graphical models, the group structure of variables may boost the dependency signals.  Similar pre-processing steps are used in other application domains, such as genomics, finance and economics.

Motivated by {a rich set of applications, we consider  variable clustering as the initial dimension reduction step applied to the} observed vector $\bX =: (X_1, \ldots, X_d) \in \RR^d$.  To the best of our knowledge,  very little is known about the effect of clustering on down-stream analysis and, consequently, on the induced graphical models.  Our contribution is the provision of a framework that allows for such an analysis. We introduce cluster-based graphical models,  show how they can be estimated and, furthermore, provide the asymptotic distribution of the edge strength estimates. 

These models are built on the assumption that the observed variables $\bX =(X_1, \ldots, X_d) \in \RR^d$ can be partitioned into $K$ unknown clusters $G^* = \{\Gs{1}, \ldots, \Gs{K}\}$ such that variables in the same cluster share the same behavior. Following the intuition behind widely-used $K$-means type procedures, we define a population-level cluster as a group of variables that are noise corrupted versions of a hub-like variable. This hub-like variable is not directly observable, and is treated as a latent factor.

Formally, we assume there exists a latent random vector $\bZ \in \RR^K$, with mean zero and covariance matrix $\Cov(\bZ) = \bCS$, such that
\begin{equation}
\label{eqn:g_latent_model}
\bX=\Ab \bZ+\bE,
\end{equation}
for a zero mean error vector $\bE$ with independent entries. The entries of the $d \times K$ matrix $\Ab$ are given by  $A_{jk}=\II\{j\in G^*_{k}\}$. A  cluster of variables consist in those components of $\Xb$ with indices in the same $\Gs{k}$. We denote $\Cov(\bE)=\bGammaS$, a diagonal matrix with entries  $\Gamma^*_{jj} =\gammaS{j}$ for any $1\leq j\leq d$.  We also assume that the mean-zero noise $\bE$ is independent of $\bZ$. \cite{Bunea2018} show that the clusters are uniquely defined by the model in \eqref{eqn:g_latent_model}, provided that the smallest cluster contains at least two variables and $\bCS$ is strictly positive definite; this result holds irrespective of distributional assumptions.

To keep the presentation focused, in this work we assume that $\bZ\sim \cN(0,\bCS)$ and $\bE \sim \cN(0,\bGammaS)$, which implies $\bX \sim \cN(0,\bSigmaS)$ with $\bSigmaS=\Ab\bCS\Ab^T+\bGammaS$. In this context we consider two related, but different, graphical models:
\begin{itemize}
\item[(i)] The {\it latent variable graph}, associated  with  the sparsity pattern of  the precision matrix
\begin{equation}
\label{teta}
  \bTheta^*:=\Cb^{*-1}
\end{equation}
of the Gaussian vector $\bZ \in \RR^K$. The latent variable graph encodes conditional independencies (CI's) among the unobserved, latent variables $\bZ$.
\item[(ii)] The {\it cluster-average graph}, associated with the sparsity pattern of  the precision matrix
\begin{equation}
\label{omega}
  \bOmega^*:=\bS^{*-1},
\end{equation}
where $\bEssS$  is the covariance matrix of $\bar{\bX} \in \RR^{K}$, and $\bar{\bX} =: (\bar{X}_1, \ldots, \bar{X}_K)$ is the within cluster average given by $\bar{X}_k =: \frac{1}{|\Gs{k}|}\sum_{i \in \Gs{k}} X_i$. The cluster-average graph encodes CI's among averages of observable random variables. In particular, we have
\[
\bEssS = \bCS + \bar\bGammaS,
\]
where $\bar \bGammaS=\textrm{diag}(\bar\gamma^*_{1},...,\bar\gamma^*_K)$ with $\bar\gamma^*_k=\frac{1}{|G_k^*|^2}\sum_{j\in G_k^*}\gamma_j^*$.
\end{itemize}

Although both these graphs correspond to vectors of dimension $K$ and are constructed based on the partition $G^*$, they are in general different as the sparsity patterns of $\bTheta^*$  and $\bOmega^*$ will typically differ, and have different interpretations. Therefore it would be misleading to use one as a proxy for the other when drawing conclusions. For instance, in the neuroscience example, if we interpret each latent variable as the function of a ROI, then the latent variable graph {encodes the CI relationships} between functions, which is one question of scientific interest. The difference between $\bTheta^*$  and $\bOmega^*$ shows that this question will not be typically answered by studying the cluster-average graph, although that  may be tempting to do by practitioners. 

\subsection{Our Contributions}
Since the two cluster-based graphical models introduced above can  both be of interest in a large array of applications, we provide inferential tools for both of them in this work. We assume that we observe $n$ i.i.d. copies $\bX_1, \ldots, \bX_n$ of $\bX$. The focus  of our work is on post-clustering and post-regularization inference for these two sparse precision matrices.  To this end, we derive the asymptotic distribution of  novel estimators of their entries.  These estimators can be used to answer any inferential questions of interest at the edge strength level, or can be combined to provide sparsity pattern recovery under FDR control.  In Section \ref{sec:main_results:fdr_control} we provide an instance of the latter. 

Inference for the entries of a Gaussian precision matrix has received a large amount of attention in the past few years,  most notably post-regularization inference, for instance \cite{Ren2013,Zhang2014,jankova2014confidence,gu2015local,barber2015rocket,jankova2017honest,javanmard2013confidence,van2013asymptotically,ning2017general,cai2017confidence,neykov2015unified,ning2017likelihood,fang2017testing}. These works generalize the classical ideas of one-step estimation \citep{bickel1975one} to the high-dimensional setup by first constructing a sparse estimator of the precision matrix via regularization, and then building de-sparsified updates that are asymptotically normal.  The effect of the initial regularization step is controlled in the second step, and inference after regularization becomes valid.

In this paper we consider a similar estimation strategy, but differs from the existing literature in several important ways.  In our work, we add another layer of data-dependent dimension reduction, via clustering, and provide a framework within which the variability induced by clustering can be controlled.  Even after controlling for the clustering variability, we note that the existing  procedures for estimation and, especially,  post-regularization inference in Gaussian graphical models are not immediately applicable to our problem for the following reasons:
\begin{enumerate}
  \item They are developed for variables that can be observed directly. From this perspective, they could, in principle, be applied to the cluster-average graph, but are not directly extendable to the latent graph;
  \item To the best of our knowledge, all existing methods for precision matrix  inference require the largest eigenvalue of the corresponding  covariance matrix to be upper bounded by a constant.  Such an assumption implies, in turn, that the Euclidean norm of each row of the covariance matrix is bounded, which reduces significantly the parameter space for which inference is valid. The assumption holds, for instance, when the number of variables is bounded, or when the entries of each row are appropriately small.
\end{enumerate}

To overcome these limitations, we take a different approach in this work, that allows us to lift  unpleasant technical conditions associated with other procedures, while maintaining the validity of inference for  both the latent  and the average graph. We summarize our main contributions below. \\



\noindent {\bf 1.} {\bf Methods for  estimation tailored to high dimensional inference in latent, cluster-based,  graphical models.}
We develop a new estimation strategy tailored to our final goal, that of  constructing approximately Gaussian estimators for the entries of the precision matrices $\bTheta^*$ and $\bOmega^*$ given in (\ref{teta}) and (\ref{omega}) above. Although we work under the assumption that the  data is Gaussian, likelihood based estimators may be unsatisfactory, because their analysis can require stringent assumptions, as explained above (see also \cite{jankova2014confidence}), or may become analytically intractable, as argued in Section \ref{sec_LVG}, for the latent graph. We propose a method that mimics very closely the principles underlying the construction of an efficient score function for estimation in the presence of high dimensional nuisance parameters (see for instance \cite{bickel1993efficient}), but we do not base it on the corresponding likelihood-derived quantities.  The underlying  principles are explained in Section \ref{sec:general}. To the best of our knowledge, this is the  first estimation method of the latent precision matrix, that can be analyzed theoretically for inferential purposes.  As an added benefit of our  estimation framework, the same principles can be applied for the estimation, and analysis,  of the cluster-average graph.

\noindent {\bf 2.} {\bf The analysis of estimators of the precision matrix of unobserved cluster centers:  Berry-Esseen-type bounds for Gaussian approximations.} We verify, theoretically, that the estimators constructed with an inferential goal in mind do indeed have the desired  properties. 
To this end, we derive the asymptotic distribution of the proposed entry-wise estimates of the latent graph, and also of the cluster-average graph. Moreover, we quantify the speed with which this limit is obtained,  which we show to be proportional to $1/\sqrt{n}$ in both cases. We do so by establishing  Berry-Esseen type bounds on the difference between the cumulative distribution function of our estimators and that of a standard Gaussian random variable that are valid for each $K, d$ and $n$,  and are presented in Theorems \ref{thm:xi_asymptotic} and  \ref{thm:theta_asymptotic}, respectively.  As immediate applications, we can construct approximate confidence intervals for one or a finite number of entries of the latent or average graph, or known linear functionals of such entries. While the answers we thus provide at the average graph level are similar to existing results, established for the full CI\ graph based on all $d$ nodes, for instance by \citep{Ren2013,Zhang2014,jankova2014confidence,gu2015local,barber2015rocket,jankova2017honest,javanmard2013confidence,van2013asymptotically,ning2017general,cai2017confidence,neykov2015unified}, the results for the latent graph are, to the best of our knowledge, the first such results in the literature.  

We note, furthermore, that the average cluster graph can be viewed as a graph with observable nodes, the cluster averages,  only after the clusters are estimated from the data.  Our theoretical analysis takes this step into account. 
We discuss, in Section \ref{sec:introduction}, clustering methods tailored to model (\ref{eqn:g_latent_model}), where the number of clusters $K$ is unknown and is allowed to grow with $n$.  Using the results of \cite{Bunea2018}, these methods  yield a partition $\widehat G = G^*$, with high probability, provided that $\lambda_{\text{min}}(\Cb^*) > c$, for  a small positive quantity $c$  made precise in  Section \ref{sec:introduction}.  A lower bound on the smallest eigenvalue of the covariance matrix is the minimal condition under which inference in {\it any} graphical model can be performed.  Therefore, consistent clustering via the model \eqref{eqn:g_latent_model} does not  require a  further reduction of the parameter space for which the more standard post-regularized inference can be developed.  Moreover, as Section \ref{sec:main_results} shows, asymptotic inference based on the estimated clusters reduces to asymptotic inference relative to the true clusters, $G^*$, without any need for data splitting. This fact holds true for both the average and the latent variable graph, and  is in sharp contrast with a phenomenon often encountered in post-model selection inference, such as in variable selection in linear regression \citep{lockhart2014significance,lee2013exact,taylor2014post}. In that case, reducing inference to the consistently selected set of variables can only be justified over a reduced part of the parameter space \citep{bunea2004consistent}, and is therefore not a popular practice.  

Another technical contribution is that the asymptotic normality of the estimators is established under relaxed conditions. Unlike the existing literature on de-biased  inference on graphical models \citep{jankova2014confidence,jankova2017honest}, we do not require the bounded operator norm condition for the covariance matrix such as $\lambda_{\max}(\bS^*)\leq C$ for cluster-average graph. As shown by \cite{jankova2014confidence} and explained above, the analysis of the multivariate Gaussian likelihood  may
require stringent assumptions for cluster-average graph and becomes intractable for the latent graph. By using the proposed {pseudo-likelihood} function which has a much simpler form, we can remove the unpleasant assumption on the bounded operator norm. In addition, we re-analyze the CLIME estimator \citep{Cai11a} $\hat\bOmega_{\cdot k}$ (the $k$th column of $\hat\bOmega$) under our Assumption 4.1 and 4.2, which is used as the initial estimator for inference. As explained in Section \ref{sec_average}, we can show that the CLIME estimator satisfies $\|\hat\bOmega_{\cdot k}-\bOmega_{\cdot k}^*\|_1\lesssim s_1\sqrt{\frac{\log (K\vee n)}{n}},$ where $s_1$ is the sparsity of $\bOmega_{\cdot k}^*$. This result does not require the bounded operator norm or matrix $L_1$ norm condition and can be of independent interest.

To illustrate how one can use these newly developed inferential tools,  we focus on the estimation of the sparsity pattern of the graphs, which can be equivalently viewed as a multiple-testing problem.  It is well known that the exact sparsity pattern can be recovered, with high probability, only if the entries of each precision matrix are above the minimax optimal noise level $O(\sqrt{\log d/n})$ \citep{ravikumar2011high,Meinshausen06}.  Since our aim is inference on the sparsity pattern without further restrictions on the parameter space, the next best type of error that we can control is the False Discovery Rate (FDR) \citep{BH95}.   In Section \ref{sec:main_results:fdr_control} we  use these results for pattern recovery under FDR control, and explain the  effect  of the asymptotic approximations on this quantity.

This paper is organized as follows. Section \ref{sec:background} below contains a brief summary of existing results on model-assisted clustering, via model \eqref{eqn:g_latent_model}. Section \ref{sec:inference}  describes the estimation procedures for the latent variable graph and the cluster-average graph, respectively. In Section \ref{sec:main_results} we establish  Berry-Esseen type  central limit theorems for the estimators derived in Section \ref{sec:inference}, and  provide bounds on the FDR associated with each graphical model under study, respectively. Section \ref{sec:numerical} gives numerical results using both simulated and real data sets.

\section{Background}\label{sec:background}
\subsection{Notation}
\label{sec:introduction:notation}
The following notation is adopted throughout this paper. Let $d$ denote the ambient dimension, $n$ the sample size, $K$ the number of clusters  and $m$ the minimum cluster size. The matrix $\bCS$ denotes the population covariance of the latent vector $\bZ$. Likewise, the matrices $\bGammaS$, $\bSigmaS$, $\bThetaS$, $\bEssS$ and $\bOmegaS$ denote population-level quantities.

For $\vb=(v_1,...,v_d)^{T} \in \mathbb{R}^d$, and $1 \leq q \leq \infty$, we define $\|{\vb}\|_q=(\sum_{i=1}^d |v_i|^q)^{1/q}$, $\|{\vb}\|_0=|\textrm{supp}(\vb)|$, where $\textrm{supp}(\vb)=\{j: v_j\neq 0\}$ and $|A|$ is the cardinality of a set $A$. Denote $\|{\vb}\|_{\infty} = \max_{1\leq i \leq d} |v_i|$ and $\vb^{\otimes 2}=\vb\vb^T$. Assume that $\vb$ can be partitioned as $\vb=(\vb_1,\vb_2)$. Let $\nabla f(\vb)$ denote the gradient of the function $f(\vb)$, and $\nabla_1 f(\vb)=\partial f(\vb)/\partial \vb_1$. Similarly,  let $\nabla^2 f(\vb)$ denote the Hessian of the function $f(\vb)$ and $\nabla^2_{12} f(\vb)=\partial^2 f(\vb)/\partial \vb_1\partial \vb_2$.

For a $d\times d$ matrix $\Mb=[M_{jk}]$, let  $\|{\Mb}\|_{\max}=\max_{jk}|M_{jk}|$, $\|{\Mb}\|_1=\sum_{jk}|M_{jk}|$, and $\|{\Mb}\|_{\infty}=\max_{k}\sum_{j}|M_{jk}|$.
If the matrix $\Mb$ is symmetric, then $\lambda_{\min}(\Mb)$ and $\lambda_{\max}(\Mb)$ are the minimal and maximal eigenvalues of $\Mb$. Let $[d]=\{1,2,....,d\}$. For any $j\in [d]$, we denote the $j$th row and $j$th column of $\Mb$ as $\Mb_{j\cdot}$ and $\Mb_{\cdot j}$, respectively. Similarly, let $\Mb_{-j,-k}$ be the sub-matrix of $\Mb$ with the $j^{th}$ row and $k^{th}$ column removed. The notation $\cS^{d\times d}$ refers to the set of all real, symmetric $d \times d$ matrices. Likewise, $\psdc{d} \subset \cS^{d\times d}$ is the positive semi-definite cone. We use $\otimes$ and $\circ$ to denote the Kronecker and Hadamard product of two matrices, respectively; we also may write  $\Mb^{\otimes 2}=\Mb\otimes\Mb$. Let $\eb_{j}$ denote the vector of all zeros except for a one in the $j^{th}$ position. The vector $\bone$ is the vector of all ones. $a\vee b=\max(a,b)$.

\subsection{Model Assisted Variable Clustering}
\label{sec:introduction}

In this section we review existing results on variable clustering. \cite{Bunea2018} showed that if we use model \eqref{eqn:g_latent_model} to define clusters of variables, these clusters are uniquely defined, up to label switching, so long as $m=: \min_{1 \leq k \leq K} |G_k^*| \geq 2$ and the components of the latent vector $\bZ$ are different almost surely, or equivalently
\[
\Delta(\bCS)=: \min_{j < k} \EE(Z_j - Z_k)^2  > 0.
\]
Since
\[
\Delta(\bCS)=  \min_{j<k} (\eb_{j}-\eb_{k})^T\bCS(\eb_{j}-\eb_{k})\geq 2 \lmin{\bCS},
\]
the clusters  are uniquely defined as soon as  $\lmin{\bCS} > 0$, which is the minimal condition under which one can study properties of the corresponding precision matrix.

In addition, \cite{Bunea2018} developed two algorithms, PECOK and COD, that are shown to recover the clusters exactly, from  $n$ i.i.d. copies $\bX_1, \ldots, \bX_n$ of $\bX$, as soon as  \[
\lmin{\bCS} \geq c,
\]
for a positive quantity $c$ that approaches 0 as $n$ grows. For the COD procedure,
\[
c = O\left(\|\bSigma^*\|_{\max}\sqrt{\log (d \vee n)/n}\right).
\]
On the other hand, for the PECOK procedure
$$c = O\left(\|\bGamma^*\|_{\max}\sqrt{K\log (d \vee n)/mn}\right),$$
which can be much smaller when one has a few, balanced, clusters.

These values of $c$ are shown to be minimax or near-minimax optimal for cluster recovery.  We refer to Theorems 3
and 4 in  \cite{Bunea2018} for the precise expressions and details.  Under these minimal conditions on $\lmin{\bCS}$, the exact recovery of the clusters holds  with probability larger than $1 - 1/(d \vee n)$. Because these conditions are sufficiently weak, we are able to show, in Section \ref{sec:main_results}, that inference in cluster-based graphical models is not hampered by the clustering step.

For completeness, we outline the PECOK algorithm below, which consists in a convex relaxation of the $K$-means algorithm, further tailored to estimation of clusters $G^* = \{\Gs{1}, \ldots, \Gs{K}\}$ defined via the interpretable model  (\ref{eqn:g_latent_model}).   The PECOK algorithm consists in  the following three steps:
\begin{enumerate}
\item Compute an estimator $\tilde\bGamma$ of the matrix $\bGamma^*$.
\item Solve the semi-definite program (SDP)
\begin{equation}\label{eqn:pecok_sdp}
\widehat \Bb =\argmax _{\Bb \in \cD}\langle \widehat{\bSigma} - \widetilde{\bGamma}, \Bb \rangle, \end{equation}
where  $\widehat \bSigma$ is the sample covariance matrix and
\begin{equation} \label{eq:domain}
\cD:=\left\{ \Bb \in R^{d\times d}:
                \begin{array}{l}
                  \bullet\ \bB  \succcurlyeq 0 \  \ \text{(symmetric and positive semidefinite)} \\
                  \bullet\  \sum_a B_{ab} = 1,\ \forall b\\
        \bullet\ B_{ab}\geq 0,\ \forall a,b\\
        \bullet\ \tr(\Bb) = K
                \end{array}
              \right\}.
  \end{equation}
\item Compute $\widehat G$ by applying a clustering algorithm on the rows (or equivalently columns) of $\widehat \Bb$.
\end{enumerate}
The construction of an  accurate estimator $\widetilde \bGamma$ of $\bGamma^*$, before the cluster structure is known,  is a  crucial step for guaranteeing the statistical optimality of  the PECOK estimator.  Its construction is given in   \cite{Bunea2018}, and included in Appendix \ref{pregamma}, for the convenience of the reader.

We will employ an efficient algorithm for solving  \eqref{eqn:pecok_sdp}. Standard black-box SDP solvers, for a fixed precision, exhibit $\cO(d^7)$ running time on \eqref{eqn:pecok_sdp}, which is prohibitively expensive. \cite{Eisenach2019b} recently introduced the FORCE algorithm, which requires worst case $\cO(d^{6}K^{-2})$ time to solve the SDP, and in practice often performs the clustering rapidly.

The key idea behind the FORCE algorithm is that an optimal solution to \eqref{eqn:pecok_sdp} can be attained by first transforming \eqref{eqn:pecok_sdp} into an eigenvalue problem, and then using a first-order method. Iterations of the first-order method are interleaved with a dual step to round the current iterate to an integer solution of the clustering problem, and then searches for an optimality certificate. By using knowledge of both the primal and the dual SDPs, FORCE is able to find the solution much faster than a standard SDP solver. We refer to \cite{Eisenach2019b} for the detailed algorithm.

\section{Estimation of Cluster-based Graphical Models}\label{sec:inference}

In this section, we propose a unified estimation approach, that utilizes similar  loss functions for estimation and inference in the cluster-average and the latent variable graphs. We first describe our general principle, and then demonstrate its application to the two graphical models.

\subsection{One-step Estimators for High-Dimensional Inference}\label{sec:general}
\label{sec:inference:z_est}
Assume that we observe $n$ i.i.d. realizations $\bX_1,...,\bX_n$ of $\bX \in \RR^d$. Let $Q(\bbeta, \bX)$ denote a known mapping of $\bbeta$ and $\bX$ to $\RR$, where $\bbeta$ is a $q$-dimensional unknown parameter of the distribution of $\bX$. Often this $Q$ is referred to as the loss function. We define the target parameter $\bbeta^*$ as
\[
\bbeta^*=\argmin \EE(Q(\bbeta, \bX)).
\]

Next, let us partition $\bbeta$ as $\bbeta=(\theta,\bgamma)$, where $\theta \in \RR$ is the univariate parameter of interest, and $\bgamma \in \RR^{q-1}$ is a nuisance parameter. Our goal is to construct a $n^{1/2}$-consistent and asymptotically normal estimator for $\theta$ in high-dimensional models with $q=\textrm{dim}(\bbeta)\gg n$. In this case, the dimension of the nuisance parameter $\gamma$ is large, which makes the inference on $\theta$ challenging. We start from the empirical risk function over $n$  observations defined as
\begin{equation}\label{eqQn}
Q_n(\bbeta)=\frac{1}{n} \sum_{i=1}^nQ(\bbeta,  \bX_i).
\end{equation}
One standard choice for $Q_n$ is the negative log-likelihood function of the data. In this work, we conduct inference based on an alternative loss function, as the analysis of the log-likelihood can require unpleasant  technical conditions that we would like to avoid, as discussed in Sections \ref{sec:main_results:assumption}. That said, we mimic likelihood principles as much as possible, in order to make intuitive the construction below. For these reasons we will refer to $Q_n(\bbeta)$ as the negative {\it pseudo-likelihood} function.

For now we leave $Q_n(\bbeta)$ unspecified --  a detailed discussion of its selection for inference in the latent variable graph and the cluster-average graph will be given in the following two subsections. In terms of $Q_n$, we define the pseudo-information matrix for one observation as $\Ib(\bbeta)=\EE(\nabla^2 Q(\bbeta^*, \bX_i))$. We can partition this matrix as
\begin{equation}\label{eqpartition}
\Ib(\bbeta) =
\begin{bmatrix}
\Ib_{11} & \Ib_{12}  \\
\Ib_{21} & \Ib_{22}  \\
\end{bmatrix},
\end{equation}
with the partitions corresponding to those of $\bbeta=(\theta,\bgamma)$.

When $Q_n$ is the negative log-likelihood function, and the dimension of the parameter is independent of $n$, then $h(\btheta; \bgamma)$ given by \eqref{eqeff} is called the {\it efficient score function} for $\theta$, and classical theory shows that it admits solutions that are consistent, asymptotically normal and attain the information bound given by the reciprocal of \eqref{eqinfor} \citep{van,bickel1993efficient}.

With these goals in mind, we similarly define the corresponding pseudo-score function for estimating $\theta$ in the presence of the nuisance parameter $\bgamma$ as
\begin{align*}
h(\theta; \bgamma) &= \nabla_1Q_n(\bbeta)-\Ib_{12}\Ib_{22}^{-1}\nabla_2Q_n(\bbeta)\\
&= \frac{1}{n}\sum_{i=1}^n \Big(\nabla_1Q(\bbeta, \bX_i)-\Ib_{12}\Ib_{22}^{-1}\nabla_2Q(\bbeta, \bX_i)\Big) \numberthis \label{eqeff}
\end{align*}
and  define the pseudo information of $\theta$, in the presence of the nuisance parameter $\bgamma$, as
\begin{equation}\label{eqinfor}
\Ib_{1\mid 2}=\Ib_{11}-\Ib_{12}\Ib_{22}^{-1}\Ib_{21}.
\end{equation}
When the dimension of $\bgamma$ is fixed, one can easily estimate $\Ib_{12}$ and $\Ib_{22}$ in (\ref{eqeff}) by their sample versions $\hat \Ib_{12}$ and $\hat \Ib_{22}$. However, such simple procedure fails when the dimension of $\bgamma$ is greater than the sample size, as $\hat \Ib_{22}$ is rank deficient. To overcome this difficulty, rather than estimating $\Ib_{12}$ and $\Ib_{22}^{-1}$ separately, we directly estimate \begin{equation}\label{w}  \wb^T=\Ib_{12}\Ib_{22}^{-1} \end{equation} by
\begin{equation}\label{eqwd}
        \hat\wb=\argmin \|{\wb}\|_1, ~~~~\textrm{s.t.}~~~~ \|{\nabla^2_{12} Q_n(\hat\bbeta)- \wb^T\nabla^2_{22} Q_n(\hat\bbeta)}\|_\infty\leq\lambda',
\end{equation}
where $\lambda'$ is a non-negative tuning parameter, and $\hat\bbeta=(\hat\theta,\hat\bgamma)$ is an initial estimator,  which is usually defined case by case, for a given model. Then, we can plug $\hat\wb$ and $\hat\bgamma$ into the pseudo-score function, which gives
\begin{equation}\label{eqeffest}
\hat h(\theta, \hat\bgamma)=\nabla_1Q_n(\theta, \hat\bgamma)-\hat\wb^T\nabla_2Q_n(\theta, \hat\bgamma).
\end{equation}
Following the Z-estimation principle \citep{van,bickel1993efficient}, one could define the final  estimator of $\theta$ as the solution of the pseudo-score function $\hat h(\theta, \hat\bgamma)$. However, in many examples, the  pseudo-score function $\hat h(\theta, \hat\bgamma)$ may have multiple solutions and it becomes unclear which root serves as a consistent estimator; see \cite{small2000eliminating} for further discussion of the general estimating function context. To bypass this issue, we consider the following simple one-step estimation approach. Given the initial estimator $\hat\theta$ from the partition of $\hat\bbeta$, we perform a Newton-Raphson update based on the pseudo-score function $\hat h(\theta, \hat\bgamma)$,
to obtain  $\tilde{\theta}$, which is traditionally referred to as a one-step estimator by \cite{bickel1975one}. Specifically, we construct
\begin{equation}\label{eqest}
\tilde\theta=\hat\theta-\hat \Ib_{1|2}^{-1}\hat h(\hat\theta,\hat\bgamma),
\end{equation}
where $\hat \Ib_{1|2}$ is an estimator of the partial information matrix $\Ib_{1\mid 2}$.  In  Sections \ref{sec_average} and \ref{sec_LVG} below we show that, under appropriate conditions,  the one-step estimator $\tilde\theta$  constructed using the empirical risk functions $Q_n$ -- defined in \eqref{eqloss} and \eqref{eqloss2}, respectively --  satisfies
\begin{equation}
  \label{eq2}
n^{1/2}(\tilde{\theta}-\theta^*)=-\Ib^{-1}_{1\mid 2}n^{1/2} h(\bbeta^*)+o_p(1).
\end{equation}
By applying the central limit theorem to $h(\bbeta^*)$, we establish the asymptotic normality of  $\tilde{\theta}$ in Theorems \ref{thm:xi_asymptotic} and \ref{thm:theta_asymptotic}.

When $Q_n(\bbeta)$ is the negative log-likelihood of the data, \cite{ning2017general} successfully used this approach and the resulting estimator $\tilde\theta$ is asymptotically equivalent to the de-biased estimator in \cite{Zhang2014,van2013asymptotically}. As we explain in the following subsections, analysis based on the log-likelihood becomes intractable for the latent graphical model and requires stringent technical conditions for the cluster-average graphical model. To overcome this difficulty, we employ pseudo-score functions derived from \eqref{eqloss} and \eqref{eqloss2}. The resulting one-step estimator still attains the information bound established in the literature, and more importantly requires weaker technical assumptions than the existing methods.  In addition to \eqref{eq2}, we derive explicitly the speed at which the normal approximation is attained.

\subsection{Estimation of the Cluster-Average Graph}\label{sec_average}

Recall that we assume $\bZ\sim \cN(0,\bCS)$ and $\bE \sim \cN(0,\bGammaS)$, which implies $\bX \sim \cN(0,\bSigmaS)$ with $\bSigmaS=\Ab\bCS\Ab^T+\bGammaS$. The within-cluster average $\bar{\bX} =: (\bar{X}_1, \ldots, \bar{X}_K)\in\RR^K$ is given by $\bar{X}_k =: \frac{1}{|\Gs{k}|}\sum_{i \in \Gs{k}} X_i$, corresponding to the population level clusters. Because $\bX \sim \cN(0,\bSigmaS)$, we can verify that $\bar \bX \sim \cN(0,\bEssS)$, where
\begin{equation}
\label{eqn:s_star_definition}
\bEssS = \bCS + \bar\bGammaS,
\end{equation}
and $\bar \bGammaS=\textrm{diag}(\bar\gamma^*_{1},...,\bar\gamma^*_K)$ with $\bar\gamma^*_k=\frac{1}{|G_k^*|^2}\sum_{j\in G_k^*}\gamma_j^*$.
Recall that the precision matrix of $\bar\bX$ is
\[
\bOmega^*={\bS^{*}}^{-1} =  (\bCS + \bar\bGammaS)^{-1}.
\]

In this section we give the construction of the estimators  of  the cluster-average graph corresponding to $\bar{\bX}$. Specifically,  we use the generic strategy outlined in the previous section in order to construct  $n^{1/2}$-consistent and asymptotically normal estimators for each  component $\bOmega^{*}_{t,k}$ of the precision matrix $\bOmega^{*}$, for $1\leq t<k\leq K$. For the estimation of each entry, the remaining $K(K+1)/2-1$ parameters in $\bOmega^{*}$ are treated as nuisance parameters.

Since  we observe $n$ i.i.d. samples of $\bX \in \RR^p$,  if the clusters and their number were known, then we would implicitly observe $n$ i.i.d. samples of $\bar \bX \in \RR^K$. A priori, the clusters are not known else this problem would simply reduce to the standard setting. However to explain our method, we first assume that clustering is given, and then show how to lift this assumption.

Following our general principle, we would naturally  tend to choose the negative log-likelihood function of the cluster-averages $(\bar \bX_1,...,\bar\bX_n)$ as the empirical risk function $Q_n(\bbeta)$ in (\ref{eqQn}). Along this line, \cite{jankova2014confidence} proposed the de-biased estimator for Gaussian graphical models.  However, the inference requires the irrepresentable condition \citep{ravikumar2011high} on $\bS^*$, which can be restrictive. The alternative methods proposed by \cite{Ren2013,jankova2017honest} imposed the condition that the largest eigenvalue of $\bS^*$ is bounded. These technical conditions on $\bS^*$ are difficult to justify and can be avoided by using our approach. We propose to estimate each sparse row of $\bOmega^*$ as explained below.

Let $\bar\bS=n^{-1}\sum_{i=1}^n\bar\bX_i\bar\bX_i^T$ denote the sample covariance matrix of $\bar\bX_i$.  When $K$ is small, the maximum likelihood estimator of $\bOmega^*$ is $\bar\bS^{-1}$, which can be viewed as the solution of the following equation $\bar\bS\bOmega-\Ib_K=0$. Thus, in the low dimensional setting, this equation  defines the maximum likelihood estimator. Since we are only interested in $\bOmega^*_{t,k}$, we can extract the $k$th column from the left hand side of the above equation, and use it as the pseudo-score function $\bU_n(\bOmega_{\cdot k})=\bar\bS\bOmega_{\cdot k}-\eb_k$. To apply the inference strategy in Section \ref{sec:general}, we need to  construct a valid empirical  risk function $Q_n(\bOmega_{\cdot k})$ such that   $\nabla Q_n(\bOmega_{\cdot k})=\bU_n(\bOmega_{\cdot k})$.

Simple algebra shows that a possible choice is
\begin{equation}
Q_n(\bOmega_{\cdot k})=\frac{1}{2}\bOmega_{\cdot k}^T\bar\bS\bOmega_{\cdot k}-\eb_k^T\bOmega_{\cdot k}=\frac{1}{n}\sum_{i=1}^n (\frac{1}{2}\bOmega_{\cdot k}^T \bar\bX_i\bar\bX_i^T\bOmega_{\cdot k}-\eb_k^T\bOmega_{\cdot k}),\label{eqloss}
\end{equation}
which we view in the sequel as the empirical risk  corresponding to the population level risk
\begin{equation} \label{risk-ave} \EE Q(\bOmega_{\cdot k},\bar\bX) = \frac{1}{2} \bOmega_{\cdot k}^T \bS^*\bOmega_{\cdot k}  -\eb_k^T\bOmega_{\cdot k}, \end{equation}
based on the loss function
\begin{equation} \label{loss1}
Q(\bOmega_{\cdot k},\bar\bX)=:   \frac{1}{2} \bOmega_{\cdot k}^T \bar\bX\bar\bX^T\bOmega_{\cdot k}  -\eb_k^T\bOmega_{\cdot k}.
\end{equation}

Since \begin{equation}\label{qlike}
\nabla \EE Q(\bOmega^*_{\cdot k},\bar\bX) =\bS^*\bOmega^{*}_{\cdot k}-\eb_k=0
\end{equation}
and
 \begin{equation}\label{qlike2}
\nabla^2 \EE Q(\bOmega_{\cdot k}^*,\bar\bX)=\bS^*,
\end{equation}
then  the population risk $ \EE Q(\bOmega^*_{\cdot k},\bar\bX)$  has the rows $\bOmega^{*}_{\cdot k}$ of the  target precision  matrix  $\bOmega^*$ as the unique minimizers, as desired, provided that $\bS^*$ is positive definite, an assumption we make in Section \ref{sec:main_results:assumption}.

 We note that the choice of the  empirical risk    $Q_n(\cdot)$ and that of the corresponding  pseudo-score $\bU_n(\cdot)$ is not unique. We chose the particular form (\ref{eqloss})  because it is quadratic in $\bOmega_{\cdot k}$, which greatly simplifies the theoretical analysis and leads to weaker technical assumptions. Moreover, the property (\ref{qlike}) is the same as that  of the score function  corresponding to the negative log-likelihood function, supporting our  terminology.

We use the general strategy presented in Section \ref{sec:inference:z_est} to construct estimators that employ the empirical risk  $Q_n(\cdot)$ defined by (\ref{eqloss}) above.
We first recall that $Q_n(\cdot)$ depends on the unknown cluster structure $G^*$ via $\bar\bX_i$. We note that in general the estimated group $\hat G_k$ may differ from $G_k^*$ by a label permutation. For notational simplicity, we ignore this label permutation issue and treat $\hat G_k$ as an estimate of $G_k^*$ (rather than $G_j^*$ for some $j\neq k$).
To define our estimator  of $\bOmega^{*}_{t,k}$, we first replace  $\bar \bX_i$ by $\hat \bX_i$ and denote $\hat\bS=n^{-1}\sum_{i=1}^n\hat\bX_i\hat\bX_i^T$, where $\hat X_{ik}= \frac{1}{|\hat G_k|}\sum_{j \in \hat G_k} X_{ij}$.

Let $(t,k)$ be arbitrary, fixed. Replacing $\bar\bS$ by $\hat\bS$ in $Q_n(\cdot)$, we follow  Section \ref{sec:general} to define the pseudo-score function
\begin{equation}
h(\bOmega_{\cdot k}) = \vb_t^{*T}(\hat\bS\bOmega_{\cdot k} - \eb_k),
\end{equation}
where ${\vb}^*_{t}$ is a $K$-dimensional vector with $(\vb_{t}^*)_t=1$ and $(\vb_{t}^*)_{-t}=-\wb^*_{t}$ with $\wb^*_t=(\bS^*_{-t,-t})^{-1}\bS^*_{-t,t}$, which is consistent with the definition in (\ref{w}) above.
To make inferences based on $h(\bOmega_{\cdot k})$, we further need to estimate $\wb^*_t$ and $\bOmega^*_{\cdot k}$. Following (\ref{eqwd}), an estimate of $\wb^*_{t}$ is given by
\begin{equation}\label{eqw}
\hat \wb_{t}=\argmin \|\wb\|_1, ~~\textrm{s.t}~~\|\hat\bS_{t,-t}-\wb^T\hat\bS_{-t,-t}\|_\infty\leq\lambda',
\end{equation}
where $\lambda'$ is a tuning parameter. Then we can define $\hat \vb_{t}$ accordingly, and
\begin{equation}\label{eqscore}
\hat h(\bOmega_{\cdot k})=\hat\vb_{t}^T(\hat{\bS}\bOmega_{\cdot k}-\eb_k).
\end{equation}
Recall that the construction of the one-step estimator (\ref{eqest}) requires an initial estimator of $\bOmega^*_{\cdot k}$.

To be concrete, we consider the following initial estimator of $\bOmega^*_{\cdot k}$,
\begin{equation}\label{eqclime1}
\hat \bOmega_{\cdot k}=\argmin \|\bbeta\|_1, ~~\textrm{s.t}~~\|\hat\bS\bbeta-\eb_k\|_{\max}\leq\lambda,
\end{equation}
where $\lambda$ is a tuning parameter. This estimator has the same form as the CLIME estimator for the $k$-th column of $\bOmega$ \citep{Cai11a}. However, unlike the CLIME estimator which requires $\lambda \asymp \|\bOmega^*_{\cdot k}\|_1\sqrt{\log K/n}$, in our Theorem \ref{thm:xi_asymptotic} we assume $\lambda =C \sqrt{\log (K\vee n)/n}$, where $C$ only depends on the minimum eigenvalue of $\Cb^*$ and the largest diagonal entries of $\Cb^*$ and $\bGamma^*$ which are assumed bounded by constants in Assumptions \ref{asmp:bounded_latent_covariance} and \ref{asmp:bounded_errors}. With this choice of $\lambda$, we show in Lemma \ref{lem:group_averages_consistency} in Appendix \ref{sec:main_proofs_ca} that
\begin{equation}\label{eqrateomega}
\|\hat\bOmega_{\cdot k}-\bOmega_{\cdot k}^*\|_1\lesssim s_1\sqrt{\frac{\log (K\vee n)}{n}},
\end{equation}
with high probability, where $s_1$ is the sparsity level of $\bOmega_{\cdot k}^*$. The sharp concentration of the gradient and Hessian of the empirical risk function provided in Lemma \ref{lem:group_averages_gradient_hessian} is the key for this result.  As a comparison, Theorem 6 in \cite{Cai11a} only implies $\|\hat\bOmega_{\cdot k}-\bOmega_{\cdot k}^*\|_1\lesssim s_1\|\bOmega_{\cdot k}^*\|_1^2\sqrt{\frac{\log K}{n}}$.  For many sparse matrices, the $\ell_1$ norm of a column, $\|\bOmega_{\cdot k}^*\|_1$, can grow to infinity with $K$ or $s_1$, and thus \eqref{eqrateomega} gives a faster rate. 

In this case, and when $\lambda_{\min}(C^*) > c$,  Lemma \ref{lem:group_averages_gradient_hessian} is instrumental in showing that the extra $\|\bOmega_{\cdot k}^*\|_1^2$ factor in  the rate of the original CLIME estimator  can be avoided, whereas if only the marginal components of $\bX$ are assumed to be  sub-Gaussian, as in  \cite{Cai11a},  it may be unavoidable,  without further conditions on $\bOmega^{*}$. We direct the reader to Appendix \ref{sec:indepth_comparison} for a more detailed discussion of the distinction between our results and those in \citet{Cai11a,cai2016estimating}.

Leveraging the block matrix inverse formula, we can show that the partial pseudo information matrix reduces to $\Ib_{1|2}=1/\Omega_{t,t}^{*}$. Finally, the one-step estimator is defined as
\begin{equation}\label{eqOmega}
\tilde\Omega_{t,k}= \hat\Omega_{t,k}-\hat h(\hat\bOmega_{\cdot k}) \hat\Omega_{t,t},
\end{equation}
in accordance with \eqref{eqest}. In Section \ref{sec:main_results}, we show that under mild regularity conditions $n^{1/2}(\tilde\Omega_{t,k}-\Omega_{t,k}^*)\leadsto N(0,s_{tk}^2)$, where $s_{tk}^2=\Omega_{t,k}^{*2} + \Omega^*_{t,t}\Omega^*_{k,k}$. If $\hat s_{tk}^2=\hat\Omega_{t,k}^2 + \hat\Omega_{t,t}\hat\Omega_{k,k}$ is a consistent  estimator of the asymptotic variance, then a $(1-\alpha)\times 100\%$ confidence interval for $\Omega_{t,k}$ is
\[
[\tilde\Omega_{t,k}-z_{1-\alpha/2}\hat s_{tk}/n^{1/2}, \tilde\Omega_{t,k}+z_{1-\alpha/2}\hat s_{tk}/n^{1/2}],
\]
where $z_{\alpha}$ is the $\alpha$-quantile of a standard normal distribution.

Equivalently, we can use the scaled test statistics $\tilde{\Omega}_{t,k}$ to construct a test for $H_0: \Omega^*_{t,k}=0$ versus $H_1: \Omega^*_{t,k} \neq 0$ with $\alpha$ significance level. Namely, the null hypothesis is rejected if and only if the above $(1-\alpha)\times 100\%$ confidence interval does not contain $0$. We will employ such  tests in Section  \ref{sec:main_results}.

\subsection{Latent Variable Graph}\label{sec_LVG}

Recall that the structure of the latent variable graph is encoded by the sparsity pattern of $\bThetaS ={\bCS}^{-1}$, which is generally different from the cluster-average group as ${\bCS}^{-1}$ and $\bOmega^*=(\bCS + \bar\bGamma^*)^{-1}$ may have different sparsity patterns.
In this section,  we focus on the inference on the component $\ThetaS{t}{k}$, for some $1\leq t<k\leq K$. Similar to the cluster-average graph, we first discuss the likelihood approach. The negative log-likelihood corresponding to model \eqref{eqn:g_latent_model} indexed by the parameter $(\bTheta, \bGamma)$  is, up to some additive and multiplicative constants,
$$
\ell(\bTheta, \bGamma)= \log|(\Ab\bTheta^{-1}\Ab^T+\bGamma)| + \tr(\hat\bSigma(\Ab\bTheta^{-1}\Ab^T+\bGamma)^{-1}),
$$
where  $\hat{\bSigma}=n^{-1}\sum_{i=1}^n \bX_i\bX^T_i$.
It is straightforward to show that the Fisher information matrix for $(\bTheta, \bGamma)$ is given by
\begin{align}
\Ib(\bTheta, \bGamma) &= \begin{bmatrix}(\Mb^*\Ab^T\bGamma^{*-1}\bSigma^*\bGamma^{*-1}\Ab\Mb^*)^{\otimes 2} & (\Mb^*\Ab^T\bGamma^{*-1}\bSigma^{-1}\Fb^{*T})^{\otimes 2}\Db_d \\ \Db_d^T(\Fb^*\bSigma^{*-1}\bGamma^{*-1}\Ab\Mb^*)^{\otimes 2} & \Db_d^T(\Fb^*\bSigma^{*-1}\Fb^{*T})^{\otimes 2}\Db_d\end{bmatrix}, \label{eqinforlatent}
\end{align}
where $\Db_d = (\Ib_d \otimes 1_d^T)\circ(1_d^T \otimes \Ib_d)$, $\Mb^*=(\bTheta^*+\Ab^T\bGamma^{*-1}\Ab)^{-1}$ and $\Fb^*=\Ib_d -\Ab\Mb^*\Ab^T\bGamma^{*-1}$.
As seen in Section \ref{sec:general}, the inference based on the likelihood or equivalently efficient score function (\ref{eqeff}) requires the estimation of $\Ib_{12}\Ib_{22}^{-1}$ which,  given the complicated structure of the information matrix (\ref{eqinforlatent}), becomes analytically intractable.

A solution to this  problem is inference  based on an empirical risk function similar to (\ref{eqloss}), but tailored  to  the latent variable graph. With a slight abuse of notation, and reasoning as in (\ref{qlike}) and (\ref{qlike2}),  we notice that, for each $k$,
\begin{equation} \label{risk-latent} \EE Q(\bTheta_{\cdot k},\bX) =  \frac{1}{2}\bTheta_{\cdot k}^T \bCS \bTheta_{\cdot k}  -\eb_k^T\bTheta_{\cdot k}, \end{equation}
has  the target $\bTheta^{*}_{\cdot k}$ as a unique minimizer, where the loss function  $Q(\bTheta_{\cdot k},\bX)$ is defined as
\begin{equation} \label{loss-latent}
Q(\bTheta_{\cdot k},\bX)=\frac{1}{2}\bTheta_{\cdot k}^T \bar\Cb \bTheta_{\cdot k}  -\eb_k^T\bTheta_{\cdot k},
\end{equation}
and the matrix $\bar \Cb :=(\bar C_{jk})_{j,k}$ has entries
\begin{equation} \label{eqCi}
\bar C_{jk}=\frac{1}{|G^*_j||G^*_k|}\sum_{a\in  G^*_j, b\in  G^*_k} (X_a X_b-\bar{\Gamma}_{ab}),
\end{equation}
and $\bar{\Gamma}_{ab}=0$ if $a\neq b$ and $\bar{\Gamma}_{aa}=X_aX_a-\frac{1}{|G^*_k|-1}\sum_{a\in G^*_k, a\neq j}X_aX_j$. Since $\EE(\bar \Cb)=\Cb^*$,   the risk  relative to the  loss function in (\ref{loss-latent}) is indeed (\ref{risk-latent}), and   the empirical risk  is
\begin{equation}\label{eqloss2}
Q_n(\bTheta_{\cdot k})=\frac{1}{n}\sum_{i=1}^n (\frac{1}{2}\bTheta_{\cdot k}^T \bar\Cb^{(i)} \bTheta_{\cdot k}  -\eb_k^T\bTheta_{\cdot k}),
\end{equation}
where $\bar\Cb^{(i)}$ is obtained by replacing $\bX$ in $\bar C_{jk}$ by $\bX_i$. Similar to the cluster-average graph, $Q_n(\bTheta_{\cdot k})$ also depends on the unknown cluster structure. We estimate  $G^*_k$ by $\hat G_k$, and define $\hat\bGamma=(\hat\Gamma_{ab})$, where $\hat{\Gamma}_{ab}=0$ if $a\neq b$ and $\hat{\Gamma}_{aa}=\frac{1}{n}\sum_{i=1}^n(X_{ia}X_{ia}-\frac{1}{|\hat G_k|-1}\sum_{a\in \hat G_k, a\neq j}X_{ia}X_{ij})$, and $\hat\Cb=(\hat C_{jk})$ where $\hat C_{jk}=\frac{1}{n}\sum_{i=1}^n(\frac{1}{|\hat G_j||\hat G_k|}\sum_{a\in  \hat G_j, b\in \hat G_k} (X_{ia} X_{ib}-\hat{\Gamma}_{ab}))$.   We replace $\bar \Cb$ by $\hat\Cb$ in (\ref{eqloss2}) above and  follow exactly the strategy of Section \ref{sec_average}, with $\hat \bS$ replaced by $\hat \Cb$,  to construct the corresponding pseudo-score function $\hat h(\bTheta_{\cdot k})$, similarly to (\ref{eqscore}),
 and the initial estimator $\hat \bTheta_{\cdot k}$, similarly to (\ref{eqclime1}). We combine these quantities, following  the general strategy (\ref{eqest}), as above,  to obtain the final one-step estimator of $\Theta_{t,k}^*$, defined as
\begin{equation}\label{tteta}
\tilde\Theta_{t,k}= \hat\Theta_{t,k}-\hat h(\hat\bTheta_{\cdot k}) \hat\Theta_{t,t},
\end{equation}
after observing that, in this case, $\Ib_{1|2}=1/ {\bTheta}^{*}_{t,t}$.

Although the form of this estimator is similar to (\ref{eqOmega}), derived for the cluster-average graph,  the study of  the asymptotic normality of $\tilde\Theta_{t,k}$  reveals that its asymptotic variance  is much more involved, as will be discussed in detail in Section \ref{hard}.

\section{Main Theoretical Results}
\label{sec:main_results}

\subsection{Assumptions}
\label{sec:main_results:assumption}
In this section we state the two  assumptions under which all our results are proved.

\begin{assumption}\label{asmp:bounded_latent_covariance}
The covariance matrix $\bCS$ of $\bZ$ satisfies: $c_1 \leq \lmin{\bCS}$ and $\max_{t}\CS{t}{t} \leq c_2$, for some absolute constants $c_1, c_2 > 0$.
\end{assumption}
\begin{assumption}\label{asmp:bounded_errors}
The matrix $\bGammaS$ satisfies: $\max_{ 1 \leq i \leq d} \gamma^*_i \leq c_3$ for some absolute constant $c_3 > 0$, where $\gammaS{i}$ are the entries of the diagonal matrix $\bGammaS$.
\end{assumption}

Assumptions \ref{asmp:bounded_latent_covariance}  and \ref{asmp:bounded_errors} are minimal conditions for inference on precision matrices. Furthermore, they imply the conditions needed for clustering consistency derived in  \cite{Bunea2018} and discussed in Section \ref{sec:introduction}, for $n$ sufficiently large.  Their work only requires that $\lmin{\bCS}$ is bounded from below by a  sequence that converges to zero, as soon as $\|\bSigma^*\|_{\max}$  and $\|\bGamma^*\|_{\max}$ are bounded. This is strengthened by our assumptions. In general, a constant lower bound on $\bCS$ is standard in any inference on graphical models and is needed to show  the asymptotic normality of the estimator introduced above \citep{Ren2013,jankova2014confidence,jankova2017honest}.

\subsection{Asymptotic Normality via Berry-Esseen-type  Bounds}\label{sec_theorem}

\subsubsection{Results for the  Cluster-Average Graph}

In the section, we show that  the estimators $\tilde\Omega_{t,k}$ given by (\ref{eqOmega})  are  asymptotically normal,  for all $t<k$ . We define the sparsity of the cluster-average graph as $s_1 \in \NN$ such that
\[
\max_{1\leq j\leq K}\sum_{k=1}^K \II(\Omega^*_{j,k} \neq 0)\leq s_1.
\]
Recall that the estimators (\ref{eqw}) and (\ref{eqclime1}) depend on the tuning parameters $\lambda$ and $\lambda'$. In the following theorem, we choose $\lambda \asymp \lambda' \asymp \sqrt{\frac{\log (K\vee n)}{n}}$.  For notational simplicity, we use $C$ to denote a generic constant, the  value of which may change from line to line.

\begin{theorem}
\label{thm:xi_asymptotic}
If Assumptions \ref{asmp:bounded_latent_covariance} and \ref{asmp:bounded_errors} hold, we have
\begin{equation}\label{eqxi_asymptotic2}
\max_{1\leq t< k\leq K}\sup_{x \in \RR}\Big|\PP(\hat T_{t,k} \leq x )-\Phi(x)\Big|\leq \frac{C}{(d \vee n)^3} +  \frac{Cs_1\log (K\vee n)}{n^{1/2}}+  \frac{C}{(K\vee n)^3}
\end{equation}
where $\hat T_{t,k} = \frac{n^{1/2}(\tilde\Omega_{t,k}-\Omega^*_{t,k})}{\hat s_{tk}}$, $\hat s_{tk}^2=\hat\Omega_{t,k}^2 + \hat\Omega_{t,t}\hat\Omega_{k,k}$ and $C$ is a positive constant.
\end{theorem}

Theorem \ref{thm:xi_asymptotic}, proved in Appendix \ref{sec:main_proofs_ca},  gives the rate of the normal approximation of the distribution of the scaled and centered  entries $\tilde\Omega_{t,k}$.  The right hand side in (\ref{eqxi_asymptotic2}) is non-asymptotic and is valid for each $K$, $n$ and $d$.   Its first, small, term is the price to pay for having first used the data for clustering, and it is dominated by the other two terms. From this perspective, the clustering step is the least taxing, as long as we can ensure its consistency, which in turn can be guaranteed under the minimal assumptions \ref{asmp:bounded_latent_covariance} and \ref{asmp:bounded_errors} already needed  for the remaining steps.

The second, and dominant, term regards  the normal approximation of  the distribution of \begin{equation}\label{unscaled} {n^{1/2}(\tilde\Omega_{t,k}-\Omega^*_{t,k})}.\end{equation} Specifically,  as an intermediate step, Proposition \ref{prop:asymptotic_normality_group_averages} in Appendix \ref{sec:main_proofs_ca}  shows that the  difference between the c.d.f. of (\ref{unscaled}), scaled by  $s_{tk} = \sqrt{\Omega_{t,k}^{*2} + \Omega^*_{t,t}\Omega^*_{k,k}}$, and that of a standard Gaussian random variable is bounded by  $\frac{s_1\log (K\vee n)}{n^{1/2}}$. Therefore, asymptotic normality holds as soon as this quantity converges to zero, which agrees with the weakest sparsity conditions for Gaussian graphical model inference in the literature \citep{Ren2013,jankova2017honest}. In addition, the asymptotic variance $s_{tk}^2$ agrees with the minimum variance bound in Gaussian graphical models \citep{jankova2017honest}. Thus, inference based on  the empirical risk function (\ref{eqloss}) does not lead to any asymptotic efficiency loss. Unlike the previous works, we do not require the bounded operator norm condition, $\lambda_{\max}(\bS^*)\leq C$. This condition is avoided in our analysis by using a more convenient empirical risk function (\ref{eqloss}), as opposed to the log-likelihood in \cite{jankova2014confidence}, and a CLIME-type initial estimator (\ref{eqclime1}) satisfying (\ref{eqrateomega}), as opposed to the node-wise Lasso estimator in \cite{jankova2017honest}.

The last term in the normal approximation is $O((K\vee n)^{-3})$ which is dominated by the second one, and is associated with the replacement of the true variance $s_{tk}^{2}$ by the estimate $\hat s_{tk}^2$. Finally, we note that the powers of the first and the third term in the right hand side of (\ref{eqxi_asymptotic2}) can be replaced by $2 +\delta$, for any $\delta > 0$, and a change in this power also changes the associated constant $C$ in the term $\frac{Cs_1\log (K\vee n)}{n^{1/2}}$. As shown in Theorem \ref{thm:fdr_bound_av}, to obtain valid FDR control, we need $K^2/(K\vee n)^{2+\delta}=o(1)$, which holds for any $\delta>0$. For simplicity, we choose $\delta=1$ which gives the power 3.

\subsubsection{Results for the Latent Variable Graph}\label{hard}
In this section we show that the estimators $\tilde{\Theta}_{t,k}$ given by (\ref{tteta}) are asymptotically normal, for all $t<k$. We define the sparsity of the latent graph as $s_0 \in \NN$ such that
\[
\max_{1\leq j\leq K}\sum_{k=1}^K \II(\ThetaS{j}{k} \neq 0)\leq s_0.
\]
Inference for  the estimator $\tilde\Theta_{tk}$  follows the  general approach outlined in Section  \ref{sec:inference:z_est}.  We prove in Proposition \ref{prop:asymptotic_normality_latent} in Appendix \ref{sec:main_proofs_lvg} that
\begin{equation}
\label{key}
n^{1/2}(\tilde\Theta_{t,k}-\Theta_{t,k}^*)=\frac{1}{n^{1/2}}\sum_{i=1}^n \Theta_{t,t}^*\vb_t^{*T}(\bar \Cb^{(i)}\bTheta^*_{\cdot k}-\eb_k)+o_p(1),
\end{equation}
where ${\vb}^*_{t}$ is a $K$-dimensional vector with $(\vb_{t}^*)_t=1$ and $(\vb_{t}^*)_{-t}=-\wb^*_{t}$ with $\wb^*_t=(\Cb^*_{-t,-t})^{-1}\Cb^*_{-t,t}$. and $\bar \Cb^{(i)}$ is defined in (\ref{eqCi}).  The terms of  the sum in  display (\ref{key}) are mean zero random variables,  and their  variance is
$$
\sigma^2_{tk}=\EE(\Theta_{t,t}^*\vb_t^{*T}(\bar \Cb^{(i)}\bTheta^*_{\cdot k}-\eb_k))^2,
$$
which does not have an explicit  closed form, unlike  the asymptotic variance of  the estimates of the entries of $\bOmega^*$. However,
we show in Proposition \ref{lem:latent_variable_variance} in Appendix \ref{sec:main_proofs_lvg} that $\sigma^2_{tk}$ admits an approximation that is easy to estimate:
$$
\Big|\sigma^2_{tk}-[(\ThetaS{t}{k})^2 + \ThetaS{k}{k}\ThetaS{t}{t}]\Big|\lesssim \frac{s_0}{m},
$$
where $m=\min_{1 \leq k \leq K} | \Gs{k} |$.
Guided by this approximation,  we  estimate  $\sigma^2_{tk}$ by
$$
\hat\sigma^2_{tk}=\hat\Theta_{t,k}^2 + \hat\Theta_{k,k}\hat\Theta_{t,t}.
$$
When all clusters have the equal size, we obtain $K=d/m$.  Thus the $O(\frac{s_0}{m})$ terms can be ignored asymptotically in the sense that $\frac{s_0}{m}=\frac{s_0K}{d}\leq \frac{K^2}{d}=o(1)$, when the clusters are approximately balanced, and their number  satisfies $K^2=o(d)$. This is a reasonable assumption in most applied clustering problems.  We note that the estimator $\hat\sigma^2_{tk}$ may be inconsistent when the size of some clusters is  too small.  However, we recall that our ultimate goal is to use these estimators for recovering the sparsity pattern of $\bThetaS$ under FDR control. To evaluate the sensitivity of our overall procedure  to the size of the smallest cluster, we conduct simulation studies in Section \ref{sec:numerical}. The results shows that the proposed method works well as soon as  $m>4$.

The following theorem gives  the Berry--Esseen normal approximation bound for the estimators of the entries of the precision matrix corresponding to the latent variable graph.

\begin{theorem}
\label{thm:theta_asymptotic}
If Assumptions \ref{asmp:bounded_latent_covariance} and \ref{asmp:bounded_errors} hold, then
\begin{equation}
\label{eqtheta_asymptotic2}
\max_{1\leq t< k\leq K}\sup_{x \in \RR}\Big|\PP(\hat T_{t,k} \leq x)-\Phi(x)\Big|\leq \frac{C}{(d \vee n)^3} + \frac{C}{(K\vee n)^3}+\frac{Cs_0\log (K\vee n)}{n^{1/2}}+\frac{Cs_0}{m},
\end{equation}
where $\hat T_{t,k} = \frac{n^{1/2}(\tilde\Theta_{t,k}-\Theta^*_{t,k})}{\hat\sigma_{tk}}$ and $C$ is a positive constant.
\end{theorem}

Compared to the average graph, the Berry-Esseen bound in (\ref{eqtheta_asymptotic2}) contains an additional $\cO(\frac{s_0}{m})$ term, stemming from the approximation of the analytically intractable asymptotic variance by an estimable quantity. The proof is deferred to Appendix \ref{sec:main_proofs_lvg}.

\subsection{Application to Post-clustering FDR Control }
\label{sec:main_results:fdr_control}
Given the edge-wise inferential results for the cluster-average  and latent variable graphs established above, we explain in this section  how to combine them to control graph-wise inferential uncertainty. Specifically, we view the task of recovering the sparsity pattern as a multiple testing problem by selecting:
\begin{align}\label{eq:multtest1}
    \Hb_{0;tk}: \Omega^*_{t,k} =0 \quad {\rm vs. } \quad \Hb_{1;tk}: \Omega^*_{t,k} \neq 0 \quad \text{for all } 1\leq t<k\leq K,
\end{align}
for the cluster-average graph, and
\begin{align}\label{eq:multtest2}
    \Hb_{0;tk}^{'}: \ThetaS{t}{k} =0 \quad {\rm vs. } \quad \Hb_{1;tk}^{'}: \ThetaS{t}{k} \neq 0 \quad \text{for all } 1\leq t<k\leq K.
\end{align}
for the latent variable graph. In the following, we apply the B-Y procedure by  \cite{Benjamini2001}  for FDR control in the cluster-average graph. The procedure for latent variable graph is identical. In this section, we are not claiming to develop a different FDR procedure; rather we make the simple point that if you can obtain asymptotic $p$-values of $d$ dependent statistics we can combine them in standard ways to control the FDR, which is a direct consequence of the Berry-Esseen bounds derived in Section \ref{sec_theorem}.

Define the set of true null hypotheses, $\cH_{0}\coloneqq \{(t,k):\, 1\leq t<k\leq K,\textrm{ such that } \Omega^*_{t,k} = 0\}$, as the set of indices $(t,k)$  for which there is no edge between the nodes $t$ and $k$. To control the error incurred by  multiple testing, we focus on the false discovery rate (FDR), which is the average number of Type I errors relative to the total number of discoveries  \citep{BH95}. Recall that $\tilde\Omega_{t,k}$ is a consistent and asymptotically normal estimator of $\Omega_{t,k}$.  We consider the natural test statistic $\tilde W_{t,k}=n^{1/2}\tilde\Omega_{t,k}/\hat s_{tk}$ for $\Hb_{0;tk}$, where $\hat s^2_{tk}=\hat\Omega_{t,k}^2 + \hat\Omega_{t,t}\hat\Omega_{k,k}$. Given a cutoff $\tau>0$, the total number of discoveries is
\[
R_{\tau} := \sum_{1\leq t < k \leq K} \II[|\tilde W_{t,k}| > \tau ].
\]
Similarly, the number of false positives or false discoveries is given by
\[
V_{\tau} := \sum_{(t, k) \in\cH_{0}} \II[|\tilde W_{t,k}| > \tau ].
\]
The FDR is formally defined as the expected ratio of $V_\tau$ over $R_{\tau}$,
\[
\FDR(\tau) := \EE\left[\frac{V_\tau}{R_\tau}\II[R_\tau > 0] \right],
\]
where the indicator function is included to remove the trivial case $R_\tau=0$.

Our goal is to find a data-dependent cutoff $\tau$ such that $\FDR(\tau)\leq \alpha+o(1)$ for any given $0 < \alpha < 1$.  This is the best one can hope for when, as in our case, the distribution of the test statistics $\tilde W_{t,k}$ is only available asymptotically.  The Berry-Esseen type bounds, derived in Theorems \ref{thm:xi_asymptotic} and \ref{thm:theta_asymptotic}, allow us to precisely quantify the price we must pay for the asymptotic approximation and are instrumental to understanding asymptotic FDR control.

In addition, the test statistics $\tilde W_{t,k}$ for different hypotheses are dependent. To allow for the dependence, instead of the standard B-H procedure \citep{BH95}, we consider the more flexible B-Y procedure by \cite{Benjamini2001}. The resulting FDR procedure is as follows: reject all hypotheses such that $|\tilde W_{t,k}| \geq \hat\tau$, where
\begin{equation}
\label{eqn:selection_rule_av}
\hat \tau := \min\left\{\tau > 0 : \tau \geq  \Phi^{-1}\left(1-\frac{\alpha R_{\tau}}{2N_{BY}|\cH|} \right) \right\} \text{ and } N_{BY} = \sum_{i=1}^{|\cH|} \frac{1}{i},
\end{equation}
where $|\cH|=K(K-1)/2$ is the total number of hypotheses.

Our next result shows when the FDR  based on our test statistics is guaranteed to be no greater than $\alpha$, asymptotically. The proofs can be found in Appendix \ref{sec:main_proofs_ca}.

\begin{theorem}
\label{thm:fdr_bound_av}~~
\begin{enumerate}
\item Assume that the conditions in Theorem \ref{thm:xi_asymptotic} hold. For any $0 < \alpha < 1$, we have
\begin{equation}
\label{eqfdr}
\FDR(\hat \tau) \leq \alpha +2 |\cH_0|b_n,
\end{equation}
where $b_n=\frac{C}{(d\vee n)^3}+\frac{C}{(K\vee n)^3}+\frac{Cs_1\log (K\vee n)}{n^{1/2}}$, and $|\cH_0|$ is the number of true null hypotheses in (\ref{eq:multtest1}).
\item Assume that the conditions in Theorem \ref{thm:theta_asymptotic} hold. If we define the test statistic as $\tilde W_{t,k}=n^{1/2}\tilde\Theta_{t,k}/\hat \sigma_{tk}$, and $\hat\tau$ as in (\ref{eqn:selection_rule_av}), we have
\begin{equation}
\label{eqfdr_av}
\FDR(\hat \tau) \leq \alpha +2 |\cH'_0|c_n,
\end{equation}
where $c_n=\frac{C}{(d\vee n)^3}+\frac{C}{(K\vee n)^3}+\frac{Cs_0\log (K\vee n)}{n^{1/2}}+\frac{Cs_0}{m}$, and $|\cH'_0|$ is the number of true null hypotheses in (\ref{eq:multtest2}).
\end{enumerate}
\end{theorem}

This theorem implies that our method can control  the FDR asymptotically, in the sense that $\FDR(\hat\tau)\leq \alpha+o_p(1)$, provided  that $s_1|\cH_0|\log (K\vee n)=o(n^{1/2})$, for the average graph, and  $s_0|\cH^{'}_0|\log (K\vee n)=o(n^{1/2})$ for the latent graph.   Gaussian graphical model  estimation under FDR control was   recently studied by \cite{liu2013gaussian}. They showed that the B-H procedure can control FDR asymptotically under certain conditions. Their approach is based on the following Cramer-type moderate deviation result using our terminology,
\begin{equation}
\label{eqdeviation}
\max_{(t,k)\in\cH_0}\sup_{0\leq t\leq 2\sqrt{\log K}} \Big|\frac{\PP(\hat T_{t,k}\geq t)}{2-2\Phi(t)}-1\Big|=o(1),
\end{equation}
where $\hat T_{t,k}$ is test statistic they proposed for estimation  of the  Gaussian graphical model structure. The result \eqref{eqdeviation} controls the relative error of the Gaussian approximation within the moderate deviation regime $[0, 2\sqrt{\log K}]$, whereas our result is based on the control of the absolute error via the Berry-Esseen-type Gaussian approximation.  One of the main advantages of their result  is that the number of clusters is allowed to be  $K=o(n^r)$, where $r$ is a constant that can be greater than 1. However, to prove \eqref{eqdeviation}, they required that the number of strong signals tends to infinity, that is  $|\{(t,k): \Omega_{t,k}^*/\sqrt{\Omega_{k,k}^*\Omega_{t,t}^*}\geq C\sqrt{\log K/n}\}|\rightarrow\infty$, which reduces significantly the parameter space for which inference is valid. In contrast, the  aim of this work is the study of pattern recovery without conditions on the signal strength of the entries of the target precision matrices, as  in practice it is difficult to assess whether these conditions are met. For completeness, we provide the detailed analysis of the B-H procedure including technical conditions, theoretical results and further discussion of the B-H procedure in Appendix \ref{app:FDR}.

The overall message conveyed by Theorem \ref{thm:fdr_bound_av} is that, in the absence of any signal strength assumptions,  cluster-based graphical models can still be  recovered, under FDR control,  provided that  the number of clusters $K$ is not very high relative to $n$, and provided that the clusters are not very small. This further stresses the importance of an initial dimension reduction step in high-dimensional graphical model estimation. For instance, results similar to those of  Theorem \ref{thm:xi_asymptotic}  can be derived along the same lines for the estimation of the sparsity pattern of $\bSigma^{-1}$, for a generic, unstructured, covariance matrix of $\bX$, where one replaces $K$ by $d$ throughout, and $s_1$ is replaced by $s$, the number of non-zero entries in the $d \times d$ matrix  $\bSigma^{-1}$. Then, the analogue of  \eqref{eqfdr_av} of Theorem \ref{thm:fdr_bound_av} shows that FDR control in generic graphical models, based on {\it asymptotic approximations of $p$-values}, cannot generally be guaranteed if $d > n$. Our work shows that extra structural  assumptions, for instance those motivated by clustering, do alleviate this problem.  The simulation study presented in the next section provides further support to our findings.

\section{Numerical Results}
\label{sec:numerical}
This section contains simulations and a real data analysis that  illustrate the finite sample performance of the inferential procedures  developed in the previous sections for the latent variable graph and cluster-average graph, respectively.

\subsection{Synthetic Datasets}
\label{sec:numerical:sim}
In this subsection, we demonstrate the effectiveness of the FDR control procedures on synthetic datasets. We consider two settings  $(n,d)=(800,400)$ and $(500,1000)$, and in each setting we vary the value of $K$ and $m$. The error variable $\bE$ is sampled from the multivariate normal distribution with covariance $\bGamma^*$ whose entries $\gammaS{i}$ are generated from $U[0.25,0.5]$. Recall that the latent variable $\bZ$ follows from $\bZ\sim \cN(0,\bCS)$. We consider three different models to generate the graph structure of $\bZ$. Once the graph structure is determined, the corresponding adjacency matrix $\Wb$ is found, and the precision matrix $\bThetaS=\Cb^{*-1}$ is taken as $\bThetaS = c\Wb + (|\lmin{\Wb}| + 0.2)\Ib$, where $c=0.3$ when $d=400$ and $c=0.5$ when $d=1000$. Finally, we assign the cluster labels for all variables so that all clusters have approximately equal size, which gives us the matrix $\Ab$. Given $\Ab, \bZ$ and $\bE$, we can generate $\bX$ according to the model  (\ref{eqn:g_latent_model}).

We consider the following three generating models  for the graph structure of $\bZ$:
\begin{itemize}
\item \noindent{\it Scale-Free Graph -- }
The Scale-Free model is a generative model for network data, whose degree distribution follows a power law. To be concrete, we generate the graph one node at a time, starting with a 2 node chain. For nodes $3,\dots,K$, node $t$ is added and one edge is added between $t$ and one of the $t-1$ previous nodes. Denoting by $k_i$ the current degree of node $i$ in the graph, the probability that node $t$ and node $i$ are connected is $p_i = k_i / (\sum_i k_i)$. The number of edges in the resulting graph is always $K$.
\item \noindent{\it Hub Graph -- }
The $K$ nodes of the graph are partitioned evenly into groups of size $N$. Within each group, one node is selected to be the group hub, and an edge is added between it and the remainder of its group. $N$ is either 5 or 6 depending upon the choice of $K$. The number of edges in the graph is $K(N-1)/N$, so for $K=100$ with $N=5$, the number of edges in the resulting graph is 80.
\item \noindent{\it Band3 Graph -- }
This model generates a graph with a Toeplitz adjacency matrix. There is an edge between node $i$ and node $j$ if and only if $|i - j| \leq B$, where we set $B=3$ in this scenario. In general, the number of edges in a band graph with $K$ nodes is given by $BK - \frac{3}{2}B^2 + \frac{5}{2}B$. So, for $K=100$ and $B=3$, the number of edges in the graph is 294.
\end{itemize}

Recall that $\bar \bX \sim \cN(0,\bEssS)$, where $\bEssS$ is defined in (\ref{eqn:s_star_definition}). To determine the structure of the average graph, we numerically compute $\bS^{*-1}$ and threshold the matrix at $10^{-8}$.

We examine the empirical FDR of our procedures on some synthetic datasets. The following protocol is followed in all the experiments:
\begin{enumerate}
\item Generate the graph structure of $\bZ$ as specified above.
\item Simulate $n$ observations from our model (\ref{eqn:g_latent_model}).
\item Estimate the cluster partition $\hat G$. For computational convenience, we apply the FORCE algorithm \citep{Eisenach2019b} with known $K$.
\item Construct the test statistic $\tilde W_{t,k}$ defined in Section \ref{sec:main_results:fdr_control}. The regularization parameters $\lambda$ and $\lambda'$ are chosen by 5-fold cross validation.
\item Find the FDR cutoff (\ref{eqn:selection_rule_av}) at level $\alpha$; we consider three cases $\alpha=0.05, 0.1, 0.2$.
\end{enumerate}

The simulation is repeated 100 times. To compare with our Benjamini-Yekutieli based FDR procedure, we also report the empirical FDR based on the more classical Benjamini-Hochberg procedure. That is, we apply the same procedures 1-4, but in step 5 we replace the FDR cutoff in (\ref{eqn:selection_rule_av}) with the Benjamini-Hochberg (B-H) cutoff, i.e.,
$$
\hat \tau_{BH} := \min\left\{\tau > 0 : \tau \geq  \Phi^{-1}\left(1-\frac{\alpha R_{\tau}}{2|\cH|} \right) \right\}.
$$
Table \ref{fig:synth_fdr} compares the empirical FDR based on our method with the B-H procedure under different $m, K$ settings when $d=400$. When $m$ is relatively large (e.g., $m=20$), both methods can control FDR on average, although our method is relatively more conservative. As expected, the FDR control problem becomes more challenging for large $K$ and small $m$. In this case the graph contains more nodes and each cluster contains fewer variables. We observe that when $m=5$ our method can still control FDR reasonably well but the B-H method produces empirical FDR far beyond the nominal level, especially for hub graphs. The inferior performance of the B-H procedure is due to the fact that the dependence among the test statistics is not accounted for, demonstrating that the B-Y procedure is indeed necessary at least in the current simulation settings. Finally, we examine the empirical power of the FDR procedure under each scenario, which is defined as
\[
\textrm{Average}\left[\sum_{(t,k) \in \cH_1}\frac{\II[\tilde W_{t,k} \geq \hat \tau]}{|\cH_1|} \right],
\]
where $\cH_1$ is the set of alternative hypothesis. Table \ref{fig:synth_fdr_pow} gives the empirical power of our FDR procedure, and the B-H procedure when $d=400$. It shows that our procedure and the B-H procedure have very high power in all scenarios. The same phenomenon is observed when $d$ is large, i.e., $d=1000$; see Tables \ref{fig:synth_fdr_highd} and \ref{fig:synth_fdr_pow_highd}. In summary, our proposed procedure can identify most of the signals in the graph while keeping FDR well controlled.

\begin{remark}[Inexact Recovery]
  While our FORCE algorithm guarantees the estimated number of clusters is always $K$ since we use the true value $K$ as the input, the recovered partition $\hat G$ may still differ from the true partition $G^*$. In order to compare our results to the ground-truth, we need to find an ``aligned`` version of $\bThetaS$ (or $\bOmegaS$). Specifically, we calculate the mapping $f:[K]\rightarrow[K]$ defined by
  \[
  f(k) = \argmax_{l\in[K]} |\Gh{k} \cap \Gs{l}|,
  \]
  and then construct the matrix $\bTheta^*_A$, defined entry-wise by
  \[
  (\Theta^*_A)_{s,t} := (\Theta^*)_{f(s),f(t)}.
  \]
Once we have obtained the ``aligned" ground-truth matrix $\bTheta_A^*$, we can proceed with computing the metric of interest (FDR and power). Although the preceding discussion is in terms of $\bThetaS$, the same procedure applies to $\bOmega^*$.
    
\end{remark}

\begin{remark}[Discussion of the B-H procedure]
  Based on the results in Tables \ref{fig:synth_fdr_highd} and \ref{fig:synth_fdr_pow_highd}, it is clear that the power of both our procedure and the B-H procedure are satisfactory. However, for both setups ($d=400,1000)$, the B-H procedure leads to the inflated FDR relative to the nominal level. While Appendix \ref{app:FDR} shows that the B-H procedure can control FDR asymptotically under certain conditions, it seems in numerical examples the dependence among the test statistics leads to substantial errors in FDR control that are indeed not ignorable. 
\end{remark}

\begin{table}[h]
\centering
{\footnotesize %
\begin{tabular}{c c c c l l l l l l l l l}
\toprule
& & & & \multicolumn{3}{c}{$\mathbf{\balpha = 0.05}$} & \multicolumn{3}{c}{$\mathbf{\balpha = 0.1}$} & \multicolumn{3}{c}{$\mathbf{\balpha = 0.2}$}\\
\cline{5-13}
& & $K$ & $m$ & Scalefree & Band3 & Hub & Scalefree & Band3 & Hub & Scalefree & Band3 & Hub \\
\multirow{8}{*}{\begin{sideways}{\it B-Y Based Procedure}\end{sideways}} & \multirow{4}{*}{\begin{sideways}Latent\end{sideways}}
  & 80 & 5 & 1.16\% & 1.42\% & 5.99\% & 2.01\% & 2.64\% & 7.60\% & 4.01\% & 4.73\% & 10.33\% \\
& & 66 & 6 & 0.93\% & 1.03\% & 1.08\% & 1.99\% & 1.98\% & 1.96\% & 3.73\% & 3.68\% & 4.00\% \\
& & 50 & 8 & 1.16\% & 0.99\% & 1.09\% & 2.20\% & 1.77\% & 1.98\% & 3.90\% & 3.49\% & 3.73\% \\
& & 20 & 20 & 0.99\% & 0.88\% & 1.30\% & 1.81\% & 1.66\% & 2.26\% & 3.75\% & 3.38\% & 4.71\% \\
\cline{2-13}
 & \multirow{4}{*}{\begin{sideways}Grp. Av.\end{sideways}}
  & 80 & 5 & 1.15\% & 1.40\% & 6.29\% & 2.14\% & 2.67\% & 7.89\% & 4.00\% & 4.70\% & 10.43\% \\
& & 66 & 6 & 0.91\% & 1.04\% & 1.04\% & 1.85\% & 2.00\% & 2.14\% & 3.70\% & 3.71\% & 3.83\% \\
& & 50 & 8 & 1.16\% & 1.00\% & 1.11\% & 2.22\% & 1.80\% & 1.94\% & 3.98\% & 3.49\% & 3.61\% \\
& & 20 & 20 & 0.94\% & 0.88\% & 1.23\% & 1.81\% & 1.68\% & 2.26\% & 3.65\% & 3.38\% & 4.71\% \\
\hline \hline
\multirow{8}{*}{\begin{sideways}{\it  B-H Based Procedure}\end{sideways}} & \multirow{4}{*}{\begin{sideways}Latent\end{sideways}}
  & 80 & 5 & 8.23\% & 9.24\% & 15.60\% & 15.01\% & 16.39\% & 24.00\% & 28.16\% & 28.97\% & 38.58\% \\
& & 66 & 6 & 7.18\% & 7.38\% & 7.31\% & 14.12\% & 14.01\% & 13.78\% & 26.56\% & 25.84\% & 27.28\% \\
& & 50 & 8 & 6.94\% & 6.75\% & 6.69\% & 13.23\% & 12.73\% & 12.74\% & 25.20\% & 23.62\% & 26.05\% \\
& & 20 & 20 & 5.43\% & 4.54\% & 6.38\% & 11.09\% & 8.47\% & 10.71\% & 21.46\% & 17.15\% & 20.75\% \\
\cline{2-13}
 & \multirow{4}{*}{\begin{sideways}Grp. Av.\end{sideways}}
  & 80 & 5 & 8.42\% & 9.26\% & 15.77\% & 15.25\% & 16.43\% & 24.15\% & 28.03\% & 29.02\% & 38.57\% \\
& & 66 & 6 & 7.21\% & 7.37\% & 7.45\% & 13.99\% & 13.92\% & 13.81\% & 26.19\% & 25.77\% & 27.17\% \\
& & 50 & 8 & 6.82\% & 6.78\% & 6.67\% & 13.23\% & 12.67\% & 12.89\% & 25.10\% & 23.61\% & 25.91\% \\
& & 20 & 20 & 5.38\% & 4.51\% & 6.49\% & 11.01\% & 8.52\% & 10.66\% & 21.58\% & 17.09\% & 20.71\% \\
\bottomrule
\end{tabular}

}%
\caption{Averaged empirical FDR for synthetic data experiments with $d=400$ and $n=800$.  \label{fig:synth_fdr}}
\end{table}

\begin{table}[h]
\centering
{\footnotesize %
\begin{tabular}{c c c c l l l l l l l l l}
\toprule
& & & & \multicolumn{3}{c}{$\mathbf{\balpha = 0.05}$} & \multicolumn{3}{c}{$\mathbf{\balpha = 0.1}$} & \multicolumn{3}{c}{$\mathbf{\balpha = 0.2}$}\\
\cline{5-13}
& & $K$ & $m$ & Scalefree & Band3 & Hub & Scalefree & Band3 & Hub & Scalefree & Band3 & Hub \\
\multirow{8}{*}{\begin{sideways}{\it B-Y Based Procedure}\end{sideways}} & \multirow{4}{*}{\begin{sideways}Latent\end{sideways}}
  & 80 & 5 & 88.22\% & 99.99\% & 99.00\% & 91.10\% & 100.00\% & 99.04\% & 93.38\% & 100.00\% & 99.11\% \\
& & 66 & 6 & 93.62\% & 100\% & 100\% & 95.42\% & 100\% & 100\% & 96.86\% & 100\% & 100\% \\
& & 50 & 8 & 97.18\% & 100\% & 100\% & 97.90\% & 100\% & 100\% & 98.47\% & 100\% & 100\% \\
& & 20 & 20 & 100\% & 100\% & 100\% & 100\% & 100\% & 100\% & 100\% & 100\% & 100\% \\
\cline{2-13}
 & \multirow{4}{*}{\begin{sideways}Grp. Av.\end{sideways}}
  & 80 & 5 & 88.27\% & 99.99\% & 98.93\% & 91.06\% & 100.00\% & 98.98\% & 93.37\% & 100.00\% & 99.07\% \\
& & 66 & 6 & 93.62\% & 100\% & 100\% & 95.40\% & 100\% & 100\% & 96.91\% & 100\% & 100\% \\
& & 50 & 8 & 97.18\% & 100\% & 100\% & 97.86\% & 100\% & 100\% & 98.47\% & 100\% & 100\% \\
& & 20 & 20 & 100\% & 100\% & 100\% & 100\% & 100\% & 100\% & 100\% & 100\% & 100\% \\
\hline \hline
\multirow{8}{*}{\begin{sideways}{\it  B-H Based Procedure}\end{sideways}} & \multirow{4}{*}{\begin{sideways}Latent\end{sideways}}
  & 80 & 5 & 95.75\% & 100.00\% & 99.18\% & 97.15\% & 100.00\% & 99.23\% & 98.48\% & 100.00\% & 99.26\% \\
& & 66 & 6 & 97.85\% & 100\% & 100\% & 98.78\% & 100\% & 100\% & 99.34\% & 100\% & 100\% \\
& & 50 & 8 & 99.27\% & 100\% & 100\% & 99.57\% & 100\% & 100\% & 99.80\% & 100\% & 100\% \\
& & 20 & 20 & 100\% & 100\% & 100\% & 100\% & 100\% & 100\% & 100\% & 100\% & 100\% \\
\cline{2-13}
 & \multirow{4}{*}{\begin{sideways}Grp. Av.\end{sideways}}
  & 80 & 5 & 95.82\% & 100.00\% & 99.14\% & 97.14\% & 100.00\% & 99.18\% & 98.53\% & 100.00\% & 99.23\% \\
& & 66 & 6 & 97.88\% & 100\% & 100\% & 98.78\% & 100\% & 100\% & 99.34\% & 100\% & 100\% \\
& & 50 & 8 & 99.24\% & 100\% & 100\% & 99.59\% & 100\% & 100\% & 99.78\% & 100\% & 100\% \\
& & 20 & 20 & 100\% & 100\% & 100\% & 100\% & 100\% & 100\% & 100\% & 100\% & 100\% \\
\bottomrule
\end{tabular}

}%
\caption{Averaged FDR power for synthetic data experiments with $d=400$ and $n=800$. \label{fig:synth_fdr_pow}}
\end{table}

\begin{table}[h]
\centering
{\footnotesize %
\begin{tabular}{c c c c l l l l l l l l l}
\toprule
& & & & \multicolumn{3}{c}{$\mathbf{\balpha = 0.05}$} & \multicolumn{3}{c}{$\mathbf{\balpha = 0.1}$} & \multicolumn{3}{c}{$\mathbf{\balpha = 0.2}$}\\
\cmidrule{5-13}
& & $K$ & $m$ & Scalefree & Band3 & Hub & Scalefree & Band3 & Hub & Scalefree & Band3 & Hub \\
\multirow{8}{*}{\begin{sideways}{\it B-Y Based Procedure}\end{sideways}} & \multirow{4}{*}{\begin{sideways}Latent\end{sideways}}
  & 80 & 12 & 1.40\% & 1.52\% & 1.78\% & 2.78\% & 2.81\% & 3.29\% & 5.71\% & 5.09\% & 6.30\% \\
& & 66 & 15 & 1.23\% & 1.29\% & 1.31\% & 2.62\% & 2.39\% & 2.69\% & 5.08\% & 4.33\% & 4.84\% \\
& & 50 & 20 & 1.05\% & 0.94\% & 1.26\% & 2.27\% & 1.85\% & 2.79\% & 4.28\% & 3.65\% & 4.58\% \\
& & 20 & 50 & 1.04\% & 1.03\% & 1.05\% & 1.81\% & 1.96\% & 2.20\% & 3.46\% & 3.26\% & 4.19\% \\
\cmidrule{2-13}
 & \multirow{4}{*}{\begin{sideways}Grp. Av.\end{sideways}}
  & 80 & 12 & 1.55\% & 1.57\% & 1.78\% & 2.78\% & 2.80\% & 3.26\% & 5.57\% & 5.13\% & 6.13\% \\
& & 66 & 15 & 1.30\% & 1.30\% & 1.24\% & 2.64\% & 2.35\% & 2.64\% & 5.07\% & 4.34\% & 4.89\% \\
& & 50 & 20 & 1.09\% & 0.93\% & 1.26\% & 2.32\% & 1.89\% & 2.75\% & 4.29\% & 3.64\% & 4.51\% \\
& & 20 & 50 & 1.09\% & 1.01\% & 1.05\% & 1.91\% & 1.96\% & 2.14\% & 3.50\% & 3.28\% & 4.25\% \\
\hline \hline
\multirow{8}{*}{\begin{sideways}{\it  B-H Based Procedure}\end{sideways}} & \multirow{4}{*}{\begin{sideways}Latent\end{sideways}}
  & 80 & 12 & 10.99\% & 10.04\% & 12.00\% & 20.04\% & 18.18\% & 21.47\% & 35.50\% & 31.46\% & 37.24\% \\
& & 66 & 15 & 9.28\% & 8.20\% & 9.14\% & 16.88\% & 14.94\% & 16.82\% & 31.53\% & 27.75\% & 31.81\% \\
& & 50 & 20 & 7.73\% & 6.91\% & 7.96\% & 14.74\% & 13.13\% & 15.33\% & 27.03\% & 23.92\% & 28.72\% \\
& & 20 & 50 & 5.00\% & 4.85\% & 6.38\% & 11.01\% & 9.38\% & 11.60\% & 21.49\% & 17.00\% & 23.15\% \\
\cmidrule{2-13}
 & \multirow{4}{*}{\begin{sideways}Grp. Av.\end{sideways}}
  & 80 & 12 & 10.99\% & 10.00\% & 12.00\% & 20.29\% & 18.20\% & 21.38\% & 35.56\% & 31.48\% & 37.10\% \\
& & 66 & 15 & 9.28\% & 8.24\% & 8.99\% & 17.03\% & 14.87\% & 16.87\% & 31.49\% & 27.82\% & 31.63\% \\
& & 50 & 20 & 7.60\% & 6.93\% & 7.94\% & 14.61\% & 13.13\% & 15.40\% & 27.12\% & 23.94\% & 28.71\% \\
& & 20 & 50 & 4.95\% & 4.91\% & 6.38\% & 10.97\% & 9.35\% & 11.60\% & 21.36\% & 16.94\% & 23.41\% \\
\bottomrule
\end{tabular}

}%
\caption{Averaged empirical FDR for synthetic data experiments with $d=1000$ and $n=500$. \label{fig:synth_fdr_highd}}
\end{table}

\begin{table}[h]
\centering
{\footnotesize %
\begin{tabular}{c c c c l l l l l l l l l}
\toprule
& & & & \multicolumn{3}{c}{$\mathbf{\balpha = 0.05}$} & \multicolumn{3}{c}{$\mathbf{\balpha = 0.1}$} & \multicolumn{3}{c}{$\mathbf{\balpha = 0.2}$}\\
\cmidrule{5-13}
& & $K$ & $m$ & Scalefree & Band3 & Hub & Scalefree & Band3 & Hub & Scalefree & Band3 & Hub \\
\multirow{8}{*}{\begin{sideways}{\it B-Y Based Procedure}\end{sideways}} & \multirow{4}{*}{\begin{sideways}Latent\end{sideways}}
  & 80 & 12 & 71.42\% & 99.95\% & 100.00\% & 77.14\% & 99.97\% & 100.00\% & 82.29\% & 99.98\% & 100.00\% \\
& & 66 & 15 & 82.98\% & 99.98\% & 100.00\% & 86.32\% & 99.99\% & 100.00\% & 90.02\% & 100.00\% & 100.00\% \\
& & 50 & 20 & 90.45\% & 99.99\% & 100.00\% & 93.04\% & 99.99\% & 100.00\% & 94.86\% & 100.00\% & 100.00\% \\
& & 20 & 50 & 99.95\% & 100.00\% & 100.00\% & 100.00\% & 100.00\% & 100.00\% & 100.00\% & 100.00\% & 100.00\% \\
\cmidrule{2-13}
  & \multirow{4}{*}{\begin{sideways}Grp. Av.\end{sideways}}
  & 80 & 12 & 71.54\% & 99.95\% & 100.00\% & 77.13\% & 99.97\% & 100.00\% & 82.25\% & 99.98\% & 100.00\% \\
& & 66 & 15 & 82.74\% & 99.98\% & 100.00\% & 86.26\% & 99.99\% & 100.00\% & 90.08\% & 100.00\% & 100.00\% \\
& & 50 & 20 & 90.45\% & 99.99\% & 100.00\% & 93.00\% & 99.99\% & 100.00\% & 94.76\% & 100.00\% & 100.00\% \\
& & 20 & 50 & 99.95\% & 100.00\% & 100.00\% & 100.00\% & 100.00\% & 100.00\% & 100.00\% & 100.00\% & 100.00\% \\
\hline \hline
\multirow{8}{*}{\begin{sideways}{\it  B-H Based Procedure}\end{sideways}} & \multirow{4}{*}{\begin{sideways}Latent\end{sideways}}
  & 80 & 12 & 87.15\% & 99.99\% & 100.00\% & 91.33\% & 100.00\% & 100.00\% & 94.66\% & 100.00\% & 100.00\% \\
& & 66 & 15 & 93.57\% & 100.00\% & 100.00\% & 95.63\% & 100.00\% & 100.00\% & 97.38\% & 100.00\% & 100.00\% \\
& & 50 & 20 & 96.27\% & 100.00\% & 100.00\% & 97.78\% & 100.00\% & 100.00\% & 98.94\% & 100.00\% & 100.00\% \\
& & 20 & 50 & 100.00\% & 100.00\% & 100.00\% & 100.00\% & 100.00\% & 100.00\% & 100.00\% & 100.00\% & 100.00\% \\
\cmidrule{2-13}
  & \multirow{4}{*}{\begin{sideways}Grp. Av.\end{sideways}}
  & 80 & 12 & 87.22\% & 99.99\% & 100.00\% & 91.33\% & 100.00\% & 100.00\% & 94.58\% & 100.00\% & 100.00\% \\
& & 66 & 15 & 93.52\% & 100.00\% & 100.00\% & 95.65\% & 100.00\% & 100.00\% & 97.35\% & 100.00\% & 100.00\% \\
& & 50 & 20 & 96.33\% & 100.00\% & 100.00\% & 97.78\% & 100.00\% & 100.00\% & 98.94\% & 100.00\% & 100.00\% \\
& & 20 & 50 & 100.00\% & 100.00\% & 100.00\% & 100.00\% & 100.00\% & 100.00\% & 100.00\% & 100.00\% & 100.00\% \\
\bottomrule
\end{tabular}

}%
\caption{Averaged FDR power for synthetic data experiments with $d=1000$ and $n=500$. \label{fig:synth_fdr_pow_highd}}
\end{table}

\subsection*{Group Average Procedures Do Not Recover Latent Graphs}

In this section we demonstrate through simulation studies that procedures specifically tailored to recovering the latent variable graph are necessary. We do this by using both the methodology for the latent variable graph and the group-average graph to recover the latent graph structure. Because the differences between the latent and group-average graph can be small, we use larger sample sizes than in the previous studies.

The experimental procedure is almost exactly the same as before, but now we hold most of the parameters of the generating distribution fixed and examine the effects of the error variance $\bGammaS$ and sample size $n$ on the efficacy of our methodologies for recovering the latent variable graph. In all the experiments, we impose a Band3 structure on the latent graph, set $\bGammaS = \gamma\Ib$, and use $(d,K,m) = (400,20,20)$.

First we note that if $\bGammaS = \bzero$, the two graph structures are actually the same -- it is only as we introduce error into the observed variables that the latent variable and group-average graphs begin to differ.  In Tables \ref{fig:synth_t1_wrongtest} and \ref{fig:synth_fdr_wrongtest} we see that as we increase the error variance, the the group averages methodology applied to the latent graph is unable to control the Type I error rate and the FDR at the desired level. Further, by increasing the sample size $n$, we observe that the performance of the group average procedures applied to recovering the latent graph gets worse -- this is as expected because with increased $n$, the tests become accurate, but are by construction estimating the wrong graph!

\begin{table}[h]
\centering
{\footnotesize %
\begin{tabular}{c c l l l l l l l l}
\toprule
 & & \multicolumn{4}{c}{\bf L.V. Test Procedure} & \multicolumn{4}{c}{\bf G.A. Test Procedure} \\
\cline{3-10}
$n$ & $\gamma$ & $\alpha=0.01$ & $\alpha=0.05$ & $\alpha=0.1$ & $\alpha=0.2$ & $\alpha=0.01$ & $\alpha=0.05$ & $\alpha=0.1$ & $\alpha=0.2$ \\
\midrule
1600 & 0.5 & 0.56\% & 2.67\% & 5.83\% & 10.22\% & 0.61\% & 3.00\% & 6.39\% & 10.39\%  \\
& 1.0 & 0.78\% & 3.11\% & 5.22\% & 10.61\% & 0.78\% & 3.83\% & 6.89\% & 13.39\%  \\
& 1.5 & 0.61\% & 2.89\% & 5.11\% & 8.72\% & 0.94\% & 3.33\% & 6.22\% & 10.50\%  \\
& 2.0 & 1.00\% & 3.39\% & 6.39\% & 11.06\% & 1.39\% & 4.78\% & 7.61\% & 12.44\%  \\
& 2.5 & 0.56\% & 3.50\% & 6.17\% & 11.22\% & 1.89\% & 5.33\% & 9.17\% & 15.28\%  \\
& 3.0 & 1.72\% & 5.11\% & 8.61\% & 14.28\% & 3.00\% & 8.56\% & 11.89\% & 17.61\%  \\
& 3.5 & 1.06\% & 3.83\% & 6.78\% & 11.72\% & 3.39\% & 8.44\% & 12.94\% & 18.44\%  \\
& 4.0 & 1.50\% & 4.33\% & 8.33\% & 12.83\% & 5.22\% & 10.94\% & 14.22\% & 20.78\%  \\
\midrule
3200 & 0.5 & 0.28\% & 1.89\% & 3.22\% & 7.50\% & 0.44\% & 2.17\% & 4.00\% & 8.11\%  \\
& 1.0 & 0.44\% & 3.00\% & 5.11\% & 9.44\% & 1.11\% & 4.22\% & 7.83\% & 14.17\%  \\
& 1.5 & 0.44\% & 2.72\% & 4.50\% & 9.56\% & 1.78\% & 5.61\% & 9.00\% & 13.61\%  \\
& 2.0 & 0.89\% & 3.56\% & 6.78\% & 11.61\% & 2.94\% & 7.50\% & 11.33\% & 16.78\%  \\
& 2.5 & 1.06\% & 3.61\% & 6.06\% & 10.67\% & 2.78\% & 9.67\% & 13.44\% & 19.00\%  \\
& 3.0 & 1.50\% & 5.44\% & 8.39\% & 14.33\% & 7.72\% & 13.78\% & 18.00\% & 22.78\%  \\
& 3.5 & 1.67\% & 5.44\% & 9.72\% & 15.72\% & 7.61\% & 13.83\% & 18.22\% & 22.89\%  \\
& 4.0 & 2.78\% & 5.83\% & 10.06\% & 16.28\% & 11.28\% & 17.11\% & 22.22\% & 27.06\%  \\
\bottomrule
\end{tabular}

}%
\caption{Averaged Type I for the latent variable graph using both the latent variable and group averages methodology. \label{fig:synth_t1_wrongtest}}
\end{table}

\begin{table}[h]
\centering
{\footnotesize %
\begin{tabular}{c c l l l l l l l l}
\toprule
 & & \multicolumn{4}{c}{\bf L.V. Test Procedure} & \multicolumn{4}{c}{\bf G.A. Test Procedure} \\
\cline{3-10}
$n$ & $\gamma$ & $\alpha=0.01$ & $\alpha=0.05$ & $\alpha=0.1$ & $\alpha=0.2$ & $\alpha=0.01$ & $\alpha=0.05$ & $\alpha=0.1$ & $\alpha=0.2$ \\
\midrule
1600 & 0.5 & 0.83\% & 3.81\% & 4.95\% & 11.44\% & 0.41\% & 3.23\% & 5.51\% & 12.57\%  \\
& 1.0 & 0.62\% & 3.23\% & 6.07\% & 10.11\% & 1.03\% & 4.00\% & 7.51\% & 14.13\%  \\
& 1.5 & 0.21\% & 3.61\% & 5.88\% & 10.61\% & 1.23\% & 4.76\% & 6.25\% & 12.25\%  \\
& 2.0 & 1.23\% & 3.81\% & 6.43\% & 13.51\% & 1.64\% & 5.88\% & 9.26\% & 14.74\%  \\
& 2.5 & 0.83\% & 4.00\% & 7.16\% & 12.25\% & 3.23\% & 5.88\% & 10.95\% & 16.52\%  \\
& 3.0 & 2.24\% & 6.61\% & 10.11\% & 16.81\% & 3.23\% & 9.94\% & 15.19\% & 20.66\%  \\
& 3.5 & 1.23\% & 4.95\% & 6.98\% & 13.36\% & 3.81\% & 10.78\% & 15.79\% & 22.08\%  \\
& 4.0 & 1.64\% & 4.95\% & 8.40\% & 15.19\% & 6.25\% & 14.29\% & 17.95\% & 24.88\%  \\
\midrule
3200 & 0.5 & 0.41\% & 2.44\% & 4.00\% & 7.51\% & 0.41\% & 2.64\% & 4.19\% & 8.75\%  \\
& 1.0 & 0.83\% & 3.42\% & 6.80\% & 10.61\% & 1.03\% & 4.57\% & 8.05\% & 16.38\%  \\
& 1.5 & 0.41\% & 3.23\% & 5.70\% & 9.09\% & 2.64\% & 7.34\% & 10.45\% & 16.81\%  \\
& 2.0 & 0.62\% & 4.38\% & 6.98\% & 12.25\% & 4.00\% & 9.77\% & 13.36\% & 20.27\%  \\
& 2.5 & 1.03\% & 4.57\% & 7.51\% & 11.60\% & 4.00\% & 12.09\% & 16.81\% & 23.44\%  \\
& 3.0 & 1.84\% & 7.51\% & 10.45\% & 16.38\% & 9.94\% & 17.81\% & 22.95\% & 28.14\%  \\
& 3.5 & 1.64\% & 6.98\% & 11.28\% & 18.92\% & 10.95\% & 18.23\% & 23.08\% & 27.93\%  \\
& 4.0 & 4.38\% & 8.57\% & 11.60\% & 18.92\% & 15.94\% & 22.95\% & 27.16\% & 31.82\%  \\
\bottomrule
\end{tabular}

}%
\caption{Averaged empirical FDR error rates for the latent variable graph using both the latent variable and group averages methodology. \label{fig:synth_fdr_wrongtest}}
\end{table}

\subsection{fMRI Dataset}
The study of brain relationships in humans via modern neuroimaging techniques has attracted an enormous amount of scientific interest in recent years. One fundamental goal in these studies is to understand the functional communication between brain regions, which plays a key role in
complex cognitive processes. To study the functional connectivity network, \cite{Power2011} partitioned the human brain into regions of interest (ROIs) and represented each ROI by a node in a graph. They identified several ``subgraphs" of highly related nodes (i.e, clusters), which can represent the major functional systems of the brain. One of the main goals in their work is to understand the relationship between different subgraphs or clusters. To this end, they estimated the relationship between clusters by thresholding the correlation matrix of all nodes in the same clusters in an adhoc way. The analysis of ``network of clusters" can be conducted under a rigorous statistical framework by using the proposed cluster-based graphical models. In the section, we will apply the proposed method to study the dependence structure among functional systems of the brain. 




As an illustration, we focus on the publicly available resting-state fMRI data from the Neuro-bureau pre-processed repository \citep{Bellec2015}. Specifically, we use the data from patient 1018959, session 1 in the KKI dataset. This fMRI dataset was pre-processed using the Athena pipeline and mapped to T1 MNI152 coordinate space. We choose this dataset to make our experiments easily reproducible, as the data are available pre-processed using standard alignment and denoising procedures. In a recent work, \cite{luo2014hierarchical} estimated the fMRI network in a similar clustering model by an iterative procedure that maximizes the likelihood function with respect to the unknown cluster assignment and the precision matrix. However, their optimization problem is non-convex due to estimating the unknown cluster assignment and the algorithm is computationally intensive in high dimension. 

Following \cite{Power2011}, we directly extract the 264 ROIs, which gives us $d=264$ mean activities across $n=148$ time periods. Since the FORCE algorithm requires to know the number of clusters in advance, we first apply the CORD algorithm \citep{Bunea2018} to estimate the number of clusters, which gives us $K=53$. Then we reapply the FORCE algorithm with $K=53$ to obtain the corresponding clusters. 
Recall that the goal is to analyze the dependence structure among functional systems of the brain. Under the latent variable model (\ref{eqn:g_latent_model}), we treat the measurement of ROIs as the observed variable $\bX$ and the underlying brain function as the latent variable $\bZ$. Thus, the relationship between brain functional systems can be modeled by the latent variable graph. Using the FDR control procedures with $\alpha=0.01$, the estimated latent variable graph is shown in Figure \ref{fig:graph_fmri_latent}. As a comparison, we also show the cluster average graph  in Figure \ref{fig:graph_fmri_averages}. However, in this example the biological meaning of the cluster-average variables $\bar \bX$ is unclear so that the interpretation of the cluster average graph seems difficult. 


For purposes of clarity, we only display the nodes and connections corresponding to the 10 largest groups in these two figures. The groups are colored according to the functional network the majority of their nodes belong to as given in \cite{Power2011}. It is known that the default mode may comprise multiple interacting subsystems \citep{andrews2010functional}. Thus our graph contains multiple nodes for the default mode, where each node may represent different subsystems. To demonstrate the difference between the  latent variable graph and the cluster average graph, we focus on three regions, default mode (pink), dorsal attention (brown) and salience (yellow). In the latent variable graph, dorsal attention and salience are connected and both of them are strongly connected with many subsystems of the default mode. This dependence structure is supported by the neuroscience theory that ``the default mode, engaged during rest and internally directed tasks, should exhibit anticorrelation with networks engaged during externally directed tasks, such as the dorsal attention network and the salience network" \citep{zhou2017hierarchical}. Moreover, this finding is also consistent with many existing empirical results such as  \cite{fox2005human,fransson2005spontaneous,smith2009correspondence,raichle2015brain}. However, in the cluster-averages graph, dorsal attention is conditionally independent of any subsystems of the default mode and salience is only loosely connected with the default mode. In summary, the latent variable graph seems to be more biologically meaningful than the cluster-average graph in order to interpret the dependence structure of the functional systems of the brain.



\begin{figure}
\centering
\includegraphics[width=0.9\textwidth]{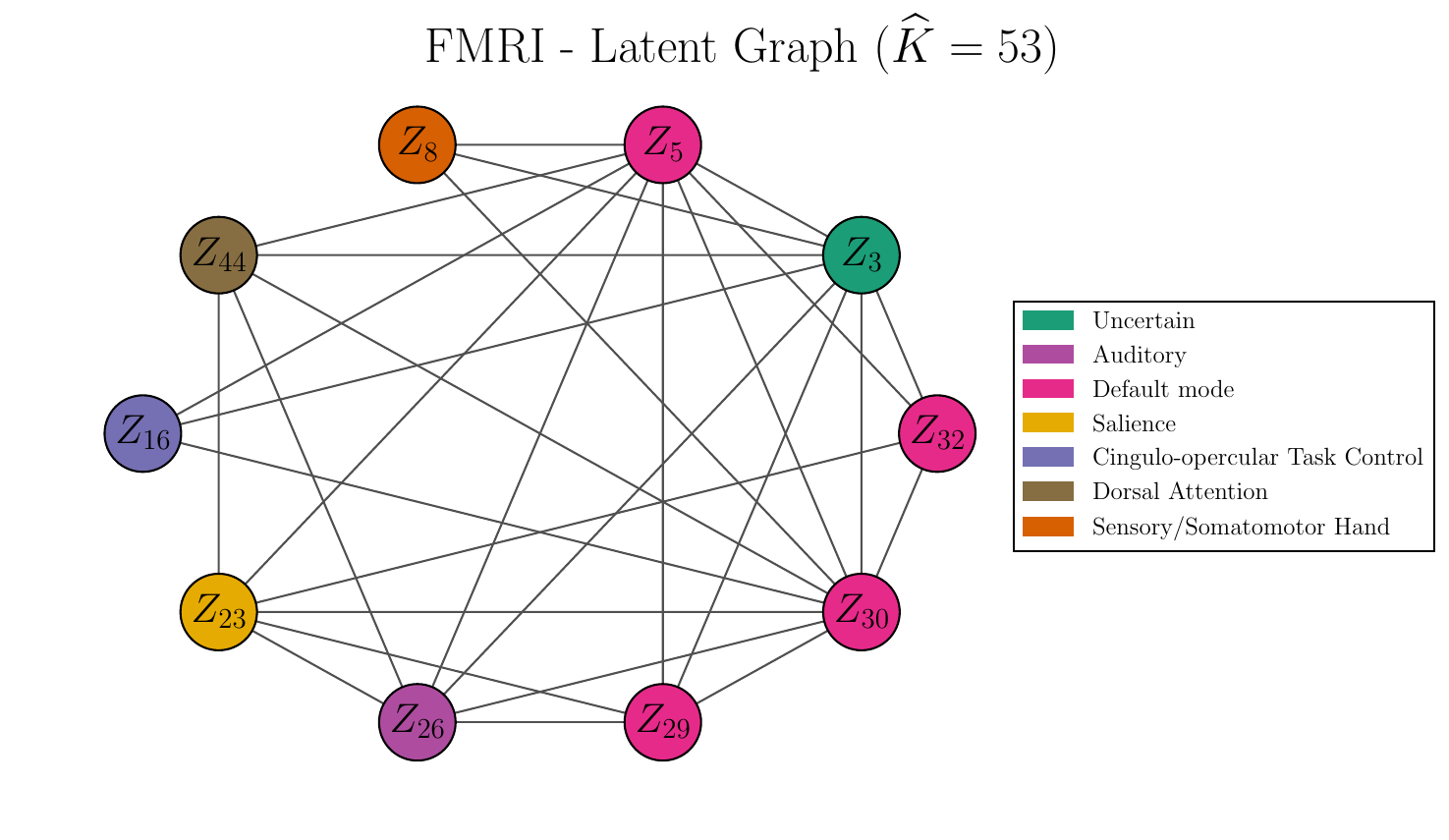}
\caption{Recovered latent graph structure between 10 largest clusters in fMRI data with FDR level $\alpha=0.01$ colored according to their functions.}
\label{fig:graph_fmri_latent}
\end{figure}

\begin{figure}
\centering
\includegraphics[width=0.9\textwidth]{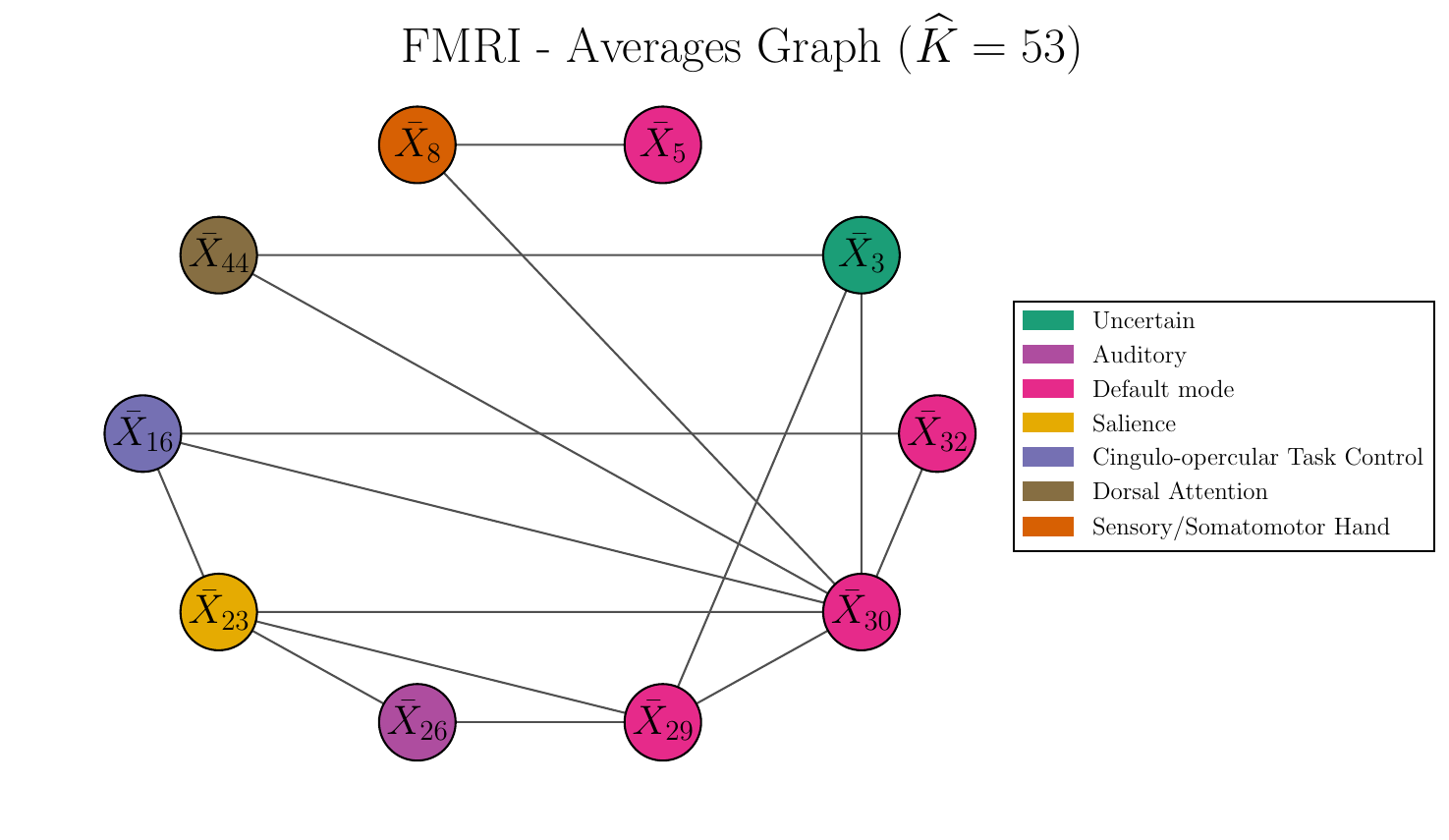}
\caption{Recovered cluster-averages graph structure between 10 largest clusters in fMRI data with FDR level $\alpha=0.01$.}
\label{fig:graph_fmri_averages}
\end{figure}


\acks{We would like to thank the reviewers and editor for their suggestions to improve this paper. Florentina Bunea was partially supported by NSF-DMS 712709. Yang Ning was partially supported by NSF-DMS 1854637. We are grateful to Xi Luo for help with the interpretation of our data analysis results. }


\newpage

\appendix


\section{Estimation in the Cluster-Average Graph}
\label{sec:main_proofs_ca}

In this section we provide the proofs of  the results needed for establishing the asymptotic  normality of the estimators of the  entries of $\bOmega^*$.
These results make use of the fact that consistent clustering is possible, under our assumptions, as stated below.
\begin{lemma}
\label{lem:pecok}
Let $\cE=: \{\hat G=G^*\}$, for $\hat{G}$ estimated by either the COD or the PECOK algorithm of \cite{Bunea2018}.
Then, under Assumptions \ref{asmp:bounded_latent_covariance} and \ref{asmp:bounded_errors}, we have
$$
\PP(\cE)\geq 1-\frac{C}{(d\vee n)^3}.
$$
\end{lemma}
The  conclusion of this Lemma is proved in Theorem 3, for the COD algorithm, and Theorem 4, for the PECOK algorithm, of \cite{Bunea2018}. Lemma \ref{lem:pecok}  allows us to replace $\hat{G}$ by $G^*$ in all the results below,  while incurring a small error, of order  $O\left(\frac{C}{(d\vee n)^3}\right)$, which will be shown to be dominated by other error bounds.

\begin{remark}
\label{rem:averages_assumptions}
While Assumptions \ref{asmp:bounded_latent_covariance} and \ref{asmp:bounded_errors} are made for $\Cb^*$, they do imply that $c_1 \leq \lmin{\bEssS}$ and $\max_{t}\EssS{t}{t} \leq c_2 + c_3$ holds for $\bS^*$. Furthermore Lemma \ref{lem:latent_re_condition} implies the same restricted eigenvalue condition on $\lmin{\bEssS}$ as on $\bCS$.
\end{remark}

\subsection{Main Proofs for the Cluster-Average Graph Estimators}
Before proving Theorem \ref{thm:xi_asymptotic} and Claim 1 of Theorem \ref{thm:fdr_bound_av}, we state two propositions -- one regarding the asymptotic normality of $\OmegaT{t}{k}$ and the other regarding the convergence rate of the variance estimator. The proofs are deferred until after that of the main result

\begin{proposition}[Asymptotic Normality of $\OmegaT{t}{k}$]
\label{prop:asymptotic_normality_group_averages}
Under the same conditions as in Theorem \ref{thm:xi_asymptotic}, we have
\[
\max_{1\leq t< k\leq K}\sup_{x \in \RR}\Big|\PP\Big(\frac{n^{1/2}(\tilde\Omega_{t,k}-\Omega^*_{t,k})}{s_{tk}}\leq x, \cE\Big)-\Phi(x)\Big|\leq \frac{C}{(K\vee n)^3}+\frac{Cs_1\log (K\vee n)}{n^{1/2}}.
\]
\end{proposition}

\begin{proposition}[Convergence Rate of the Variance Estimator]
\label{prop:convergence_group_averages_variance}
Under the same conditions as in Theorem \ref{thm:xi_asymptotic}, we have
\[
\max_{1\leq t< k\leq K}|\hat s_{t,k}^2 - s_{t,k}^2| \leq C \sqrt{\frac{s_1\log(K \vee n)}{n}},~~
\max_{1\leq t< k\leq K}\Big|\frac{\hat s_{t,k}}{s_{t,k}}-1\Big| \leq C \sqrt{\frac{s_1\log(K \vee n)}{n}},
\]
with probability at least $1-(K\vee n)^{-3}$.
\end{proposition}

\subsection*{Proof of Theorem \ref{thm:xi_asymptotic}}
The proof relies crucially on Proposition \ref{prop:asymptotic_normality_group_averages}.  Once that result is established, the proof of Theorem \ref{thm:xi_asymptotic} follows standard steps as explained below.

Denote $\cE'=\{\max_{1\leq t< k\leq K}|\hat s_{tk}/s_{tk}-1|\leq r\}$, where $r=C\sqrt{\frac{s_1\log (K\vee n)}{n}}$, and let $\bar\cE'$ signify the complement of the event $\cE'$. Let $T_{t,k}$ and $\hat T_{t,k}$ denote the statistics $\frac{n^{1/2}(\tilde\Omega_{t,k}-\Omega^*_{t,k})}{s_{tk}}$ and $\frac{n^{1/2}(\tilde\Omega_{t,k}-\Omega^*_{t,k})}{\hat s_{tk}}$, respectively.
We first consider the bound
\begin{align*}
\PP\Big(\hat T_{t,k} \leq x\Big)-\Phi(x)&\leq \PP\Big(\hat T_{t,k} \leq x, \cE', \cE\Big)-\Phi(x)+\PP(\bar\cE')+\PP(\bar\cE)\\
&=\PP\Big(T_{t,k}\leq x\frac{\hat s_{tk}}{s_{tk}}, \cE',\cE\Big)-\Phi(x)+\PP(\bar\cE')+\PP(\bar\cE).
\end{align*}
Proposition \ref{prop:convergence_group_averages_variance} implies $\PP(\bar\cE')\leq C(K\vee n)^{-3}$ and Lemma \ref{lem:pecok}  above implies $\PP(\bar\cE)\leq C(d\vee n)^{-3}$ for some constant $C$. In addition, for $x\geq 0$,
\begin{align}
&\PP\Big(T_{t,k} \leq x\frac{\hat s_{tk}}{s_{tk}}, \cE',\cE \Big)-\Phi(x) \leq \PP\Big( T_{t,k} \leq x(1+r),\cE \Big)-\Phi(x)\nonumber\\
&=\Big\{\PP\Big( T_{t,k} \leq x(1+r), \cE \Big)-\Phi(x(1+r))\Big\}+\Big\{\Phi(x(1+r))-\Phi(x)\Big\}.\label{eqthmxi1}
\end{align}
For the first term, Proposition \ref{prop:asymptotic_normality_group_averages} implies
$$
\max_{1\leq t< k\leq K}\sup_{x \in \RR}\Big|\PP\Big( T_{t,k} \leq x(1+r),\cE\Big)-\Phi(x(1+r))\Big|\lesssim \frac{s_1\log (K\vee n)}{n^{1/2}}.
$$
By the mean value theorem, the second term $\Phi(x(1+r))-\Phi(x)=\phi(x(1+tr))xr$, for some $t\in[0,1]$. It is easily seen that $\sup_{x\in\RR}\sup_{t\in[0,1]}|\phi(x(1+tr))x|\leq C$ for some constant $C$. Plugging it into (\ref{eqthmxi1}), we obtain
\begin{equation}\label{eqthmxi2}
\max_{1\leq t< k\leq K}\sup_{x \in \RR}\Big\{\PP\Big( T_{t,k} \leq x\frac{\hat s_{tk}}{s_{tk}}, \cE',\cE\Big)-\Phi(x)\Big\}\lesssim \frac{s_1\log (K\vee n)}{n^{1/2}}+r\lesssim \frac{s_1\log (K\vee n)}{n^{1/2}}.
\end{equation}
When $x<0$, similar to (\ref{eqthmxi1}), the bound is
\begin{align*}
&\PP\Big( T_{t,k} \leq x\frac{\hat s_{tk}}{s_{tk}}, \cE',\cE\Big)-\Phi(x)\\
&\leq \Big\{\PP\Big( T_{t,k} \leq x(1-r),\cE\Big)-\Phi(x(1-r))\Big\}+\Big\{\Phi(x(1-r))-\Phi(x)\Big\}.\label{eqthmxi1}
\end{align*}
Thus (\ref{eqthmxi2}) holds for $x<0$ as well. Combining these results, we obtain
$$
\max_{1\leq t< k\leq K}\sup_{x \in \RR}\Big\{\PP\Big(\hat T_{t,k} \leq x\Big)-\Phi(x)\Big\}\lesssim \frac{s_1\log (K\vee n)}{n^{1/2}}+\frac{1}{(K\vee n)^3}+\frac{1}{(d\vee n)^3}.
$$
Following the similar argument, we can also derive
$$
\max_{1\leq t< k\leq K}\sup_{x \in \RR}\Big\{\Phi(x)-\PP\Big(\hat T_{t,k} \leq x\Big)\Big\}\lesssim \frac{s_1\log (K\vee n)}{n^{1/2}}+\frac{1}{(K\vee n)^3}+\frac{1}{(d\vee n)^3}.
$$
This completes the proof.

\subsection*{Proof of Theorem \ref{thm:fdr_bound_av}, Claim 1}
The proof follows verbatim that of Theorem 8.5 in \cite{giraud2014introduction}, with the exception of the fact that exact $p$-values are replaced by approximate $p$-values, including the rate of approximation. We include the full proof for the convenience of the reader.  By the definition of the FDR,
\begin{equation}
\FDR(\hat \tau)=\EE\Big[\frac{ \sum_{(t, k) \in\cH_{0}} \II[|\tilde W_{t,k}| > \hat\tau ] \II[R_{\hat\tau} > 0]}{R_{\hat\tau}}\Big]=\sum_{(t, k) \in\cH_{0}}\EE\Big[\frac{\II[|\tilde W_{t,k}| > \hat\tau ] \II[R_{\hat\tau} > 0]}{R_{\hat\tau}}\Big].\label{eqfdr1}
\end{equation}
To handle the $R_{\hat\tau}$ in the denominator, we use the identity $1=\sum_{i=R_{\hat\tau}}^\infty\frac{R_{\hat\tau}}{i(i+1)}$. This implies
$$
1/R_{\hat\tau}=\sum_{i=R_{\hat\tau}}^\infty\frac{1}{i(i+1)}=\sum_{i=1}^\infty\frac{\II[i\geq R_{\hat\tau}]}{i(i+1)}
$$
Plugging this into (\ref{eqfdr1}) and bringing the expectation inside the summation gives that
\begin{align}
\FDR(\hat \tau)&=\sum_{(t, k) \in\cH_{0}}\sum_{i=1}^\infty\frac{1}{i(i+1)}\EE\Big[\II[|\tilde W_{t,k}| > \hat\tau ] \II[R_{\hat\tau} > 0]\II[i\geq R_{\hat\tau}]\Big]\nonumber\\
&\leq \sum_{(t, k) \in\cH_{0}}\sum_{i=1}^\infty\frac{1}{i(i+1)}\EE\Big[\II[|\tilde W_{t,k}| > \Phi^{-1}\left(1-\frac{\alpha R_{\hat\tau}}{2N_{BY}|\cH|} \right) ] \II[R_{\hat\tau} > 0]\II[i\geq R_{\hat\tau}]\Big]\nonumber\\
&\leq \sum_{(t, k) \in\cH_{0}}\sum_{i=1}^\infty\frac{1}{i(i+1)}\EE\Big[\II[|\tilde W_{t,k}| > \Phi^{-1}\left(1-\frac{\alpha (i\wedge |\cH|)}{2N_{BY}|\cH|} \right) ] \Big],\label{eqfdr2}
\end{align}
where the second line follows from the definition of the FDR cutoff and the last inequality holds since $R_{\hat\tau}\leq (i\wedge |\cH|)$. The Berry-Esseen bound in Theorem \ref{thm:xi_asymptotic} implies that
$$
\PP(|\tilde W_{t,k}| > \Phi^{-1}\left(1-\frac{\alpha (i\wedge |\cH|)}{2N_{BY}|\cH|} \right))\leq \frac{\alpha (i\wedge |\cH|)}{N_{BY}|\cH|}+2b_n.
$$
Thus, it follows that
\begin{align*}
\FDR(\hat \tau)&\leq \sum_{(t, k) \in\cH_{0}}\sum_{i=1}^\infty \frac{1}{i(i+1)}\Big(\frac{\alpha (i\wedge |\cH|)}{N_{BY}|\cH|}+2b_n\Big)\\
&=\alpha\frac{|\cH_0|}{|\cH|}\Big(\sum_{i=1}^{|\cH|}\frac{i}{i(i+1)N_{BY}}+\sum_{i=|\cH|+1}^{\infty}\frac{|\cH|}{i(i+1)N_{BY}}\Big)+2|\cH_0| b_n\\
&=\alpha\frac{|\cH_0|}{|\cH|}\Big(\sum_{i=1}^{|\cH|}\frac{1}{i+1}\frac{1}{N_{BY}}+\frac{|\cH|}{|\cH|+1}\frac{1}{N_{BY}}\Big)+2|\cH_0| b_n\\
&\leq \alpha+2|\cH_0| b_n,
\end{align*}
where the last step follows from $\frac{|\cH_0|}{|\cH|}\leq 1$ and the definition of $N_{BY}$. This completes the proof of the Theorem \ref{thm:fdr_bound_av}, Claim 1.

\subsection*{Proof of Proposition \ref{prop:asymptotic_normality_group_averages}}

\begin{proof}
The proof is done in two steps. In Step 1, we show that intersected with the event $\cE$,
\begin{equation}
\label{eqn:asymptotic_normality_ga_1}
n^{1/2}|(\OmegaT{t}{k} - \OmegaS{t}{k}) /s_{t,k} + \OmegaS{t}{t}h(\bOmegaS[\cdot k])/s_{t,k}| \leq \frac{s_1 \log (K\vee n)}{n^{1/2} },
\end{equation}
with probability at least $1-C/(K\vee n)^3$ and then use Lemma \ref{lem:group_averages_clt} to obtain the result.   To prove \eqref{eqn:asymptotic_normality_ga_1}, we decompose it as
\begin{align*}
&n^{1/2}|(\OmegaT{t}{k} - \OmegaS{t}{k})  + \OmegaS{t}{t} h(\bOmega_{\cdot k}^*)| \\
&= n^{1/2}|(\OmegaH{t}{k} - \OmegaS{t}{k}) -\OmegaH{t}{t} \hat h(\hat\bOmega_{\cdot k}) + \OmegaS{t}{t} h(\bOmega_{\cdot k}^*)| \\
&\leq \underlabel{n^{1/2} |(\OmegaH{t}{k} - \OmegaS{t}{k}) - \OmegaS{t}{t} (h(\hat\bOmega_{\cdot k})-h(\bOmega_{\cdot k}^*))|}{I.1} \\
&\quad + \underlabel{n^{1/2} |\OmegaS{t}{t} (\hat h(\hat\bOmega_{\cdot k})-h(\hat\bOmega_{\cdot k}))|}{I.2} + \underlabel{n^{1/2} |(\OmegaH{t}{t}-\OmegaS{t}{t})\hat h(\hat\bOmega_{\cdot k})
|}{I.3}.
\end{align*}
In the following, we study these three terms separately. Recall that $h(\bOmega_{\cdot k}) = \vb_t^{*T}(\hat\bS\bOmega_{\cdot k} - \eb_k),$ and $\hat h(\bOmega_{\cdot k})=\hat\vb_{t}^T(\hat{\bS}\bOmega_{\cdot k}-\eb_k)$, where ${\vb}^*_{t}$ is a $K$-dimensional vector with $(\vb_{t}^*)_t=1$ and $(\vb_{t}^*)_{-t}=-\wb^*_{t}$ with $\wb^*_t=(\bS^*_{-t,-t})^{-1}\bS^*_{-t,t}$.
Term I.1 reduces to
\begin{align}
|I.1|&=n^{1/2} |(\OmegaH{t}{k} - \OmegaS{t}{k}) - \OmegaS{t}{t} \vb_t^{*T}\hat\bS(\hat\bOmega_{\cdot k}-\bOmega_{\cdot k}^*)|\nonumber\\
&\leq n^{1/2}|(\OmegaH{t}{k} - \OmegaS{t}{k})(1-\Omega_{t,t}^*(\hat\bS_{t,t}-\wb^{*T}_t\hat\bS_{-t,t}))|\label{eqasym1}\\
&+n^{1/2}\OmegaS{t}{t} |(\hat \bS_{t,-t}-\wb^{*T}_t\hat \bS_{-t,-t})(\hat\bOmega_{-t,k}-\bOmega^*_{-t,k})|. \label{eqasym2}
\end{align}
Note that $1/\Omega_{t,t}^* = S^*_{t,t}-\wb^{*T}_t\bS^*_{-t,t}$. The term in \eqref{eqasym1} can be bounded by
\begin{align*}
&n^{1/2}|(\OmegaH{t}{k} - \OmegaS{t}{k})\Omega_{t,t}^*(\hat S_{t,t}- S_{t,t}^*)|+n^{1/2}|(\OmegaH{t}{k} - \OmegaS{t}{k})\Omega_{t,t}^*\wb_t^{*T}(\hat\bS_{-t,t}-\bS_{-t,t}^*)|\\
&\leq n^{1/2}|\Omega_{t,t}^*| \|\hat\bOmega_{\cdot k}-\bOmega_{\cdot k}^*\|_1 \max(\|\hat\bS-\bS^*\|_{\max}, \|\wb_t^{*T}(\hat\bS_{-t\cdot}-\bS_{-t\cdot}^*)\|_\infty)\\
&\leq \frac{Cs_1\log(K\vee n)}{n^{1/2}},
\end{align*}
with probability at least $1-(K\vee n)^{-3}$, by $\lambda_{\max}(\bOmega^*)\leq C$ and the concentration and error bound results in Lemmas \ref{lem:group_av_S_consistency}, \ref{lem:group_averages_gradient_hessian}, \ref{lem:group_averages_consistency}. The term in (\ref{eqasym2}) can be bounded by
$$
n^{1/2}\Omega^*_{t,t} \|\hat \bS_{t,-t}-\wb^{*T}_t\hat \bS_{-t,-t}\|_\infty\|\hat\bOmega_{-t,k}-\bOmega^*_{-t,k}\|_1\leq \frac{Cs_1\log(K\vee n)}{n^{1/2}},
$$
with probability at least $1-(K\vee n)^{-3}$ again by Lemmas \ref{lem:group_averages_gradient_hessian}, \ref{lem:group_averages_consistency}. Thus, $|\text{I.1}|\leq \frac{s_1\log(K\vee n)}{n^{1/2}}$ with probability at least $1-(K\vee n)^{-3}$.

For term I.2, we have
\begin{align*}
|I.2|&=n^{1/2}\Omega_{t,t}^* |(\hat \vb_t-\vb^*)^T(\hat\bS\hat\bOmega_{\cdot k} - \eb_k)|\\
&\leq n^{1/2}\Omega_{t,t}^* \|\hat \vb_t-\vb^*\|_1\|\hat\bS\hat\bOmega_{\cdot k} - \eb_k\|_\infty\leq \frac{Cs_1\log(K\vee n)}{n^{1/2}},
\end{align*}
with probability at least $1-(K\vee n)^{-3}$ by Lemma \ref{lem:group_averages_consistency} and the constraint of the CLIME-type estimator.

To control term I.3, first we observe that
\begin{align*}
|\hat h(\hat\bOmega_{\cdot k})|&=|\hat\vb_t^T(\hat\bS\hat\bOmega_{\cdot k}-\eb_k)|\\
&\leq |\vb_t^{*T}(\hat\bS \bOmega^*_{\cdot k}-\eb_k)|+ |\vb_t^{*T}\hat\bS (\hat\bOmega_{\cdot k}-\bOmega^*_{\cdot k})|+|(\hat\vb_t-\vb_t^{*})^T(\hat\bS \hat\bOmega_{\cdot k}-\eb_k)|\\
&\leq \|\vb^*_t\|_1\|\hat\bS \bOmega^*_{\cdot k}-\eb_k\|_\infty+\|\hat\bS_{t,-t}-\wb_t^{*T}\hat\bS_{-t,-t}\|_\infty\|\hat\bOmega_{\cdot k}-\bOmega^*_{\cdot k}\|_1\\
&\quad +\|\hat\vb_t-\vb_t^{*}\|_1\|\hat\bS \hat\bOmega_{\cdot k}-\eb_k\|_\infty.
\end{align*}
As shown in the proof of Lemma \ref{lem:group_averages_clt}, $\|\vb^*_t\|_1\leq s_1^{1/2}\|\vb^*_t\|_2\leq Cs_1^{1/2}$. The rest of the bounds on the above terms follows easily from Lemmas \ref{lem:group_averages_gradient_hessian}, \ref{lem:group_averages_consistency}. Thus, we have $|\hat h(\hat\bOmega_{\cdot k})|\leq C(s_1\log (K\vee n)/n)^{1/2}$ with high probability. Since $$|\hat\Omega_{t,t}-\Omega_{t,t}^*|\leq C(s_1\log (K\vee n)/n)^{1/2},$$ by Lemma {\ref{lem:group_averages_consistency}}, we obtain that $|\text{I.3}|\leq \frac{s_1\log(K\vee n)}{n^{1/2}}$ with probability at least $1-(K\vee n)^{-3}$.
It is easily seen that $\Omega_{t,t}^*\geq \frac{1}{S^*_{t,t}}\geq C>0$, see Remark \ref{rem:averages_assumptions}. This implies that $s^2_{t,k}=\Omega_{t,t}^*\Omega_{k,k}^*+\Omega_{t,k}^{*2}$ is lower bounded by a positive constant.  The proof of (\ref{eqn:asymptotic_normality_ga_1}) is complete.

In step 2, we need to verify that
\[
\max_{1\leq t< k\leq K}\sup_{x \in \RR}\Big|\PP\Big(\frac{n^{1/2}\Omega_{t,t}^*h(\bOmega_{\cdot k}^*)}{s_{tk}}\leq x\Big)-\Phi(x)\Big|\leq \frac{C}{(K\vee n)^3}+\frac{1}{n^{1/2}},
\]
which has been done in Lemma \ref{lem:group_averages_clt}. Thus, combining with result (\ref{eqn:asymptotic_normality_ga_1}), we can use the same simple union bound as in the proof of Theorem \ref{thm:xi_asymptotic} to obtain the desired result.
\end{proof}

\subsection*{Proof of Proposition \ref{prop:convergence_group_averages_variance}}
\begin{proof}
By Lemma \ref{lem:group_averages_consistency}, under the event $\hat G=G^*$, we have
$$
\max_{1\leq t,k\leq K} |\hat\Omega_{t,k}-\Omega^*_{t,k}|\leq \max_{1\leq k\leq K} \|\hat\bOmega_{\cdot k}-\bOmega^*_{\cdot k}\|_2\leq C_1 \sqrt{\frac{s_1\log (K\vee n)}{n}},
$$
with probability at least $1-\frac{C_4}{(K\vee n)^3}$. Under this event,
\begin{align*}
\max_{1\leq t< k\leq K}|\hat s_{t,k}^2 - s_{t,k}^2|&=\max_{1\leq t< k\leq K}|\hat\Omega_{t,k}^2 + \hat\Omega_{t,t}\hat\Omega_{k,k}-(\Omega_{t,k}^{*2} + \Omega_{t,t}^*\Omega_{k,k}^*)|\\
&\leq \max_{1\leq t< k\leq K}|(\hat\Omega_{t,k}-\Omega_{t,k}^*)(\hat\Omega_{t,k}+\Omega_{t,k}^*)|\\
&\quad~~ + \hat\Omega_{t,t}|\hat\Omega_{k,k}-\Omega_{k,k}^*| +\Omega_{k,k}^*|\hat\Omega_{t,t}-\Omega_{t,t}^*|\\
&\leq C_1 \sqrt{\frac{s_1\log (K\vee n)}{n}}(4\|\bOmega^*\|_{\max}+\delta)\\
&\leq C \sqrt{\frac{s_1\log(K \vee n)}{n}},
\end{align*}
for some constant $\delta>0$ since $\|\bOmega^*\|_{\max}\leq \lambda_{\max}(\bOmega^*)\leq C$. It is easily verified that $\Omega_{t,t}^*\geq \frac{1}{S^*_{t,t}}\geq C>0$ (see Remark \ref{rem:averages_assumptions}). This implies that $s^2_{t,k}=\Omega_{t,t}^*\Omega_{k,k}^*+\Omega_{t,k}^{*2}$ is lower bounded by a positive constant. Thus,
$$
\max_{1\leq t< k\leq K}\Big|\frac{\hat s_{t,k}}{s_{t,k}}-1\Big| \leq \max_{1\leq t< k\leq K}\Big|\frac{\hat s^2_{t,k}-s^2_{t,k}}{s_{t,k}(s_{t,k}+\hat s_{t,k})}\Big|\leq C \sqrt{\frac{s_1\log(K \vee n)}{n}}.
$$
\end{proof}

\subsection{Key Lemmas for Estimators of the Cluster-Average Graph}

\begin{remark}
In the following proofs, we always assume the event $\cE=\{\hat G=G^*\}$ holds. Using a similar argument to that used  in the proof of Theorem \ref{thm:xi_asymptotic}, the following bounds will hold with probability at least $1 - \frac{C}{(K\vee n)^3}$.
\end{remark}

\begin{lemma}[Consistency of $\bEssH$]
\label{lem:group_av_S_consistency}
If Assumptions \ref{asmp:bounded_latent_covariance} and \ref{asmp:bounded_errors} hold, then with probability greater than $1 - \frac{C}{(K\vee n)^3}$,
$$
\|\bEssH - \bEssS \|_{\max}\leq C\sqrt{\frac{\log (K\vee n)}{n}}
$$
for some  constant $C$ dependent only on $c_1$, $c_2$, and $c_3$ from Assumptions \ref{asmp:bounded_latent_covariance} and \ref{asmp:bounded_errors}.
\end{lemma}

\begin{proof}
The proof follows from the proof of Theorem 1 in \cite{Cai11a}.
\end{proof}

\begin{lemma}[Concentration of Gradient and Hessian]
\label{lem:group_averages_gradient_hessian}
If Assumptions \ref{asmp:bounded_latent_covariance} and \ref{asmp:bounded_errors} hold, then with probability greater than $1-\frac{C}{(K\vee n)^3}$, we have that
\begin{itemize}
\item[(a)] $\max_{1\leq k\leq K}\|\bEssH\bOmegaS[\cdot k]-\eb_k\|_{\infty}\leq C_1\sqrt{\frac{\log (K\vee n)}{n}}$,
\item[(b)] $\max_{1\leq t\leq K}\|\bEssH[t,-t]-\wb_t^{*T}\bEssH[-t,-t]\|_\infty\leq C_2\sqrt{\frac{\log (K\vee n)}{n}}$, and
\item[(c)] $\max_{1\leq t\leq K}\|\wb_t^{*T}(\bEssH[-t,-t]-\bEssS[-t,-t])\|_\infty\leq C_2\sqrt{\frac{\log (K\vee n)}{n}}$,
\end{itemize}
for some constants $C_1$ and $C_2$ dependent only on $c_1$, $c_2$, and $c_3$ from Assumptions \ref{asmp:bounded_latent_covariance} and \ref{asmp:bounded_errors}.
\end{lemma}

\begin{proof}
We start from the decomposition
\begin{align*}
\bEssH - \bEssS &=(\Ab^{*T}\Ab^*)^{-1}\Ab^{*T}(\hat{\bSigma}-\bSigma^*)\Ab^*(\Ab^{*T}\Ab^*)^{-1}.
\end{align*}
Denoting by $\Bb^*=\Ab^{*T}\Ab^*$, we can write
\begin{align}
&(\Ab^{*T}\Ab^*)^{-1}\Ab^{*T}(\hat{\bSigma}-\bSigma^*)\Ab^*(\Ab^{*T}\Ab^*)^{-1}\nonumber\\
&=\frac{1}{n}\sum_{i=1}^n \Bb^{*-1}\Ab^{*T}\Big\{(\Ab^*\bZ_i+\bE_i)(\Ab^*\bZ_i+\bE_i)^T-\Ab^*\Cb^*\Ab^{*T}-\bGamma^*\Big\}\Ab^*\Bb^{*-1}\nonumber\\
&=\frac{1}{n}\sum_{i=1}^n\bZ_i\bZ_i^T-\Cb^*+\frac{1}{n}\sum_{i=1}^n \bZ_i\bE_i^T\Ab^*\Bb^{*-1}+\frac{1}{n}\sum_{i=1}^n \Bb^{*-1}\Ab^{*T}\bE_i\bZ_i^T\nonumber\\
&~~~+\frac{1}{n}\sum_{i=1}^n \Bb^{*-1}\Ab^{*T}(\bE_i\bE_i^T-\bGamma^*)\Ab^*\Bb^{*-1}.\label{eqn:group_averages_decomp}
\end{align}
First, note that $\bEssH\bOmegaS[\cdot k]-\eb_k=(\bEssH-\bEssS)\bOmegaS[\cdot k]$, and $||\bOmegaS[\cdot k]||_2 \leq \lmax{\bOmegaS} \leq \lmin{\bEssS}^{-1}$. As seen in Remark \ref{rem:averages_assumptions}, we can bound the smallest eigenvalue of $\bEssS$ from below by $\lmin{\bCS}$.
We therefore obtain that $||\bOmegaS[\cdot k]||_2 \leq c_1 + c_3$. Using this bound, we apply Lemma \ref{lem:conc_sum_ZiZi}, part (d) to control the term
\[
\frac{1}{n}\sum_{i=1}^n\left(\bZ_i\bZ_i^T-\Cb^*\right)\bOmegaS[\cdot k].
\]
The remaining terms in \eqref{eqn:group_averages_decomp} can similarly be controlled using Lemma \ref{lem:conc_sum_ZiEi} and Lemma \ref{lem:conc_sum_EiEi}. Applying the triangle inequality therefore gives
\[
\max_{1\leq k\leq K}\|\bEssH\bOmegaS[\cdot k]-\eb_k\|_{\infty}\leq C_1\sqrt{\frac{\log (K\vee n)}{n}}
\]
with probability at least $1 - \frac{C}{(K\vee n)^3}$, concluding the proof of claim (a).


For the remaining two claims, we can rewrite
\[
\wb_t^* = \left(\EssS{t}{t} - \bS^{*T}_{-t,t}(\bEssS[-t,-t])^{-1}\bEssS[-t,t]\right) \bOmegaS[-t,t] = \frac{1}{\OmegaS{t}{t}}\bOmegaS[-t,t]
\]
by the block matrix inverse formula. Using Lemma \ref{lem:pd_matrix_diag}, it follows that $||\wb^*_t||_2 \leq \lmin{\bOmegaS} \max_t \EssS{t}{t}$. Then we can see that
\begin{align*}
\max_{1\leq t\leq K}||\bEssH[t,-t]-\wb_t^{*T}\bEssH[-t,-t]||_{\infty} &= \max_{1\leq t\leq K}||(\bEssH[t,-t]-\bEssS[t,-t])-\wb_t^{*T}(\bEssH[-t,-t]-\bEssS[-t,-t])||_{\infty} \\
&\leq \underlabel{\max_{1\leq t\leq K}||(\bEssH[t,-t]-\bEssS[t,-t])||_{\infty}}{(i)} + \underlabel{\max_{1 \leq t\leq K}||\wb_t^{*T}(\bEssH[-t,-t]-\bEssS[-t,-t])||_\infty}{(ii)}.
\end{align*}
By using Lemma \ref{lem:group_av_S_consistency}, (i) is bounded with high probability. Following the same step as in the proof of (a), we have that  $(ii)\leq C_2\sqrt{\frac{\log (K\vee n)}{n}}$,
with probability at least $1 - \frac{C}{(K\vee n)^3}$. Part (c) is the same as the term (ii), concluding the proof.
\end{proof}

\begin{lemma}[Consistency of Initial Estimator]
\label{lem:group_averages_consistency}
If Assumptions \ref{asmp:bounded_latent_covariance} and \ref{asmp:bounded_errors} hold, then,
\begin{itemize}
\item[(a)] $\max_{1\leq k\leq K}\|\bOmegaH[\cdot k]-\bOmegaS[\cdot k] \|_1\leq C_1 s_1\sqrt{\frac{\log (K\vee n)}{n}}$,
\item[(b)] $\max_{1\leq k\leq K}\|\bOmegaH[\cdot k]-\bOmegaS[\cdot k] \|_2\leq C_1 \sqrt{\frac{s_1\log (K\vee n)}{n}}$,
\item[(c)] $\max_{1\leq t\leq K}\|\hat\vb_{t}-\vb_{t}^*\|_1 \leq C_2 s_1\sqrt{\frac{\log (K\vee n)}{n}}$, and
\item[(d)] $\max_{1\leq k\leq t\leq K}|(\hat\vb_{t}-\vb_{t}^*)^T\bEssH(\bOmegaH[\cdot k]-\bOmegaS[\cdot k])|\leq C_3 \frac{s_1\log (K\vee n)}{n}$,
\end{itemize}
with probability at least $1-\frac{C_4}{(K\vee n)^3}$. $C_1$, $C_2$, $C_3$, $C_4$ and $C_5$ are constants, dependent only upon $c_0$, $c_1$, $c_2$, and $c_3$ from Assumptions \ref{asmp:bounded_latent_covariance} and \ref{asmp:bounded_errors}.
\end{lemma}

\begin{proof}
The proof follows from the same argument as in the proof of Lemma \ref{lem:latent_consistency} with Lemma \ref{lem:latent_gradient_hessian} replaced by Lemma \ref{lem:group_averages_gradient_hessian}.
\end{proof}

\begin{lemma}[CLT for the Pseudo-Score Function]
\label{lem:group_averages_clt}
Recall that $s_{tk}^2=\Omega_{t,k}^{*2} + \Omega^*_{t,t}\Omega^*_{k,k}$. Let $F_n$ denote the CDF of $n^{1/2}\vb_t^{*T}(\hat\bS\bOmegaS[\cdot k]-\eb_k)/(s_{tk}/\Omega^*_{t,t})$. If Assumptions \ref{asmp:bounded_latent_covariance} and \ref{asmp:bounded_errors} hold, then
\[
\max_{1\leq t< k\leq K} \sup_{x\in \RR}|F_{n}(x) - \Phi(x)  | \leq C(n^{-1/2}+(d\vee n)^{-3}),
\]
where $C$ is a constant dependent only upon $c_0$, $c_1$, and $c_2$.
\end{lemma}

\begin{proof}
The proof is very similar to that of Lemma \ref{lem:latent_clt}, so we only summarize the key steps here. For the cluster averages score function, we have the similar bound that
\[
F_{n}(x) - \Phi(x)\leq \tilde F_n(x)-\Phi(x)+\PP(\bar\cE),
\]
where $\cE$ is the event $\hat G = G^*$, and $\tilde F_n(x)$ is the CDF of $n^{-1/2}\sum_{i=1}^n\vb_t^{*T}(\bar\bX_i\bar\bX_i^T\bTheta^*_{\cdot k}-\eb_k)/(s_{tk}/\Theta^*_{tt})$. Lemma \ref{lem:pecok} implies that $\PP(\bar\cE)\leq (d\vee n)^{-3}$.

As in Lemma \ref{lem:latent_clt}, we can verify the Lyapunov condition to control $\tilde F_n(x)-\Phi(x)$. Lastly,
from the Isserlis' theorem, we get that
\[
\Var(\bv_1^T\bar\bX\bar\bX^T\bv_2)=(\bv_1^T\bS^*\bv_1)(\bv_2^T\bS^*\bv_2)+(\bv_1^T\bS^*\bv_2)^2,
\]
for any vector $\bv_1$ and $\bv_2$. Using this, it is straightforward to shown
\[
\EE[\vb_t^{*T}(\bar\bX_i\bar\bX_i^T\bTheta^*_{\cdot k}-\eb_k)^2]=(s_{tk}/\Theta^*_{tt})^2,
\]
concluding the proof.
\end{proof}


\section{Estimation in the Latent Variable Graph}
\label{sec:main_proofs_lvg}
This section contains the proofs which establish the asymptotic normality of the estimator of $\bTheta^*$.

\subsection{Main Proofs for the Latent Variable Graph Estimators}
As in Section \ref{sec:main_proofs_lvg} we develop the proofs of Theorem \ref{thm:theta_asymptotic} and Claim 2 of Theorem \ref{thm:fdr_bound_av} by first establishing two supporting propositions.

\begin{proposition}[Asymptotic Normality of $\tilde \Theta_{t,k}$]
\label{prop:asymptotic_normality_latent}
Under the same conditions as in Theorem \ref{thm:theta_asymptotic}, we get
\[
\max_{1\leq t< k\leq K}\sup_{x\in\RR} \Big|\PP\left( \frac{n^{1/2}(\tilde \Theta_{t,k} - \ThetaS{t}{k})}{\sigma_{t,k}}<x , \cE \right) - \Phi(x)\Big| \leq \frac{C}{(K\vee n)^3}+\frac{Cs_0\log (K\vee n)}{n^{1/2}}.
\]
\end{proposition}

\begin{proof}
The proof follows all the steps of Proposition \ref{prop:asymptotic_normality_group_averages}, with Lemma \ref{lem:group_av_S_consistency} replaced by Lemma \ref{lem:latent_C_consistency},  \ref{lem:group_averages_gradient_hessian} by \ref{lem:latent_gradient_hessian}, \ref{lem:group_averages_consistency} by \ref{lem:latent_consistency} and \ref{lem:group_averages_clt} by \ref{lem:latent_clt}.
\end{proof}

\begin{proposition}[Convergence Rate of the Variance Estimator]
\label{prop:convergence_latent_variance}
Under the same conditions as in Theorem \ref{thm:theta_asymptotic}, we get
\begin{align*}
\max_{1\leq t< k\leq K}|\hat \sigma_{t,k}^2 - \sigma_{t,k}^2| &\leq C \sqrt{\frac{s_0\log(K \vee n)}{n}}+ \frac{Cs_0}{m}, \text{ and }\\
\max_{1\leq t< k\leq K}\Big|\frac{\hat \sigma_{t,k}}{\sigma_{t,k}}-1\Big| &\leq C \sqrt{\frac{s_0\log(K \vee n)}{n}}+ \frac{Cs_0}{m},
\end{align*}
with probability at least $1-(K\vee n)^{-3}$.
\end{proposition}
\begin{proof}
Similar to the proof of Proposition \ref{prop:convergence_group_averages_variance}, we can prove that
$$
\max_{1\leq t< k\leq K}|\hat \sigma_{t,k}^2 - (\Theta_{t,k}^{*2}+\Theta^*_{t,t}\Theta^*_{k,k})| \leq C \sqrt{\frac{s_0\log(K \vee n)}{n}},
$$
with probability at least $1-(K\vee n)^{-3}$. Then, by Lemma \ref{lem:latent_variable_variance}, we obtain
$$
\max_{1\leq t< k\leq K}|\hat \sigma_{t,k}^2 - \sigma_{t,k}^2| \leq C \sqrt{\frac{s_0\log(K \vee n)}{n}}+ \frac{Cs_0}{m}.
$$
The second statement can be similarly derived.
\end{proof}

Having established the asymptotic normality of $\tilde \Theta_{t,k}$ and the convergence rate of $\hat \sigma_{t,k}^2$, we proceed with the proofs of the main results.

\subsection*{Proof of Theorem \ref{thm:theta_asymptotic}}
\begin{proof}
The proof follows in exactly the same manner as that of Theorem \ref{thm:xi_asymptotic}, but invokes Proposition \ref{prop:convergence_latent_variance}, instead of Proposition \ref{prop:convergence_group_averages_variance},
 and  Proposition \ref{prop:asymptotic_normality_latent}, instead of Proposition \ref{prop:asymptotic_normality_group_averages},  as one needs to establish different intermediate results,  specifically tailored to estimation of the latent graph.
\end{proof}

\subsection*{Proof of Theorem \ref{thm:fdr_bound_av}, Claim 2}
\begin{proof}
The proof follows in exactly the same manner as that of Theorem \ref{thm:fdr_bound_av}, Claim 1.
\end{proof}

\subsection{Key Lemmas for Estimators of the Latent Graph}
As before,  we assume in the following proofs that the event $\cE=\{\hat G=G^*\}$ holds. Using a similar argument to that used in the proof of Theorem \ref{thm:theta_asymptotic}, the following bounds will hold with probability at least $1 - \frac{C}{(K\vee n)^3}$.

\begin{lemma}[Consistency of $\hat \Cb$]
\label{lem:latent_C_consistency}
If Assumptions \ref{asmp:bounded_latent_covariance} and \ref{asmp:bounded_errors} hold, then with probability greater than $1 - \frac{C}{(K\vee n)^3}$,
$$
\|\hat \Cb-\Cb^*\|_{\max}\leq C\sqrt{\frac{\log (K\vee n)}{n}},
$$
for some  constant $C$ dependent only on $c_1$, $c_2$, and $c_3$ from Assumptions \ref{asmp:bounded_latent_covariance} and \ref{asmp:bounded_errors}.
\end{lemma}

\begin{proof}
Denoting $\Bb^*=\Ab^{*T}\Ab^*$, we begin by using the decomposition
\begin{equation}
\label{eqn:c_hat_minus_c_star}
\hat\Cb-\Cb^* = \Bb^{*-1}\Ab^{*T}(\hat{\bSigma}-\bSigma^*)\Ab^*\Bb^{*-1}-\Bb^{*-1}\Ab^{*T}(\hat{\bGamma}-\bGamma^*)\Ab^*\Bb^{*-1}.
\end{equation}
Next, we can write for $i \in \Gs{k}$,
\begin{align*}
(\hat\bGamma - \bGammaS)_{i,i} &= \hat \gamma_i - \gamma_i^*\\
&= \SigmaH{i}{i} - \frac{1}{|\Gs{k}| - 1}\sum_{j \in \Gs{k}, j\neq i} \SigmaH{i}{j} - \gamma_i^*\\
&= \SigmaH{i}{i} - \frac{1}{|\Gs{k}| - 1}\sum_{j \in \Gs{k}, j\neq i} \SigmaH{i}{j} - \SigmaS{i}{i} + \frac{1}{|\Gs{k}| - 1}\sum_{j \in \Gs{k}, j\neq i} \SigmaS{i}{j}\\
&= \SigmaH{i}{i} - \SigmaS{i}{i} - \frac{1}{|\Gs{k}| - 1}\sum_{j \in \Gs{k}, j\neq i}\left[\SigmaH{i}{j} - \SigmaS{i}{j}\right],
\end{align*}
which implies that $||\bGammaH - \bGammaS||_{\max} \leq 2||\bSigmaH - \bSigmaS||_{\max}$. Therefore from Lemma \ref{lem:ATAIA} we see that
\[
||\Bb^{*-1}\Ab^{*T}(\hat{\bGamma}-\bGamma^*)\Ab^*\Bb^{*-1}||_{\max}\leq \frac{2}{m}||\Bb^{*-1}\Ab^{*T}(\hat{\bSigma}-\bSigma^*)\Bb^{*-1}||_{\max},
\]
demonstrating that it will suffice to bound the first term in \eqref{eqn:c_hat_minus_c_star}, which we now do.

For the first term in \eqref{eqn:c_hat_minus_c_star}, we have
\begin{align*}
&\Bb^{*-1}\Ab^{*T}(\hat{\bSigma}-\bSigma^*)\Ab^*\Bb^{*-1} \\
&=\frac{1}{n}\sum_{i=1}^n \Bb^{*-1}\Ab^{*T}\Big\{(\Ab^*\bZ_i+\bE_i)(\Ab^*\bZ_i+\bE_i)^T-\Ab^*\Cb^*\Ab^{*T}-\bGamma^*\Big\}\Ab^*\Bb^{*-1} \\
&=\frac{1}{n}\sum_{i=1}^n\bZ_i\bZ_i^T-\Cb^*+\frac{1}{n}\sum_{i=1}^n \bZ_i\bE_i^T\Ab^*\Bb^{*-1}+\frac{1}{n}\sum_{i=1}^n \Bb^{*-1}\Ab^{*T}\bE_i\bZ_i^T \\
&~~~+\frac{1}{n}\sum_{i=1}^n \Bb^{*-1}\Ab^{*T}(\bE_i\bE_i^T-\bGamma^*)\Ab^*\Bb^{*-1}. \numberthis \label{eqlemconcen11}
\end{align*}
Using the triangle inequality, we can apply Lemma \ref{lem:conc_sum_ZiZi}, Lemma \ref{lem:conc_sum_ZiEi} and Lemma \ref{lem:conc_sum_EiEi} to bound \ref{eqlemconcen11}. Combining these results with that  $\PP(\cE)\geq 1-c_0/(d\vee n)^3$, we obtain
\[
\|\hat \Cb-\Cb^*\|_{\max}\leq C\sqrt{\frac{\log (K\vee n)}{n}}
\]
with probability at least $1 - \frac{C}{(K\vee n)}$ for some constant $C$ dependent only on $c_1$, $c_2$, and $c_3$ from Assumptions \ref{asmp:bounded_latent_covariance} and \ref{asmp:bounded_errors}.
\end{proof}

\begin{lemma}[Concentration of Gradient and Hessian of the Loss Function]
\label{lem:latent_gradient_hessian}
If Assumptions \ref{asmp:bounded_latent_covariance} and \ref{asmp:bounded_errors} hold, then with probability greater than $1-\frac{C}{(K\vee n)^3}$, we have that
\begin{itemize}
\item[(a)] $\max_{1\leq k\leq K}\|\hat \Cb\bTheta^*_{\cdot k}-\eb_k\|_{\infty}\leq C_1\sqrt{\frac{\log (K\vee n)}{n}}$,
\item[(b)] $\max_{1\leq t\leq K}\|\hat \Cb_{t,-t}-\wb_t^{*T}\hat\Cb_{-t,-t}\|_\infty\leq C_2\sqrt{\frac{\log (K\vee n)}{n}}$, and
\item[(c)] $\max_{1\leq t\leq K}\|\wb_t^{*T}(\hat\Cb_{-t,-t}-\Cb^*_{-t,-t})\|_\infty\leq C_2\sqrt{\frac{\log (K\vee n)}{n}}$,
\end{itemize}
for absolute constants $C_1$ and $C_2$ dependent only on $c_1$, $c_2$, and $c_3$ from Assumptions \ref{asmp:bounded_latent_covariance} and \ref{asmp:bounded_errors}.
\end{lemma}

\begin{proof}
Let $\Bb^*=\Ab^{*T}\Ab^*$ and observe that $\hat \Cb\bTheta^*_{\cdot k}-\eb_k=(\hat \Cb-\Cb^*)\bTheta^*_{\cdot k}$. Following the decomposition (\ref{eqlemconcen11}), we can similarly show that
\begin{align*}
&\Bb^{*-1}\Ab^{*T}(\hat{\bSigma}-\bSigma^*)\Ab^*\Bb^{*-1}\bTheta^*_{\cdot k} \\
&=\frac{1}{n}\sum_{i=1}^n\bZ_i\bZ_i^T\bTheta^*_{\cdot k}-\Cb^*\bTheta^*_{\cdot k}+\frac{1}{n}\sum_{i=1}^n \bZ_i\bE_i^T\Ab^*\Bb^{*-1}\bTheta^*_{\cdot k}+\frac{1}{n}\sum_{i=1}^n \Bb^{*-1}\Ab^{*T}\bE_i\bZ_i^T\bTheta^*_{\cdot k}\nonumber\\
&~~~+\frac{1}{n}\sum_{i=1}^n \Bb^{*-1}\Ab^{*T}(\bE_i\bE_i^T-\bGamma^*)\Ab^*\Bb^{*-1}\bTheta^*_{\cdot k} \numberthis \label{eqn:latent_gradient1}.
\end{align*}
As in the proof of Lemma \ref{lem:latent_C_consistency}, we have that
\[
||\Bb^{*-1}\Ab^{*T}(\hat{\bGamma}-\bGamma^*)\Ab^*\Bb^{*-1}\bTheta^*_{\cdot k}||_{\infty}\leq \frac{2}{m}||\Bb^{*-1}\Ab^{*T}(\hat{\bSigma}-\bSigma^*)\Ab^*\Bb^{*-1}\bTheta^*_{\cdot k}||_{\infty},
\]
demonstrating that again it will suffice to bound the first term in \ref{eqn:latent_gradient1}.

Note that $||\bThetaS[\cdot k]||_2 \leq \lmax{\bThetaS} \leq c_1^{-1}$. Therefore, by using the triangle inequality, we can apply Lemma \ref{lem:conc_sum_ZiZi}, Lemma \ref{lem:conc_sum_ZiEi} and Lemma \ref{lem:conc_sum_EiEi} to bound the first term in \ref{eqn:latent_gradient1}. Combining these results and $\PP(\cE)\geq 1-C/(d\vee n)^3$, we obtain
\[
\max_{1\leq k\leq K}\|\hat \Cb\bTheta^*_{\cdot k}-\eb_k\|_{\infty}\leq C_1\sqrt{\frac{\log (K\vee n)}{n}},
\]
with probability at least $1 - \frac{C}{(K\vee n)^3}$ for some constant $C_1$ dependent only on $c_1$, $c_2$, and $c_3$ from Assumptions \ref{asmp:bounded_latent_covariance} and \ref{asmp:bounded_errors}.

For the remaining two claims, we can rewrite
\begin{equation*}
\label{eqn:norm_wstar_bound}
\wb_t^* = \left(\CS{t}{t} - \Cb^{*T}_{-t,t}(\bCS[-t,-t])^{-1}\bCS[-t,t]\right) \bThetaS[-t,t] = \frac{1}{\ThetaS{t}{t}}\bThetaS[-t,t]
\end{equation*}
by the block matrix inverse formula. Using Lemma \ref{lem:pd_matrix_diag}, it follows that $||\wb^*_t||_2 \leq \lmax{\bThetaS} \max_t \CS{t}{t}$. Then we see that
\begin{align*}
\max_{1\leq t\leq K}||\hat \Cb_{t,-t}-\wb_t^{*T}\hat\Cb_{-t,-t}||_{\infty} &= \max_{1\leq t\leq K}||(\hat \Cb_{t,-t}-\Cb^*_{t,-t})-\wb_t^{*T}(\hat\Cb_{-t,-t}-\Cb^*_{-t,-t})||_{\infty} \\
&\leq \underlabel{\max_{1\leq t\leq K}||(\hat \Cb_{t,-t}-\Cb^*_{t,-t})||_\infty}{(i)} + \underlabel{\max_{1\leq t\leq K}||\wb_t^{*T}(\hat\Cb_{-t,-t}-\Cb^*_{-t,-t})||_\infty}{(ii)}
\end{align*}
Clearly, using Lemma \ref{lem:latent_C_consistency}, (i) is bounded with high probability. Likewise, Lemma \ref{lem:latent_C_consistency} demonstrates that $\hat\Cb_{-t,i}-\Cb^*_{-t,i}$ is a sub-exponential random vector with parameters dependent only on $c_1$, $c_2$, and $c_3$ from Assumptions \ref{asmp:bounded_latent_covariance} and \ref{asmp:bounded_errors}. Thus $\wb_t^{*T}(\hat\Cb_{-t,i}-\Cb^*_{-t,i})$ is sub-exponential and because $||\wb^*_t||_2 \leq \lmax{\bThetaS} \max_t \CS{t}{t}$, we  obtain that
\[
\max_{1\leq t\leq K}\|\hat \Cb_{t,-t}-\wb_t^{*T}\hat\Cb_{-t,-t}\|_\infty\leq C_2\sqrt{\frac{\log (K\vee n)}{n}}
\]
with probability at least $1 - \frac{C}{(K\vee n)^3}$ for some constant $C_2$. $C_2$ is dependent only on $c_1$, $c_2$, and $c_3$ from Assumptions \ref{asmp:bounded_latent_covariance} and \ref{asmp:bounded_errors}. The final result is bounded by the previous one, concluding the proof.
\end{proof}

\begin{lemma}[Consistency of Initial Estimators]
\label{lem:latent_consistency}
If Assumptions \ref{asmp:bounded_latent_covariance} and \ref{asmp:bounded_errors} hold, then
\begin{itemize}
\item[(a)] $\max_{1\leq k\leq K}\|\hat\Theta_{\cdot k}-\Theta_{\cdot k}^*\|_1 \leq C_1 s_0\sqrt{\frac{\log (K\vee n)}{n}}$,
\item[(b)] $\max_{1\leq k\leq K}\|\hat\Theta_{\cdot k}-\Theta_{\cdot k}^*\|_2 \leq C_1 \sqrt{\frac{s_0\log (K\vee n)}{n}}$,
\item[(c)] $\max_{1\leq t\leq K}\|\hat\vb_{t}-\vb_{t}^*\|_1 \leq C_2 s_0\sqrt{\frac{\log (K\vee n)}{n}}$, and
\item[(d)] $\max_{1\leq k\leq t\leq K}|(\hat\vb_{t}-\vb_{t}^*)^T\hat \Cb(\hat\bTheta_{\cdot k}-\bTheta_{\cdot k}^*)| \leq C_3 \frac{s_0\log (K\vee n)}{n}$,
\end{itemize}
with probability at least $1-\frac{C_4}{(K\vee n)^3}$. $C_1$, $C_2$, $C_3$, $C_4$ are constants, dependent only upon $c_0$, $c_1$, $c_2$, and $c_3$ from Assumptions \ref{asmp:bounded_latent_covariance} and \ref{asmp:bounded_errors}.
\end{lemma}

\begin{proof}
Below, the constants $C_a$, $C_a'$, $C_b$, $C_b'$, $C_b''$, $C_c$ and $C_c'$ will depend only upon $c_0$, $c_1$, $c_2$, and $c_3$ from Assumptions \ref{asmp:bounded_latent_covariance} and \ref{asmp:bounded_errors}.

We first prove part (b). The proof of part (a) is similar. Let $\hat\bDelta = \hat \wb_t - \wb_t^*$, noting that we can consider $\wb_t$ instead of $\vb_t$ as the $t^{th}$ entries in both the estimated and true value are 1. By $S$ denote the support of $\wb_t^*$. $\wb_t^*$ is $s_0$-sparse because $\wb_t^*$ is a multiple of $\bThetaS[-t,k]$, which we know to be $s_0$-sparse.

By Lemma \ref{lem:latent_gradient_hessian}, there exists $C_b$ such that for $\lambda \geq C_b \sqrt{\frac{\log (K\vee n)}{n}}$, $\wb_t^*$ is feasible for \eqref{eqw} with probability at least $1-\frac{C_b'}{(K\vee n)^3}$. Assuming $\wb_t^*$ is feasible, then it follows by definition that $||(\wb_t^*)_{S}||_1 \geq ||(\hat\wb_t)_{S}||_1 + ||(\hat\wb_t)_{\bar{S}}||_1$. This in turn implies by the triangle inequality that $||\hat\bDelta_{S}||_1 \geq ||\hat\bDelta_{\bar{S}}||_1$. Letting $\lambda = C_b \sqrt{\frac{\log (K\vee n)}{n}}$, it follows from the triangle inequality that
\begin{align*}
||\hat\Cb_{-t,-t} \hat\bDelta||_\infty &\leq ||\hat\wb_t^T\hat\Cb_{-t,-t} - \Cb_{t,-t}\||_\infty + ||\wb_t^{*T}\hat\Cb_{-t,-t} - \Cb_{t,-t}\||_\infty \\
&\leq 2C_b\sqrt{\frac{\log (K\vee n)}{n}}.
\end{align*}
In addition, note that $||\hat\bDelta||_1 \leq 2 ||\hat\bDelta_S||_1 \leq 2 \sqrt{s_0}||\hat\bDelta_S||_2\leq 2\sqrt{s_0}||\hat\bDelta||_2$. Therefore combining with the above, this gives
\begin{align*}
\hat\bDelta^T\hat\Cb_{-t,-t}\bDelta &\leq ||\hat\bDelta||_1||\hat\Cb_{-t,-t}\hat\bDelta||_\infty \\
&\leq 2C\sqrt{\frac{\log (K\vee n)}{n}}||\hat\bDelta||_1 \\
&\leq 4C_b\sqrt{\frac{s_0\log (K\vee n)}{n}}||\hat\bDelta||_2.
\end{align*}
From Lemma \ref{lem:latent_re_condition}, $\hat\bDelta^T\hat\Cb_{-t,-t}\hat\bDelta \geq \frac{4c_1}{3} ||\hat\bDelta||_2^2$ with probability at least $1-\frac{C_b''}{(K\vee n)^3}$. Therefore
\begin{align*}
||\hat\bDelta||_2 &\leq \frac{16C c_1}{3}\sqrt{\frac{s_0\log (K\vee n)}{n}}, \text{ and }\\
||\hat\bDelta||_1 &\leq \frac{32C_bs_0 c_1}{3}\sqrt{\frac{\log (K\vee n)}{n}},
\end{align*}
with probability at least $1-\frac{\max\{C_b',C_b''\}}{(K\vee n)^3}$.

To obtain part (c), first let $\delta = \max_{1\leq k\leq t\leq K}|(\hat\vb_{t}-\vb_{t}^*)^T\hat \Cb(\hat\bTheta_{\cdot k}-\bTheta_{\cdot k}^*)|$ and then by applying Holder's inequality and the triangle inequality, we obtain
\begin{align*}
\delta &\leq \max_{1\leq k\leq t\leq K}||\hat\vb_{t}-\vb_{t}^*||_1||\hat \Cb(\hat\bTheta_{\cdot k}-\bTheta_{\cdot k}^*)||_{\infty}\\
&\leq \max_{1\leq k\leq t\leq K}||\hat\vb_{t}-\vb_{t}^*||_1 \left( ||\bCH \bThetaH[\cdot k] - \be_k ||_{\infty} + ||\bCH\bThetaS[\cdot k] - \be_k||_{\infty} \right). \numberthis \label{eqn:latent_consistency_partc1}
\end{align*}
With choice of $\lambda$ as above, the KKT conditions give that $||\bCH \bThetaH[\cdot k] - \be_k ||_{\infty} \leq C_b \sqrt{\frac{\log (K\vee n)}{n}}$. From Lemma \ref{lem:latent_gradient_hessian}, we have that  $||\bCH\bThetaS[\cdot k] - \be_k||_{\infty} \leq C \sqrt{\frac{\log (K\vee n)}{n}}$ with probability at least $1-\frac{C}{(K\vee n)^3}$. Using part (b), we get that $\max_{1\leq t\leq K}\|\hat\vb_{t}-\vb_{t}^*\|_1 \leq Cs_0\sqrt{\frac{\log (K\vee n)}{n}}$ with probability at least $1 - \frac{C}{(K\vee n)^3}$. The desired result now follows from \eqref{eqn:latent_consistency_partc1}.
\end{proof}
\begin{lemma}[CLT for the Pseudo-Score Function]
\label{lem:latent_clt}
Recall that
\[
\sigma^2_{tk}=\EE(\Theta_{tt}^*\vb_t^{*T}(\bar \Cb^{(i)}\bTheta^*_{\cdot k}-\eb_k))^2,
\]
with $\bar \Cb^{(i)}$ defined in (\ref{eqCi}). Let $F_{n}$ denote the CDF of $n^{-1/2}\vb_t^{*T}(\hat \Cb\bTheta^*_{\cdot k}-\eb_k)/(\sigma_{tk}/\Theta^*_{tt})$. If Assumptions \ref{asmp:bounded_latent_covariance} and \ref{asmp:bounded_errors} hold, then we have
\[
 \max_{1\leq t< k\leq K}\sup_{x\in \RR}|F_{n}(x) - \Phi(x)  | \leq C (n^{-1/2}+(d\vee n)^{-3}).
\]
where $C$ is a constant dependent only upon $c_0$, $c_1$, and $c_2$.
\end{lemma}

\begin{proof}
Denote by $\cE$ the event that $\hat G = G^*$. We have
$$
F_{n}(x) - \Phi(x)\leq \tilde F_n(x)-\Phi(x)+\PP(\bar\cE),
$$
where $\tilde F_n(x)$ is the CDF of $n^{-1/2}\sum_{i=1}^n\vb_t^{*T}(\bar \Cb^{(i)}\bTheta^*_{\cdot k}-\eb_k)/(\sigma_{tk}/\Theta^*_{tt})$.  To control $\tilde F_n(x)-\Phi(x)$, we now verify the Lyapunov condition. As in the proof of Lemma \ref{lem:latent_gradient_hessian}, we can write
\[
\vb_t^{*T}(\bar \Cb^{(i)}\bTheta^*_{\cdot k}-\eb_k) = \vb_t^{*T}(\bar \Cb^{(i)}-\Cb^*)\bTheta^*_{\cdot k}.
\]
From Lemmas \ref{lem:conc_sum_ZiZi} - \ref{lem:conc_sum_EiEi} we see that the entries in $\bQ_i=(\bar \Cb^{(i)}-\Cb^*)\bTheta^*_{\cdot k}$ are sub-exponential with parameters $\alpha= C_1$ and $\nu = C_2$ which depend only upon $\lmax{\bThetaS}$, $\max_k \gammaS{k}$, and $\max_t \CS{t}{t}$.

Recall the definition of $\veetS$: $(\veetS)_t = 1$ and $(\veetS)_{-t} = -\wb^*_t = -(\bCS[-t,-t])^{-1}\bCS[-t,t]$. By the block matrix inverse formula, we can rewrite
\[
\wb_t^* = \left(\CS{t}{t} - \Cb^{*T}_{-t,t}(\bCS[-t,-t])^{-1}\bCS[-t,t]\right) \bThetaS[-t,t] = \frac{1}{\ThetaS{t}{t}}\bThetaS[-t,t].
\]
Using Lemma \ref{lem:pd_matrix_diag}, it follows that $||\wb^*_t||_2 \leq \lmax{\bThetaS} \max_t \CS{t}{t}$ and $||\vb^*_t||_2 \leq \lmax{\bThetaS} \max_t \CS{t}{t} + 1$.
From Lemma \ref{cor:sum_independent_subexponential} and the above, $\vb_t^{*T}\bQ_i$ is sub-exponential with parameters $\alpha = C_1$ and $\nu = ||\vb^*_t||_2 C_2\leq \left(\lmax{\bThetaS} \max_t \CS{t}{t} + 1\right)C_2$. Therefore, $\vb_t^{*T}(\bar \Cb^{(i)}\bTheta^*_{\cdot k}-\eb_k)$ has third moments bounded above by some constant $\rho$ that depends only upon $\lmax{\bThetaS}$, $\max_k \gammaS{k}$, and $\max_t \CS{t}{t}$. All three quantities are bounded above by constants per Assumptions \ref{asmp:bounded_latent_covariance} and \ref{asmp:bounded_errors}. Thus, $\max_{1\leq t< k\leq K}\sup_x(\tilde F_n(x)-\Phi(x))\leq Cn^{-1/2}$ by the classical Berry-Esseen Theorem, and therefore
$$
\max_{1\leq t< k\leq K}\sup_{x\in \RR}(F_{n}(x) - \Phi(x)  ) \leq C (n^{-1/2}+(d\vee n)^{-3}).
$$
Similarly, it can be shown that $\sup_{x\in \RR}(\Phi(x)-F_{n}(x) ) \leq C (n^{-1/2}+(d\vee n)^{-3}).$ This completes the proof.
\end{proof}

\begin{lemma}[Approximation for Asymptotic Variance]
\label{lem:latent_variable_variance}
Under Assumptions \ref{asmp:bounded_latent_covariance} and \ref{asmp:bounded_errors}, we have that
\begin{equation}\label{eqsigma}
\sigma^2_{tk} = \Theta_{t,k}^{*2} + \Theta^*_{t,t}\Theta^*_{k,k} + \Delta,
\end{equation}
where $|\Delta| \leq \frac{Cs_0}{m}$ and $C$ is a constant dependent only upon $c_1$, $c_2$ and $c_3$.
\end{lemma}

\begin{proof}
Recall that $\sigma^2_{tk}=\EE(\Theta_{tt}^*\vb_t^{*T}(\bar \Cb^{(i)}\bTheta^*_{\cdot k}-\eb_k))^2$. Using the identity $\vect(\Mb_1\Mb_2\Mb_3) = (\Mb_3^T \otimes \Mb_1)^T\vect (\Mb_2)(\Mb_3^T \otimes \Mb_1)$, we have
\begin{align*}
\sigma^2_{tk}
&= (\ThetaS{t}{t})^2(\bTheta_{\cdot k}^{*T} \otimes \veetS)^T\EE\left[\vect(\bar \Cb^{(i)})\vect(\bar \Cb^{(i)})^T \right] (\bTheta_{\cdot k}^{*T} \otimes \veetS) \numberthis \label{eqn:latent_variance_decomp}.
\end{align*}
\noindent{\it Computing the expectation}:
After some straightforward, albeit lengthy, algebra we can show that
\[
\EE\left[\vect(\bar \Cb^{(i)})\vect(\bar \Cb^{(i)})^T \right] = \Mb_1 + \Mb_2 + \Mb_3,
\]
where $\Mb_1 := \bCS\otimes\bCS$ and $\Mb_2 := [\bCS[\cdot j]\Cb_{\cdot i}^{*T}]_{ij}$. The matrices $\Mb_1$ and $\Mb_2$ contribute to the first two terms in  (\ref{eqsigma}). The term $\Mb_3:=\EE\left[\vect(\bar \Cb^{(i)})\vect(\bar \Cb^{(i)})^T \right] - \Mb_1 - \Mb_2$, however, is unique to the latent graph and contributes the higher order term $\Delta$ in (\ref{eqsigma}).

\noindent{\it Evaluating the first order terms}:
By (\ref{eqn:latent_variance_decomp}), we have
\[
\sigma^2_{tk}= (\ThetaS{t}{t})^2(\bTheta_{\cdot k}^{*} \otimes \veetS)^T(\Mb_1+\Mb_2)(\bTheta_{\cdot k}^{*} \otimes \veetS) + \Delta
\]
with $\Delta$ defined as
\begin{equation}
\label{eqn:latent_variance2}
\Delta := (\ThetaS{t}{t})^2 (\bTheta_{\cdot k}^{*} \otimes \veetS)^T \Mb_3 (\bTheta_{\cdot k}^{*} \otimes \veetS).
\end{equation}
Next, we observe that
\begin{align*}
(\bTheta_{\cdot k}^{*} \otimes \veetS)^T\Mb_1(\bTheta_{\cdot k}^{*} \otimes \veetS) &= \bTheta_{\cdot k}^{*T}\bCS\bTheta_{\cdot k}^{*}  \otimes \veetST \bCS \veetS \\
&= \frac{\ThetaS{k}{k}}{\ThetaS{t}{t}},
\end{align*}
where we used that $\veetST \bCS \veetS = (\ThetaS{t}{t})^{-1}$. Similarly, we can find that
\[
(\bTheta_{\cdot k}^{*} \otimes \veetS)^T \Mb_2 (\bTheta_{\cdot k}^{*} \otimes \veetS) = \frac{\Theta_{tk}^{*2}}{\Theta_{tt}^{*2}}.
\]

\noindent{\it Bounding the higher order terms}:
What remains is to bound the magnitude of the term $\Delta$ in \eqref{eqn:latent_variance2}. Lengthy algebra yields:
\[
|(\bTheta_{\cdot k}^{*} \otimes \veetS)^T \Mb_3 (\bTheta_{\cdot k}^{*} \otimes \veetS)| \leq \frac{C'}{m}(\bTheta_{\cdot k}^{*} \otimes \veetS)^T\left(\Mb_4 + \Mb_5 + \Mb_6 + \Mb_7 \right)(\bTheta_{\cdot k}^{*} \otimes \veetS),
\]
where $C'$ depends only on $c_1$, $c_2$ and $c_3$. Here, $\Mb_4 = \Ib \otimes (\bone\bone^T)$. For $l=5,6,7$ the matrices $\Mb_l$ are defined block-wise by
\[
\Mb_{5;ij} := \begin{cases}
\bone\eb_i^T &\text{ if } i \neq j \\
\bzero &\text{ o/w }
\end{cases} \text{ and }
\Mb_{6;ij} := \begin{cases}
\eb_j\bone^T &\text{ if } i \neq j \\
\bzero &\text{ o/w }
\end{cases} \text{ and }
\Mb_{7;ij} := \begin{cases}
\Ib &\text{ if } i \neq j \\
\bzero &\text{ o/w }
\end{cases}.
\]
Further lengthy algebra gives:
$$
|(\bTheta_{\cdot k}^{*} \otimes \veetS)^T \Mb_4 (\bTheta_{\cdot k}^{*} \otimes \veetS)| \le s_0\frac{2}{c_1^2}(1+\frac{c_2^2}{c_1^2}),
$$
$$
|(\bTheta_{\cdot k}^{*} \otimes \veetS)^T \Mb_5 (\bTheta_{\cdot k}^{*} \otimes \veetS)| \le s_0\frac{2\sqrt{2}}{c_1^2}(1+\frac{c_2^2}{c_1^2}),$$
$$
|(\bTheta_{\cdot k}^{*} \otimes \veetS)^T \Mb_6 (\bTheta_{\cdot k}^{*} \otimes \veetS)| \le s_0\frac{2\sqrt{2}(c_1^2 + c_2^2)}{c_1^4}\sqrt{1+\frac{c_2^2}{c_1^2}},
$$
and
$$|(\bTheta_{\cdot k}^{*} \otimes \veetS)^T \Mb_7 (\bTheta_{\cdot k}^{*} \otimes \veetS)| \le s_0\frac{2(c_1^2 + c_2^2)}{c_1^4}.$$

Plugging these bounds into the expression for $\Delta$ in \eqref{eqn:latent_variance2}, we obtain
\[
|\Delta| \leq \frac{Cs_0}{m},
\]
concluding the proof.
\end{proof}

\section{Concentration Results}
\label{sec:concentration_results}
\label{sec:concentration_of_estimators}
The lemmas below provide important results regarding the concentration properties of some of the estimators $\bCH$ and variables $\bZ$.
\begin{lemma}
\label{lem:conc_sum_ZiZi}
~~
\begin{itemize}
\item[(a)] $ \bZ_i\bZ_i^T$ consists of entries which are sub-exponential with parameters $\alpha = 4\max_{t}(\CS{t}{t})^2$ and $\nu = 2\sqrt{2}\max_{t}(\CS{t}{t})^2$,
\item[(b)] $\PP\left( \Big\| \frac{1}{n} \sum_{i=1}^n \bZ_i\bZ_i^T - \Cb^*\Big\|_{\max} \geq C \max_{t}(\CS{t}{t})^2 \sqrt{\frac{\log (K\vee n)}{n}} \right) \leq \frac{2}{(K\vee n)^3}$
\item[(c)] $\Zb_i\Zb_i^T\ub$ consists of entries which are sub-exponential with parameters $\alpha = 4\max_{t}(\CS{t}{t})^2$ and $\nu = 2\sqrt{2}||\ub||_2 \max_{t}(\CS{t}{t})^2$, and
\item[(d)] $\PP\left( \Big\| \frac{1}{n} \sum_{i=1}^n \bigl(\bZ_i\bZ_i^T - \Cb^*\bigr)\ub\Big\|_{\max} \geq C  ||\ub||_2 \max_{t}(\CS{t}{t})^2 \sqrt{\frac{\log (K\vee n)}{n}} \right) \leq \frac{2}{(K\vee n)^3}$,
\end{itemize}
where $C = 4\sqrt{3}$ and $\ub \in \RR^K$.
\end{lemma}

\begin{proof}
From Lemma \ref{lem:jointly_gaussian_product}, each element in the matrices $\bZ_i\bZ_i^T$ are sub-exponential with parameters $\alpha = 4\max_{t}(\CS{t}{t})^2$ and $\nu = 2\sqrt{2}\max_{t}(\CS{t}{t})^2$. Therefore by Corollary \ref{cor:sum_independent_subexponential}, the entries in $\frac{1}{n}\sum_{i=1}^n \bZ_i\bZ_i^T$ are sub-exponential with parameters $\alpha = \frac{4}{n}\max_{t}(\CS{t}{t})^2$ and $\nu = \frac{2\sqrt{2}}{\sqrt{n}}\max_{t}(\CS{t}{t})^2$. Therefore by the tail bound for sub-exponential random variables, we see that
\begin{align*}
&\PP\biggl(\biggl(\frac{1}{n} \sum_{i=1}^n \bZ_i\bZ_i^T - \Cb^*\biggr)_{s,t} \geq D_1\sqrt{\frac{\log (K\vee n)}{n}}\biggr)\\ &\leq \begin{cases}
\exp\left(-\frac{(\log (K\vee n)) D_1^2}{16\max_{t}(\CS{t}{t})^4} \right) & \text{ if } 0 \leq D_1 \sqrt{\frac{\log (K\vee n)}{n}} \leq 2\max_{t}(\CS{t}{t})^2 \\
\exp\left(\frac{-D_1\sqrt{n\log (K\vee n)} }{8\max_{t}(\CS{t}{t})^2} \right) & \text{ if } D_1\sqrt{\frac{\log (K\vee n)}{n}} >  2\max_{t}(\CS{t}{t})^2
\end{cases}
\end{align*}
for arbitrary $D_1 > 0$. Observe that for $n$ sufficiently large, $D_1 \sqrt{\log (K\vee n) / n} \leq 2\max_{t}(\CS{t}{t})^2$, and thus we need only consider this case. Choose $D_1 \geq 4 \sqrt{3} \max_{t}(\CS{t}{t})^2$. Then it is clear that
\[\PP\biggl( \biggl(\frac{1}{n} \sum_{i=1}^n \bZ_i\bZ_i^T - \Cb^*\biggr)_{s,t} \geq D_1\sqrt{\frac{\log (K\vee n)}{n}} \biggr) \leq \frac{1}{(K\vee n)^5}. \]
By applying the union bound across all entries in the matrix, we get the desired result that
\[\PP\left( \Big\| \frac{1}{n} \sum_{i=1}^n \bZ_i\bZ_i^T - \Cb^*\Big\|_{\max} \geq D_1\sqrt{\frac{\log (K\vee n)}{n}} \right) \leq \frac{2}{(K\vee n)^3}, \]
concluding the proof of parts (a) and (b). The proof of (c) and (d) are very similar and omitted.
\end{proof}

\begin{lemma}
\label{lem:conc_sum_ZiEi}
Let $C = 2\sqrt{3}$, $\sigma_s^2 = \frac{1}{|\Gs{s}|^2} \sum_{i \in \Gs{s}} \gamma_i$, and $\bM_i = \bZ_i\bE_i^T\Ab^*\Bb^{*-1}$. Then,
\begin{itemize}
\item[(a)] $\left(\bM_i\right)_{s,t}$ is sub-exponential with parameters $\alpha = \sqrt{2}\max\left(\sigma_s^2,\CS{t}{t} \right)$ and $\nu = \sqrt{2}\max\left(\sigma_s^2,\CS{t}{t} \right)$,
\item[(b)] $\PP\left( \Big\|\frac{1}{n}\sum_{i=1}^n \bM_i \Big\|_{\max} \geq C \max_t (\sigma_t^2 \vee \CS{t}{t})\sqrt{\frac{\log (K\vee n)}{n}} \right) \leq \frac{2}{(K\vee n)^3}$
\item[(c)] $\left(\bM_i \ub\right)_s$ is sub-exponential with parameters $\alpha = \sqrt{2}\max\left(\sigma_s^2,\CS{t}{t} \right)$ and $\nu = \sqrt{2}||\ub||_2\max\left(\sigma_s^2,\CS{t}{t} \right)$, and
\item[(d)] $\PP\left( \Big\|\frac{1}{n}\sum_{i=1}^n \bM_i \ub\Big\|_{\infty} \geq C \max_t (\sigma_t^2 \vee \CS{t}{t}) ||\ub||_2 \sqrt{\frac{\log (K\vee n)}{n}} \right) \leq \frac{2}{(K\vee n)^3}$,
\end{itemize}
where $\ub \in \RR^K$.
\end{lemma}

\begin{proof}
Let $\bM =\sum_{i=1}^n\bM_i$. From Lemma \ref{lem:ATAIA}, $\bY_1 = \Bb^{*-1}\Ab^{*T}\bE_1$ is a $K$-dimensional vector where the $k^{th}$ entry is given by
\[
(Y_1)_k = \frac{1}{|\Gs{k}|} \sum_{i \in \Gs{k}} (E_1)_i.
\]
Because the errors are all independent mean zero Gaussian random variables, $(Y_1)_s \sim \cN(0,\sigma_s^2)$. Therefore, as $Y_1$ is independent of $Z_1$ by definition, $\EE[(Y_1)_s (Z_1)_t] = \EE[(Y_1)_s]\EE[(Z_1)_t] = 0$. Further, Lemma \ref{lem:jointly_gaussian_product} gives that $(Y_1)_s (Z_1)_t$ is sub-exponential with parameters $\alpha = \nu = \sqrt{2}\max\left(\sigma_s^2,\CS{t}{t} \right)$.

Using the independence of the samples, Corollary \ref{cor:sum_independent_subexponential} gives that $\bM_{s,t}$ is sub-exponential with parameters $\alpha =\sqrt{2}\max\left(\sigma_s^2,\CS{t}{t} \right)$ and $\nu = \sqrt{2n}\max\left(\sigma_s^2,\CS{t}{t} \right)$. Then, Corollary \ref{cor:tail_bound_sum} gives that for arbitrary choice of $D_1 > 0$,
\begin{align*}
&\PP\left(\frac{1}{n}\bM_{s,t} \geq D_1\sqrt{\frac{\log (K\vee n)}{n}}\right)\\
&\leq \begin{cases}
\exp\left(-\frac{(\log (K\vee n)) D_1^2}{4\max\left(\sigma_s^2,\CS{t}{t} \right)^2} \right) & \text{ if } 0 \leq D_1 \sqrt{\frac{\log (K\vee n)}{n}} \leq \sqrt{2}\max\left(\sigma_s^2,\CS{t}{t} \right) \\
\exp\left(\frac{-D_1\sqrt{n\log (K\vee n)}}{\sqrt{2}\max\left(\sigma_s^2,\CS{t}{t} \right)} \right) & \text{ if } D_1 \sqrt{\frac{\log (K\vee n)}{n}} >  \sqrt{2}\max\left(\sigma_s^2,\CS{t}{t} \right).
\end{cases}
\end{align*}
Observe that for $n$ sufficiently large, $D_1 \sqrt{\log K / n} \leq \sqrt{2}\max\left(\sigma_s^2,\CS{t}{t} \right)$. If we choose $D_1 \geq 2\sqrt{3}\max\left(\sigma_s^2,\CS{t}{t} \right)$, then we obtain that for $n$ sufficiently large,
\[
\PP\left(\frac{1}{n}\bM_{s,t} \geq D_1\sqrt{\frac{\log (K\vee n)}{n}}\right) \leq \frac{1}{(K\vee n)^5}.
\]
Then by the union bound we can obtain
\[
\PP\left(\Big\|\frac{1}{n}\sum_{i=1}^n \bZ_i\bE_i^T\Ab^*\Bb^{*-1}\Big\|_{\max}\geq D_1\sqrt{\frac{\log (K\vee n)}{n}}\right) \leq \frac{2}{(K\vee n)^3}
\]
for $D_1 \geq 2\sqrt{3}\max\left(\max_s\sigma_s^2,\max_t\CS{t}{t} \right)$, concluding the proof of parts (a) and (b). The proof of (c) and (d) are very similar and omitted.
\end{proof}

\begin{lemma}
\label{lem:conc_sum_EiEi}
Recall that $m=\min_{k}|\Gs{k}|$, and let $\bM_i = \Bb^{*-1}\Ab^{*T}\bE_i\bE_i^T\Ab^*\Bb^{*-1}$ and $\bmu = \Bb^{*-1}\Ab^{*T}\bGammaS\Ab^*\Bb^{*-1}$. Then,
\begin{itemize}
\item[(a)]$\left( \bM_i - \bmu \right)_{t,k}$ is sub-exponential with parameters $\alpha_{t,k} = \frac{\sqrt{2}}{|\Gs{k}||\Gs{t}|}\max_{i \in \Gs{t} \cup \Gs{k}} \gammaS{i}$ and $\nu_{t,k} = \sqrt{\frac{2}{ |\Gs{k}| |\Gs{t}|}}\max_{i \in \Gs{t} \cup \Gs{k}} \gammaS{i}$,
\item[(b)] $\PP\left(\Big\|\frac{1}{n}\sum_{i=1}^n \bM_i - \bmu \Big\|_{\max}\geq C\max_{k} \gamma_k \sqrt{\frac{\log (K\vee n)}{n m^2}}\right) \leq \frac{2}{(K\vee n)^3}$
\item[(c)]$\left(\left(\bM_i - \bmu\right)\ub\right)_t$ is sub-exponential with parameters $\alpha_t = \frac{\sqrt{2}}{m^2}\max_{k}\gammaS{k} $ and $\nu_t = \sqrt{\frac{2}{ m^2}}||\ub||_2\max_{k} \gammaS{k}$, and
\item[(d)] $\PP\left(\Big\|\frac{1}{n}\sum_{i=1}^n \left(\bM_i - \bmu\right)\ub \Big\|_{\max}\geq C||\ub||_2\max_{k} \gamma_k \sqrt{\frac{\log (K\vee n)}{n m^2}}\right) \leq \frac{2}{(K\vee n)^3}$,
\end{itemize}
where $C = 2\sqrt{3}$ and and $\ub \in \RR^K$.
\end{lemma}

\begin{proof}
To obtain the results, we first bound the sum $\bM := \sum_{i=1}^n \bM_i$ entrywise. From Lemma \ref{lem:ATAIE}, Corollary \ref{cor:independent_gaussian_product} and Corollary \ref{cor:sum_independent_subexponential}, we have that $(\bM_i)_{t,k}$ is sub-exponential with parameters
\[\alpha_{t,k} = \frac{\sqrt{2}}{|\Gs{k}||\Gs{t}|}\max_{i \in \Gs{t} \cup \Gs{k}} \gammaS{i} \quad \text{and} \quad \nu_{t,k} = \sqrt{\frac{2}{ |\Gs{k}| |\Gs{t}|}}\max_{i \in \Gs{t} \cup \Gs{k}} \gammaS{i}.\]
Therefore $\frac{1}{n} M_{t,k}$ is sub-exponential with parameters
\[
\alpha = \frac{\sqrt{2}}{n|\Gs{k}||\Gs{t}|}\max_{i \in \Gs{t} \cup \Gs{k}} \gammaS{i} \quad \text{and}\quad \nu = \sqrt{\frac{2}{ n |\Gs{k}| |\Gs{t}|}}\max_{i \in \Gs{t} \cup \Gs{k}} \gammaS{i}.
\]
Next, observe that $\mu_{t,k} = \EE[M_{t,k}]$ and denote
$N = \max_{i \in \Gs{t} \cup \Gs{k}} \gammaS{i}$. Then, from Lemma \ref{lem:tail_bound_subexponential}, we obtain
\begin{align*}
&\PP\left(\frac{1}{n}M_{t,k} - \mu_{t,k} \geq D_1\sqrt{\frac{\log (K\vee n)}{n |\Gs{k}| |\Gs{t}|}}\right) \\
&\leq \begin{cases}
\exp\left(-\frac{(\log (K\vee n)) D_1^2}{4N^2} \right) & \text{ if } 0 \leq D_1\sqrt{\frac{\log (K\vee n)}{n |\Gs{k}| |\Gs{t}|}} \leq \sqrt{2} N  \\
\exp\left(\frac{-D_1\sqrt{n\log (K\vee n)}}{\sqrt{2}N} \right) & \text{ if } D_1\sqrt{\frac{\log (K\vee n)}{n |\Gs{k}| |\Gs{t}|}} > \sqrt{2}N.
\end{cases}
\end{align*}
Observe that for $n$ sufficiently large, $D_1\sqrt{\frac{\log (K\vee n)}{n |\Gs{k}| |\Gs{t}|}} \leq \sqrt{2} N$. If we choose $D_1 \geq 2\sqrt{3}N$, then we obtain that for $n$ sufficiently large,
\[
\PP\left(\frac{1}{n}M_{t,k} - \mu_{t,k} \geq D_1\sqrt{\frac{\log (K\vee n)}{n |\Gs{k}| |\Gs{t}|}}\right) \leq \frac{1}{(K\vee n)^5}.
\]
Therefore by taking the union bound, lower bounding $n |\Gs{k}| |\Gs{t}|$ by $n m^2$ and choosing $D_1 \geq 2\sqrt{3}\max_{k} \gammaS{k}$,
\[
\PP\left(\Big\|\frac{1}{n}\sum_{i=1}^n \Bb^{*-1}\Ab^{*T}(\bE_i\bE_i^T-\bGamma^*)\Ab^*\Bb^{*-1}\Big\|_{\max}\geq D_1\sqrt{\frac{\log (K\vee n)}{n m^2}}\right) \leq \frac{2}{(K\vee n)^3},
\]
concluding the proof of parts (a) and (b). The proof of (c) and (d) are very similar and omitted.
\end{proof}

\section{Auxiliary Technical lemmas}
\label{sec:misc_results}
\begin{lemma}
\label{lem:ATAIA}
For $1\leq k\leq K$, denote $m_k=|G_k^*|$. Then the matrix $(\Ab^{*T}\Ab^*)^{-1}\Ab^{*T}$ is a $K\times d$ dimensional matrix given as
\[
[(\Ab^{*T}\Ab^*)^{-1}\Ab^{*T}]_{k,i} = \begin{cases}
\frac{1}{m_k} & \text{ if } i \in \Gs{k} \\
0 & \text{ otherwise.}
\end{cases}
\]
\end{lemma}

\begin{proof}
First, we must calculate $\Bb^{*-1}\Ab^{*T}$. For $1\leq k\leq K$, denote $m_k=|G_k^*|$ and let $\eb_k$ be a unit vector in $\RR^K$ with $1$ on the $k$ position and $0$ otherwise. Without loss of generality, we permute the rows of $\Ab^*$ such that for any $1\leq k\leq K$ $\Ab^*_{j\cdot}=\eb_k$, for $\sum_{i=1}^{k-1}m_{i}+1\leq j<\sum_{i=1}^{k}m_{i}+1$ -- that is rows are ordered according to ascending group index. Here, for notational simplicity, we let $m_0=0$. Thus, $\Ab^{*T}\Ab^*=\textrm{diag}(m_1,...,m_K)$ and the result follows immediately.
\end{proof}

\begin{lemma}
\label{lem:ATAIE}
The matrices $\Bb^{*-1}\Ab^{*T}(\bE_i\bE_i^T)\Ab^*\Bb^{*-1}$ and $\Bb^{*-1}\Ab^{*T}(\Gamma^*)\Ab^*\Bb^{*-1}$ are given by
\[
(\Bb^{*-1}\Ab^{*T}(\bE_i\bE_i^T)\Ab^*\Bb^{*-1})_{t,k} = \frac{1}{|\Gs{t}\Gs{k}|}\sum_{p \in \Gs{t}}\sum_{q \in \Gs{k}}E_{i,p}E_{i,q}
\]
and
\[
(\Bb^{*-1}\Ab^{*T}(\Gamma^*)\Ab^*\Bb^{*-1})_{t,k} = \begin{cases} \frac{1}{|\Gs{t}|^2}\sum_{p \in \Gs{t}}\gamma^*_p & \text{ if } t = k \\
0 & \text{ otherwise.}
\end{cases}
\]
\end{lemma}

\begin{proof}
The result can be obtained by a straightforward computation.
\end{proof}

\begin{lemma}[Restricted Eigenvalue Condition for $\hat\Cb$]
\label{lem:latent_re_condition}
If Assumptions \ref{asmp:bounded_latent_covariance} and \ref{asmp:bounded_errors} hold, then the matrix $\hat\Cb$ satisfies with probability at least $1- \frac{C}{(K\vee n)^3}$,
\[
\kappa \leq \min\left\{\frac{\vb^T \hat\Cb \vb}{||\vb||_2^2} : \vb \in \RR^{K}\setminus\{0\}, ||\vb_{\bar{S}}||_1 \leq 3||\vb_S||_1 \right\}, \text{ and}
\]
\[
\kappa \leq \min\left\{\frac{\vb^T \hat\Cb_{-t,-t} \vb}{||\vb||_2^2} : \vb \in \RR^{K}\setminus\{0\}, ||\vb_{\bar{S'}}||_1 \leq 3||\vb_{S'}||_1 \right\},
\]
where $\kappa \geq \frac{3}{4c_1} > 0$.
\end{lemma}

\begin{proof}
We begin by proving the first claim. By Lemma \ref{lem:latent_C_consistency}, we have that $\|\hat \Cb-\Cb^*\|_{\max}\leq C_1\sqrt{\frac{\log (K\vee n)}{n}}$ with high probability. Therefore, for $K$ sufficiently large and for any $\vb \in \RR^K\setminus \{0\}$,
\[
\frac{\vb^T \hat\Cb \vb}{||\vb||_2^2} \geq \frac{3}{4}\frac{\vb^T \Cb^* \vb}{||\vb||_2^2}.
\]
The proof is then done for $\kappa = \frac{3}{4c_1}$ as we assume the minimum eigenvalue of $\Cb^*$ is bounded below by $c_0^{-1}$. The proof of the second claim is identical because $\Cb^*$ is positive semidefinite, and it is well known that the minimum eigenvalue of any principal submatrix $\Cb^*_{-t,-t}$ is bounded below by $\lambda_{\min}(\Cb^*) \geq c_0^{-1}$.
\end{proof}

\begin{lemma}
\label{lem:pd_matrix_diag}
Let $\Mb$ be a $n \times n$ positive definite matrix and denote its inverse by $\Lb$. Then, for all $i = 1,\dots,n$
\[
M_{i,i} L_{i,i} \geq 1.
\]
\end{lemma}
\begin{proof}
By the block matrix inverse formula, it follows that
\begin{equation}
\label{eqn:pd_matrix_diag1}
M_{i,i}^{-1} = L_{i,i} - \Lb_{-i,i}^T \Lb_{-i,-i}^{-1}\Lb_{-i,i}.
\end{equation}
Because $\Mb$ is positive definite, so is $\Lb$. Recall that a matrix is positive definite if and only if all its principal minors are also positive definite. Therefore, $\Lb_{-i,-i}$ is positive definite, as is $\Lb_{-i,-i}^{-1}$. Therefore, $\Lb_{-i,i}^T \Lb_{-i,-i}^{-1}\Lb_{-i,i} \geq 0$ and \eqref{eqn:pd_matrix_diag1} becomes $M_{i,i}^{-1} \leq L_{i,i}$. Lastly, if a matrix is positive definite, all its diagonal elements must be nonnegative, giving that $M_{i,i}L_{i,i} \geq 1$ as desired.
\end{proof}


\section{Basic Tail Bounds for Random Variables}

This section collects some basic tail probability results for random variables. The proof is standard and omitted.
\begin{lemma}
\label{lem:jointly_gaussian_product}
Let $\Yb = (Y_1,Y_2)$ be a jointly Gaussian random vector with covariance matrix $\Cb$. Then $Y_1Y_2$ is sub-exponential with parameters $\alpha = 4\lmax{\Cb_Y}$ and $\nu = 2\sqrt{2}\lmax{\Cb_Y}$.
\end{lemma}

\begin{corollary}
\label{cor:independent_gaussian_product}
Let $Y_1 \sim \cN(0,\sigma_1^2)$ and $Y_2 \sim \cN(0,\sigma_2^2)$ where $\sigma_1^2 \geq \sigma_2^2$.Then $Y_1Y_2$ is sub-exponential with parameters $\alpha = \sqrt{2}\sigma_1^2$ and $\nu = \sqrt{2}\sigma_1^2$.
\end{corollary}

\begin{corollary}
\label{cor:sum_independent_subexponential}
Consider $\sum_{i=1}^n X_i$ where $X_i$ are centered, independent sub-exponential random variables. Then $Y = \sum_{i=1}^n X_i$ is sub-exponential with parameters $\alpha = \max_i \alpha_i$ and $\nu = \sqrt{\sum_{i=1}^n \nu_i^2}$.
\end{corollary}

\begin{lemma}[Tail Bound for Sub-Exponential Random Variables]
\label{lem:tail_bound_subexponential}
Let $X$ be a sub-exponential random variable with mean $\mu$ and parameters $\alpha$ and $\nu$. Then
\[
\PP(X-\mu \geq t) \leq \begin{cases}
\exp(-\frac{t^2}{2\nu^2} ) & \text{ for } 0 \leq t \leq \frac{\nu^2}{\alpha} \\
\exp(-\frac{t}{2\alpha} ) & \text{ for }  t > \frac{\nu^2}{\alpha}.
\end{cases}
\]
\end{lemma}

\begin{corollary}
\label{cor:tail_bound_sum}
Consider $Y = \sum_{i=1}^n X_i$, where $X_i$ are centered, independent sub-exponential random variables. Let $\alpha = \max_i \alpha_i$ and $\nu = \sqrt{\sum_{i=1}^n \nu_i^2}$. Then,
\[
\PP(\frac{1}{n}\sum_{i=1}^n X_i \geq t) \leq \begin{cases}
\exp(-\frac{n t^2}{2\nu^2 / n} ) & \text{ for } 0 \leq t \leq \frac{\nu^2}{n\alpha} \\
\exp(-\frac{nt}{2\alpha} ) & \text{ for }  t > \frac{\nu^2}{n\alpha}.
\end{cases}
\]
\end{corollary}

\section{Construction of a Pre-Clustering Variance Estimator} \label{pregamma}
We include in this section the construction of the  pre-clustering  estimator of $\Gamma$  needed as an input of the PECOK algorithm of Section \ref{sec:introduction} above.  For any $a,b\in [d]$, define
\begin{equation} \label{eq:definition_V}
V(a,b):=   \max_{c,d \in [p]\setminus\{a,b\}} \frac{\left| (\widehat \Sigma_{ac}-\widehat\Sigma_{ad})-(\widehat\Sigma_{bc}-\widehat\Sigma_{bd}) \right|}{\sqrt{\widehat \Sigma_{cc}+ \widehat \Sigma_{dd}-2 \widehat \Sigma_{cd}}}\ ,
\end{equation}
with the convention $0/0=0$.  Guided by the block structure of $\Sigma$, we define
\[b_1(a):= \argmin_{b\in [p]\setminus\{a\}}V(a,b)\quad \text{ and }\quad b_2(a):= \argmin_{b\in [p]\setminus\{a,b_1(a)\}}V(a,b) ,\]
to be two  elements ''close'' to $a$, that is  two  indices  $b_1 = b_1(a)$ and $b_2 = b_2(a)$
such that the empirical covariance difference
$ \widehat \Sigma_{b_{i}c}- \widehat \Sigma_{b_{i}d}$, $i =1,2$,  is most similar to
$ \widehat \Sigma_{ac}- \widehat \Sigma_{ad}$, for all variables $c$ and $d$ not equal to $a$ or $b_{i}$, $i = 1,2$.  It is expected that $b_1(a)$ and $b_2(a)$ either belong to the same group as $a$, or  belong to some ''close'' groups.
Then, our estimator  $\widetilde \Gamma$ is a diagonal matrix,  defined by
\begin{equation}\label{eq:estim:gamma2}
\widetilde \Gamma_{aa}=  \widehat \Sigma_{aa}+ \widehat \Sigma_{b_{1}(a)b_{2}(a)}- \widehat \Sigma_{ab_{1}(a)}- \widehat \Sigma_{ab_{2}(a)},
\quad \text{ for $a=1,\ldots, d$.}
\end{equation}
Intuitively, $\widetilde \Gamma_{aa}$ should be close to $\Sigma_{aa}+ \Sigma_{b_{1}(a)b_{2}(a)}- \Sigma_{ab_{1}(a)}-\Sigma_{ab_{2}(a)}$, which is equal to $\Gamma_{aa}$ in the favorable event where both $b_1(a)$ and $b_2(a)$ belong to the same group as $a$.

In general, $b_1(a)$ and $b_2(a)$ cannot be guaranteed to belong to the same group as $a$. Nevertheless, these two surrogates $b_1(a)$ and $b_2(a)$  are close enough to $a$ so that $\|\widetilde{\Gamma} - \Gamma\|_{\max} \lesssim |\Gamma|_{\max}\sqrt{\log d/n}$. This last fact and the above construction are shown in \cite{Bunea2018}.

\section{Comparison with Cai et al (2016)}
\label{sec:indepth_comparison}

 In our work, we bound $\lambda_{\min}(\bS^*) \geq c_1$ and $\max_{t} S^*_{t,t} \leq c_2$  and the sparsity of $\Omega_{\cdot k}$ by $s_1$. We show that the CLIME estimator satisfies
\begin{equation}\label{eqrateomega}
\|\hat\bOmega_{\cdot k}-\bOmega_{\cdot k}^*\|_1\lesssim s_1\sqrt{\frac{\log (K\vee n)}{n}}.
\end{equation}
Let us denote our parameter space by 
$$
\cG_1=\{\bOmega: \max_k\|\Omega_{\cdot k}\|_0\leq s_1, \max_t\{(\bOmega^{-1})_{tt}\}\leq c_2, \lambda_{\max}(\bOmega)\leq 1/c_1\}.
$$

By contrast, Cai et al. (2016) considered the following parameter space for the precision matrix $\bOmega=\bS^{-1}$ (in the exact sparse case)
$$
\cG=\{\bOmega: \max_k\|\Omega_{\cdot k}\|_0\leq s_1, \|\bOmega\|_1\leq M_n, \kappa(\bOmega)=\lambda_{\max}(\bOmega)/\lambda_{\min}(\bOmega)\leq M_1\},
$$
where $\|\bOmega\|_1$ is the matrix $\ell_1$-norm, $M_n$ is allowed to increase with $n$ and $M_1$ is a constant bound for the condition number. Their minimax lower bound over $\cG$ depends on $M_n$, that is for any $\hat\bOmega_{\cdot k}$,
\begin{equation}\label{eqrateomega2}
\sup_{\bOmega\in\cG}\|\hat\bOmega_{\cdot k}-\bOmega_{\cdot k}\|_1\gtrsim M_ns_1\sqrt{\frac{\log K}{n}}.
\end{equation}

It seems that the upper bound (\ref{eqrateomega}) and the lower bound (\ref{eqrateomega2}) contradict with each other. However, this is not the case because the parameter space $\cG_1$ under which the estimator is developed is different from the parameter space $\cG$ in the lower bound. This therefore opens up the possibility of obtaining different rates over different parameter spaces, for instance one which does not depend on $||\bOmega^*||_1$, like in our case. 
However, neither parameter space can be viewed as a full relaxation of the other. Below we give  explicit examples that  show  how the parameter spaces differ.

Consider the sequence of precision matrices, indexed by $K$
\[
\bOmega_K = |\cS(K)|\Ib + \bone_{\cS(K)}\bone_{\cS(K)}^\top\in\RR^K,
\]
where $\cS(K)$ is an arbitrary subset of $[K]=\{1,2,...,K\}$ with $|\cS(K)| = K^{1/4}$ and $\bone_{\cS(K)}$ is the vector with 1's in the indices indicated by $\cS(K)$ and 0's elsewhere. By using the Sherman-Morrison formula, we can verify that the corresponding covariance matrix is
\[
\bS_K = \frac{1}{|\cS(K)|}\Ib - \frac{1}{2|\cS(K)|^2}\bone_{\cS(K)}\bone_{\cS(K)}^\top.
\]
It is easy to check that
\begin{align*}
\lambda_{\min}(\bOmega_K) &= |\cS(K)| = K^{1/4}, \\
\lambda_{\max}(\bOmega_K) &= 2|\cS(K)| = 2K^{1/4},\\
\lambda_{\min}(\bS_K) &= \frac{1}{2|\cS(K)|} = K^{-1/4}/2.
\end{align*}
Thus $\kappa(\bOmega_K) = 2$, and therefore Theorem 4.2 from Cai et al. (2016) can be applied to this sequence of parameters $\bOmega_K$ with $s_1=K^{1/4}+1$ and $M_n=2K^{1/4}$.
However, $\bOmega_K$ does not belong to our parameter space $\cG_1$ because for any choice of $c_1$, once $K$ is sufficiently large, $\lambda_{\min}(\bS_K) < c_1$. 

Similarly, we can modify the above example such that $\bOmega_K$ satisfies our condition but does not belong to $\cG$. For instance, consider
$$
\bOmega_K = \Ib -\frac{1}{1+|\cS(K)|} \bone_{\cS(K)}\bone_{\cS(K)}^\top\in\RR^K.
$$
We can show that
\[
\bS_K = \Ib + \bone_{\cS(K)}\bone_{\cS(K)}^\top,
\]
and
\begin{align*}
\lambda_{\min}(\bOmega_K) &= \frac{1}{K^{1/4}+1}, \\
\lambda_{\max}(\bOmega_K) &= 1,\\
\lambda_{\min}(\bS_K) &= 1.
\end{align*}
Our conditions $\lambda_{\min}(\bS_K) \geq c_1$ and $\max_{t} (\bS_K)_{t,t} \leq c_2$ hold, and thus we obtain the rate (\ref{eqrateomega}) for the CLIME estimator. However, $\kappa(\bOmega_K) = K^{1/4}+1\rightarrow\infty$, as $K\rightarrow\infty$. Thus, $\bOmega_K$ does not belong to the parameter space $\cG$, and Theorem 4.2 in Cai et al. (2016) is not applicable.

Since our goal is to make inference on each entry of the precision matrix say $\Omega_{tk}$, a lower bound on the minimum eigenvalue on $\bS$ is a mild and standard assumption for any high-dimensional inference on graphical models. After some careful analysis as explained in the main text, we show that under Assumption 4.1 and 4.2 the upper bound for the CLIME estimator does not depend on $||\bOmega^*||_1$ or $\lambda_{\max}(\bS^*)$.

\section{B-H procedures for FDR control}\label{app:FDR}

In this section, we analyze the theoretical properties of the B-H procedure for the cluster-average graph. The procedure for latent variable graph is identical. Recall that the test statistic for $H_{0,tk}: \Omega_{t,k}^*=0$ is $\widetilde W_{t,k}$. The B-H procedure is defined as follows. Given the desired FDR level $\alpha$, define
$$
\hat \rho=\inf\Big\{0\leq \rho\leq 2\sqrt{\log K}: \frac{2(1-\Phi(\rho))|\cH|}{\max\{\sum_{1\leq t\leq k\leq K} I(|\widetilde W_{t,k}|\geq \rho),1\}}\leq \alpha\Big\},
$$
where $|\cH|=K(K-1)/2$ is the total number of hypotheses. If $\hat\rho$ does not exist, then we set $\hat\rho= 2\sqrt{\log K}$. In the above definition, we restrict $\rho\leq 2\sqrt{\log K}$ in order to apply the Cramer-type moderate deviation result \citep{liu2013gaussian}. The B-H procedure says that we reject $H_{0,tk}$ if $|\widetilde W_{t,k}|\geq \hat\rho$. The FDR and FDP are defined as
$$
\FDP=\frac{\sum_{(t,k)\in\cH_0}I(|\widetilde W_{t,k}|\geq \hat\rho)}{\max \{\sum_{(t,k)\in\cH}I(|\widetilde W_{t,k}|\geq \hat\rho),1\}},
$$
where $\cH_{0}\coloneqq \{(t,k):\, 1\leq t<k\leq K,\textrm{ such that } \Omega^*_{t,k} = 0\}$, and $\FDR=\EE(\FDP)$. Let 
$$
\cA=\Big\{(t,k): \frac{\Omega^*_{t,k}}{\sqrt{\Omega^*_{k,k}\Omega^*_{t,t}}}\geq 4\sqrt{\log K/n}\Big\}
$$
denote the set of strong signals. The following theorem on the B-H procedure holds.
\begin{proposition}\label{propbh}
Assume that the conditions in Theorem \ref{thm:xi_asymptotic} hold. Let $K\leq n^r$ for some $r>0$. In addition, assume that $|\cA|\geq \sqrt{\log\log K}$, $s_1\log^{3/2}(K\vee n)/n^{1/2}=o(1)$ and $s_1=O(K^{c})$ for some $c<1/2$. Then, as $n,K\rightarrow\infty$, we have
$$
\frac{\FDP}{|\cH_0|/|\cH|}= \alpha+o_p(1)~~\textrm{and }~~\frac{\FDR}{|\cH_0|/|\cH|}= \alpha+o(1). 
$$
The proof of this proposition follows from Theorem 3.1 in \cite{liu2013gaussian}. The conditions $|\cA|\geq \sqrt{\log\log K}$ and $s_1=O(K^{c})$ are equation (12) in \cite{liu2013gaussian}. The condition $s_1\log^{3/2}(K\vee n)/n^{1/2}=o(1)$ together with Theorem \ref{thm:xi_asymptotic} guarantees that
\begin{equation}
\max_{1\leq t< k\leq K}\sup_{x \in \RR}\Big|\PP(\hat T_{t,k} \leq x )-\Phi(x)\Big|=o(1/\sqrt{\log (K\vee n)}),
\end{equation}
which is equivalent to condition (13) in \cite{liu2013gaussian}. We refer to \cite{liu2013gaussian} for the detailed proof.  Proposition \ref{propbh} implies that the B-H procedure can control FDR asymptotically under certain assumption. However, our numerical results show that this procedure may fail to control FDR in our model when $K$ is relatively large. One possible reason is that the $o_p(1)$ or $o(1)$ terms in the approximation of FDP or FDR in the above proposition are not sufficiently small in simulations due to the dependence of the test statistics. Thus, if the goal of the data analysis is to find statistically reliable scientific discoveries, the B-Y method, while quite conservative, may be the one that's  more suitable for this purpose.

\end{proposition}

\vskip 0.2in
\bibliography{ref}

\end{document}